\documentclass{article}

\usepackage[preprint]{neurips_2026}

\usepackage[utf8]{inputenc} 
\usepackage[T1]{fontenc}    
\usepackage{hyperref}       
\usepackage{url}            
\usepackage{booktabs}       
\usepackage{amsfonts}       
\usepackage{microtype}      
\usepackage[table,xcdraw]{xcolor}
\usepackage{caption} 
\usepackage{graphicx} 
\usepackage{algorithm}
\usepackage{algpseudocode}
\usepackage{amsmath, amssymb, leftindex, nicefrac}
\usepackage{amsthm, mathrsfs}
\usepackage{tocloft,titletoc} 
\usepackage{subcaption}
\usepackage{todonotes}
\usepackage{changes}
\usepackage{array}
\usepackage{longtable}  
\usepackage{wrapfig}
\usepackage{pifont}
\usepackage{forest}
\usepackage{tikz}
\usepackage[most]{tcolorbox}
\usetikzlibrary{arrows.meta}

\newtheorem{definition}{Definition}[section]

\newtheorem{remark}[definition]{Remark}

\newtheorem{lemma}[definition]{Lemma}
\newtheorem{theorem}[definition]{Theorem}
\newtheorem{proposition}[definition]{Proposition}

\newcommand{\proofstep}[1]{\medskip\noindent\textbf{#1.}}

\title{Escaping Local Minima Provably\\in Non-convex Matrix Sensing:\\A Deterministic Framework via Simulated Lifting}

\author{%
Tianqi Shen\textsuperscript{1} \quad
Jinji Yang\textsuperscript{1} \quad
Junze He\textsuperscript{1} \quad
Kunhan Gao\textsuperscript{2} \quad
Zeyu Zheng\textsuperscript{3} \quad
Ziye Ma\textsuperscript{1}
\\
\textsuperscript{1}Department of Computer Science, CityUHK
\\
\textsuperscript{2}Department of Computer Science, Yale University
\\
\textsuperscript{3}Department of Industrial Engineering and Operations Research, UC Berkeley
\\
\texttt{tianqshen5-c@my.cityu.edu.hk} \quad
\texttt{jinji.yang@cityu.edu.hk} \\
\texttt{junzehe2-c@my.cityu.edu.hk} \quad
\texttt{kunhan.gao@yale.edu} \\
\texttt{zyzheng@berkeley.edu} \quad
\texttt{ziyema@cityu.edu.hk}
}

\begin{document}

\maketitle

\begin{abstract}
Non-convex optimization is a fundamental challenge in operations research and modern machine learning, largely because unfavorable optimization landscapes make it difficult to find global solutions. Existing approaches can be broadly categorized into two classes: heuristic methods, which do not alter the landscape, and landscape-smoothing methods, such as convex relaxation and, most notably, over-parameterization. Although smoothing methods can be highly effective, they are often computationally expensive. This motivates our approach: simulating over-parameterization to exploit information from the corresponding high-dimensional landscape without explicitly incurring its full computational cost.
In this paper, we study low-rank matrix sensing as a representative non-convex case study whose landscape can hinder convergence to a global optimum. Recent work has shown that over-parameterization via tensor lifting can transform spurious local minima into strict saddle points, partially explaining why massive scaling can improve generalization and performance in deep learning.
Motivated by this observation, we propose a \textbf{S}imulated \textbf{O}racle \textbf{D}irection (SOD) escape mechanism that \emph{simulates} the landscape geometry and escape directions of the tensor space without explicitly lifting the problem. Specifically, we develop a mathematical framework for projecting tensor-space escape directions back onto the original matrix space, thereby guaranteeing a strict decrease in the objective value at spurious local minima. To the best of our knowledge, this is the first deterministic framework that provably escapes spurious local minima without relying on explicit smoothing or heuristic mechanisms such as randomness or momentum.
Numerical experiments demonstrate that the proposed framework reliably escapes local minima and facilitates convergence to global optima, while incurring minimal computational cost compared with explicit tensor over-parameterization. More broadly, this framework has implications for non-convex optimization beyond matrix sensing by demonstrating how simulated over-parameterization can be used to tame challenging optimization landscapes.
The project website is publicly accessible at \url{https://sod-escape.github.io/}.
\end{abstract}

\section{Introduction} 
\label{sec1}

Non-convex optimization plays a crucial role in operations research and many other fields, such as modern machine learning~\citep{danilova2022recent}. However, these problems are often NP-hard and characterized by loss landscapes containing spurious local minima, saddle points, and extended flat regions~\citep{jain2017non}. Such geometric irregularities make it challenging for algorithms originally developed for convex settings to identify global minimizers efficiently.
Formally, given an objective function $f: \operatorname{dom} f \to \mathbb{R}$ with effective domain $\operatorname{dom} f$, a point $x^\star$ is called a \textit{global minimizer} and a point $\hat{x}$ is called a \textit{local minimizer} if
\begin{align*}
	f(x^\star) &\le f(x),
	\quad \forall\, x \in \operatorname{dom} f,
	\tag{global minimizer}
	\\
	f(\hat{x}) &\le f(x),
	\quad \forall\, x \in \mathbb{B}(\hat{x}, \varepsilon)\cap \operatorname{dom} f,
	\quad \exists\, \varepsilon > 0,
	\tag{local minimizer}
\end{align*}
where $\mathbb{B}(\hat{x}, \varepsilon)$ denotes the Euclidean ball of radius $\varepsilon$ centered at $\hat{x}$. We refer to a local minimizer that is not globally optimal as 
a \textit{spurious solution}. Therefore, escaping spurious solutions $\hat{x}$ and finding global minimizers $x^\star$ in non-convex landscapes remains a longstanding and important research challenge~\citep{danilova2022recent,chen2026nexus}.
Existing approaches 
can be broadly classified into two categories according to whether they alter the optimization landscape.

The first category consists of heuristic approaches, which leave the original landscape unchanged and instead modify the optimization trajectory through mechanisms such as randomness~\citep{lin2022gradient,de2015global} like perturbed gradient descent~\citep{jin2017escape}, momentum~\citep{shi2026momentum} like Adam~\citep{kingma2014adam}, or norm constraints~\citep{yuan2015recent}. 
However, when the landscape induced by the original parameterization is highly irregular or insufficiently expressive, even sophisticated handcrafted or adaptive update rules may provide only local remedies~\citep{zhang2026data}. This observation helps explain why, in deep learning, small or shallow models trained with Adam can sometimes underperform larger or deeper models trained with stochastic gradient descent (SGD)~\citep{wilson2017marginal,nguyen2017loss}.

\begin{figure}[h]
	\centering
	\footnotesize
	\makebox[\linewidth][c]{%
	\begin{forest}
		for tree={
			grow'=east,
			parent anchor=east,
			child anchor=west,
			anchor=west,
			align=left,
			rounded corners,
			draw,
			edge={-Latex},
			l sep=4mm,
			s sep=1mm,
			inner xsep=2pt,
			inner ysep=1pt,
		},
		tikz+={\draw[-Latex] (op.south) |- (sim.west);}
		[Existing approaches
			[Heuristics
				[Randomness
					[Perturbed GD]
				]
				[Momentum
					[Adam]
				]
				[\dots]
			]
			[Over-parameterization, draw=red, name=op
				[Convex relaxation: {\scriptsize the most theoretically grounded}]
				[Model scaling: {\scriptsize the most significant branch in deep learning}, draw=blue]
				[Tensor lifting: {\scriptsize a quantitative perspective on scaling laws}]
				[Simulated over-parameterization, draw=red, name=sim, edge path={}, before computing xy={x+=0mm,y-=8mm}
					[\textbf{SOD-Escape (ours)}: \\{\scriptsize the first simulated over-parameterization}, draw=red]
				]
			]
		]
	\end{forest}%
	}
	\caption{A taxonomy of approaches to non-convex optimization by whether they alter the optimization landscape. Our method belongs to landscape smoothing via (simulated) over-parameterization.}
	\label{figure:taxonomy}
\end{figure}

We therefore favor the second category of approaches, which aim to smooth the landscape by fundamentally eliminating or mitigating spurious local minima, rather than merely attempting to bypass them within the original difficult geometry. This category can be further divided into two strategies. The first is convex relaxation~\citep{nesterov2000semidefinite,kojima2000discretization}, which is arguably the most theoretically grounded and ideal form of landscape smoothing. However, convex relaxations are often impractical for large-scale operational decision systems~\citep{yalcin2022semidefinite}. Consequently, we focus on the second strategy: over-parameterization~\citep{du2018power,du2019gradient,allen2019convergence,yousefzadeh2022deep}. 
In the era of large language models, the empirically observed scaling laws~\citep{kaplan2020scaling,zhang2026configuration} (i.e. performance tends to improve predictably as model size, data, and compute increase) provide compelling evidence that increasing model size can systematically improve the achievable loss in modern learning systems~\citep{gu2020characterize}. Then tensor lifting~\citep{ma2024algorithmic}, viewed as a new form of over-parameterization, can be interpreted as a depth extension with linear activations that modifies the optimization landscape and offers a quantitative perspective on scaling laws. By optimizing over a higher-dimensional parameter space, tensor lifting can smooth the effective landscape and reveal escape directions that are invisible in the original low-dimensional domain~\citep{xu2018benefits,sharifnassab2020bounds,ma2023over}.

However, directly optimizing in a lifted space can still be computationally expensive and memory intensive, making it unsuitable for real-time operational decision systems~\citep{liu2021we}. Therefore, we seek to simulate the smoothing effects of over-parameterization while remaining in the original parameter space. Our key idea is to use the hidden directions in tensor space as oracle information to construct a deterministic escape mechanism, termed \textit{Simulated Oracle Direction} (SOD) Escape. Unlike direct over-parameterization, SOD Escape operates entirely within the original domain while mimicking the geometric benefits of lifted optimization. This design enables deterministic escape from spurious solutions without incurring the computational or memory overhead associated with optimization in an explicitly over-parameterized space.

A central challenge arises: oracle directions that form superpositions of multiple states in the over-parameterized space, do not typically correspond to single axes in the low-dimensional space~\citep{ma2024algorithmic}. The primary theoretical task of this paper is therefore to characterize the conditions under which such superpositions can be meaningfully projected back to the original domain, thereby providing actionable guidance for escaping non-global spurious solutions.
To concretize this intuition, we study the \textit{matrix sensing} (MS) problem~\citep{davenport2012introduction} as a testbed. MS provides a common mathematical abstraction for large-scale structured estimation problems that must be solved before reliable downstream decisions can be made in many data-driven operational decision systems~\citep{xu2024implicit}. 
Representative examples include covariance estimation~\citep{yue2026geometric}, power-system state estimation~\citep{jin2019towards}, phase retrieval~\citep{candes2015phase,mcrae2026phase}, collaborative filtering~\citep{borgs2022iterative}, and quantum tomography~\citep{liu2011universal}, where the underlying operators estimate latent states from nonlinear and incomplete measurements by exploiting low-dimensional structure.
The insights developed in this work may extend beyond MS and provide general guidance for designing deterministic escape strategies for broader classes of non-convex optimization problems in operations research~\citep{ma2025can} and modern machine learning~\citep{wei2025benefits}.

The remainder of the paper is organized as follows. Section~\ref{sec1} finalizes the MS setup and preliminaries; Section~\ref{sec2} reviews non-convex optimization background and prior MS results; Section~\ref{sec3} states our main theorems and interpretations. We develop a deterministic escape mechanism and guarantees in two stages: Section~\ref{sec4} covers restricted regimes where a one-step escape in the overparameterized space followed by projection yields a closed-form escape point $\check{X}$; Section~\ref{sec5} treats general regimes where direct projection may fail, proposing truncated projected gradient descent (TPGD) to enable provably valid projection and showing that the resulting escape point $\check{X}$ still has a closed-form representation, effectively simulating the TPGD process without explicit lifting. Finally, Section~\ref{sec6} presents two case studies and a success-rate comparison study to demonstrate feasibility, together with an illustrative neural-network experiment showing that the proposed approach can extend beyond the MS setting. Section~\ref{sec7} concludes the paper.

\subsection{Problem Setup and Preliminaries}
\label{subsec:problem-setup-and-short-preliminaries}

\paragraph{Problem Setup.}
Low-rank matrix sensing aims to recover a low-rank positive semidefinite (PSD) matrix $M^\star = ZZ^\top \in \mathbb{R}^{n \times n}$ with rank $r := \operatorname{rank}(M^\star) < n$ from a set of linear measurements $\mathcal{A}$:
\begin{equation*}
    \begin{aligned}
        \min_{M \in \mathbb{R}^{n \times n}}& 
        f(M) := \frac{1}{2} \|\mathcal{A}(M) - b\|_2^2,
    	\\
    	\text{s.t. }&
    	\operatorname{rank}(M) \leq r .
    \end{aligned}
\end{equation*}
Here, $\mathcal{A}:\mathbb{R}^{n \times n}\to\mathbb{R}^m$ is a known sensing operator defined by
\begin{equation*}
	\mathcal{A}(M) = [\langle A_1, M \rangle, 
	\langle A_2, M \rangle, \ldots, 
	\langle A_m, M \rangle]^\top,
\end{equation*}
where $A_i\in\mathbb{R}^{n \times n}$ are symmetric sensing matrices and $b=\mathcal{A}(M^\star)$.

Because the rank constraint is non-convex, one common approach is convex relaxation, which replaces the rank constraint with a nuclear-norm constraint, the convex envelope of the rank function, thereby reformulating the problem as a semidefinite program (SDP)~\citep{yalcin2022semidefinite}. However, as discussed in the Introduction, this approach can be computationally expensive for large-scale operations research problems. Therefore, this paper considers the \textit{Burer-Monteiro} (BM) factorized formulation, which replaces the rank-constrained problem with an optimization problem over the product of two lower-dimensional matrices. Specifically, we optimize over $X \in \mathbb{R}^{n \times r_{\text{search}}}$ such that $XX^\top \approx ZZ^\top$. When $r_{\text{search}}>r$, this formulation corresponds to matrix over-parameterization through rank expansion. However, as noted in the Introduction, this paper does not rely on rank over-parameterization, because the true rank $r$ is often difficult to determine and thus search rank $r_{\text{search}}$ may be too large in practical operations research problems.

The BM-factorized MS problem considered in this study is formulated as
\begin{equation}
	\min_{X \in \mathbb{R}^{n \times r_{\text{search}}}} h(X) := f(XX^\top)
	= \frac{1}{2} \|\mathcal{A}(XX^\top) - b\|_2^2.
	\label{equation:matrix_objfunc}
	\tag{P1}
\end{equation}
This formulation still induces a non-convex landscape and thus is NP-hard~\citep{tillmann2013computational}. Moreover, it can represent general polynomial optimization problems~\citep{molybog2020conic} and is equivalent to training two-layer quadratic neural networks~\citep{li2018algorithmic}.
A basic example 
from \citep{zhang2018much}, which illustrates the structure of the objective, is given below:

\paragraph{Basic Example.} Recover $M^\star = zz^\top$ with $z = [1,0]^\top$ under the following setting:
\begin{equation}  
    \begin{aligned}
        A_1 = \begin{bmatrix}
            1 & 0 \\
            0 & \nicefrac{1}{2}
        \end{bmatrix},\
        A_2 = \begin{bmatrix}
            0 & \nicefrac{\sqrt{3}}{2} \\
            \nicefrac{\sqrt{3}}{2} & 0
        \end{bmatrix},\
        A_3 = \begin{bmatrix}
            0 & 0 \\
            0 & \nicefrac{\sqrt{3}}{2}
        \end{bmatrix}.
    \end{aligned}
    \label{equation:basic_example}
\end{equation}

One can verify that this non-convex landscape has a local minimum at $\hat{x} = (0, \pm 1/\sqrt{2})$, and the number of spurious solutions increases as loss landscape becomes more expressive. Consequently, BM-factorized MS problems provide a rich and analyzable \textit{testbed} for evaluating our proposed escape mechanism.
Figure~\ref{figure:basic_example} illustrates the example in Equation~\eqref{equation:basic_example}. 
The figure demonstrates how the optimization trajectories of SGD and our proposed escape method will traverse the landscape. Starting from random initializations, traditional SGD algorithm can become trapped inside attraction basins of spurious local minima, whereas our SOD step enables a direct jump out of these basins, landing near the ground truth $M^\star$. This highlights the effectiveness of deterministic escape.
\begin{figure*}[h]
	\centering
	\includegraphics[width=1.0\textwidth]{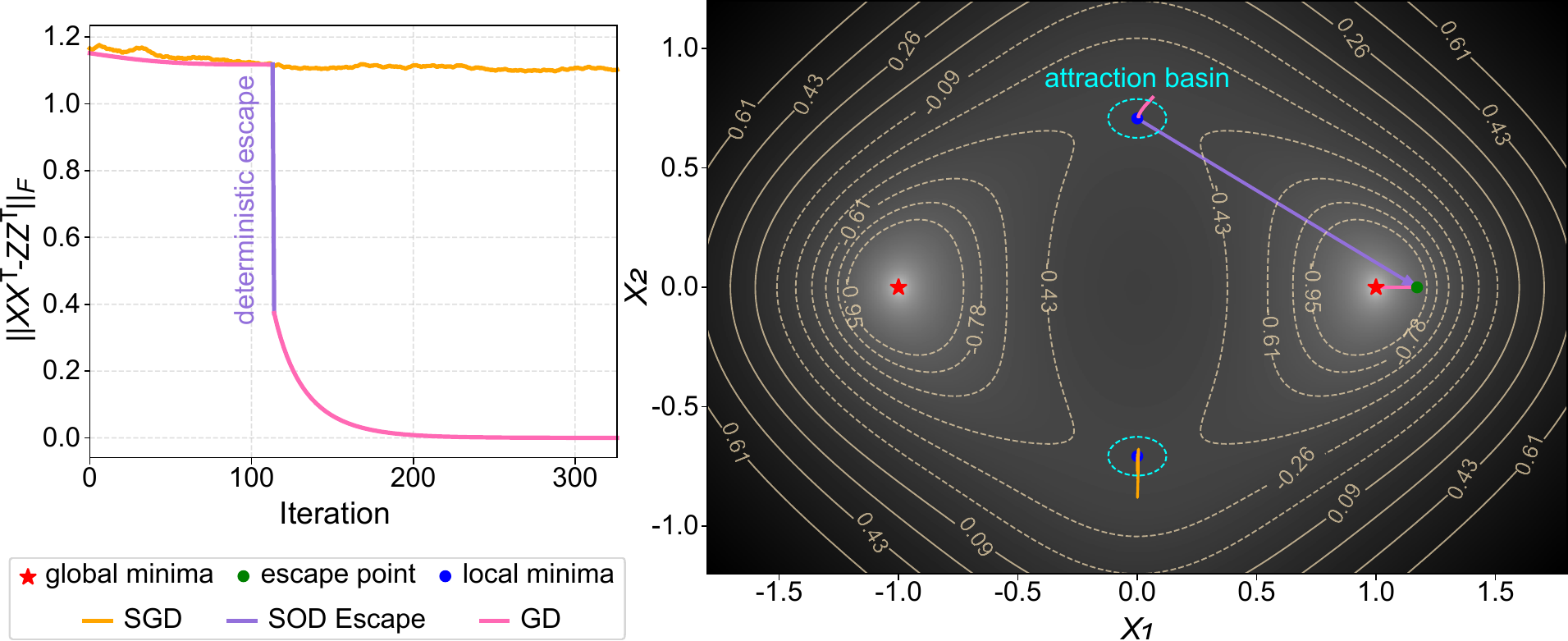}
	\caption{
		Escape behavior of the proposed SOD method from a spurious solution $\hat{x}$. 
		Left: distance to the $M^\star$ versus iterations. 
		Right: loss landscape (log-scaled contour map) and trajectories. 
		SOD jumps from $\hat{x}$ to $\check{x}$ (green) and subsequent GD converges to $M^\star$, whereas GD alone is trapped (blue) and SGD does not escape the basin within the iteration budget.
	}
	\label{figure:basic_example}
\end{figure*}

\paragraph{Short Preliminaries.}
Before we introduce the formal results of this work, we would like to first introduce some preliminary concepts and notations related to MS to facilitate understanding. In particular, the gradient of $h$ w.r.t. $X\in\mathbb{R}^{n\times r_\text{search}}$ is given by 
\begin{equation}
	\nabla h(X) 
	= 2\mathcal{A}^*\mathcal{A}
	(XX^\top - M^\star) X
	= 2
	\sum_{i=1}^{m}
	\langle
	A_i, XX^\top - M^\star
	\rangle 
	A_i X,
    \label{equation:matrix_gradient}
\end{equation}
where $\mathcal{A}^*:\mathbb{R}^m\to\mathbb{R}^{n\times n}$ is the adjoint operator of $\mathcal{A}$ defined as $\mathcal{A}^*(y)=\sum_{i=1}^m y_i A_i$. 
Moreover, if $\nabla f(M)=\mathcal{A}^*(\mathcal{A}(M)-b)$ denotes the gradient of $f$, then the chain rule yields $\nabla h(X) = 2 \nabla f(XX^\top) X$.

To analyze problem~\eqref{equation:matrix_objfunc}, a standard assumption in the MS literature is the Restricted Isometry Property (RIP) \citep{candes2010power,recht2010guaranteed}.
\begin{definition}[Restricted Isometry Property (RIP)]
\label{definition:rip}
    A linear operator $\mathcal{A}$ satisfies the $\delta_p$-RIP for some $\delta_p \in [0,1)$ if, for all symmetric matrices $M$ with $\operatorname{rank}(M)\le p$, the inequality below holds:
    \begin{equation*}
        (1 - \delta_{p}) \|M\|_F^2 
        \leq 
        \|\mathcal{A}(M)\|_2^2 
        \leq 
        (1 + \delta_{p}) \|M\|_F^2.
    \end{equation*}
\end{definition}

RIP implies that the linear measurement approximately preserves the Frobenius norm of all low-rank matrices. This assumption is satisfied by a wide range of random measurement operators \citep{bah2010improved}. 
A smaller RIP constant $\delta_p$ leads to a more benign optimization landscape for problem~\eqref{equation:matrix_objfunc}, while larger values of $\delta_p$ (closer to $1$) correspond to increasingly complicated loss surfaces often containing numerous spurious solutions \citep{ma2024absence}. 
Hence, deterministic (non-randomized, e.g. vanilla GD) algorithms may struggle to converge to the global optimum (as shown in Figure~\ref{figure:basic_example}), namely the \textit{ground-truth} matrix $M^\star$.

\subsection{Notation}
\label{subsec:notation}
Throughout the paper, we adopt the following notation. Lowercase Roman letters (e.g., $a$) denote scalars or vectors, depending on the context. Uppercase Roman letters (e.g., $A$) represent matrices. Bold lowercase and uppercase letters (e.g., $\mathbf{a},\mathbf{A}$) indicate tensors. Calligraphic letters (e.g., $\mathcal{A}$) denote linear operators.
Let $\hat{X}\in\mathbb{R}^{n\times r}$ be a local minimizer of $h$. Since $h$ is differentiable, it satisfies the first-order stationarity condition $\nabla h(\hat{X})=0$.
Let the eigendecomposition of $\nabla f(\hat{X}\hat{X}^\top)$ be
$\sum_{\varphi=1}^n \lambda_\varphi u_\varphi u_\varphi^\top$ where the eigenvalues are ordered as $\lambda_1\geq\cdots\geq\lambda_n$. 
Suppose $\hat{X}$ has a thin singular value decomposition $\hat{X}=\sum_{\phi=1}^r \sigma_\phi v_\phi q_\phi^\top$ with singular values ordered $\sigma_1\geq\cdots\geq\sigma_r$. 
The first-order optimality condition implies that $u_n\perp v_\phi$ for all $\phi$. 
We use the shorthand $\circ l$ for the $l$-fold Cartesian product, and write $[l]:=\{1,\cdots,l\}$. The symbol $\otimes$ denotes the tensor (outer) product. For $Y\in\mathbb{R}^{n\times r}$, we let $\operatorname{vec}(Y)\in\mathbb{R}^{nr}$ be its vectorization, and define the $l$-fold symmetric tensor power $\operatorname{vec}(Y)^{\otimes l}\in\mathbb{R}^{nr\circ l}$. 
For a tensor $\mathbf{v}\in\mathbb{R}^{(nr)\circ l}$, we denote by $\operatorname{stack}(\mathbf{v})\in\mathbb{R}^{nrl}$ its stacked vector form. The operators $\operatorname{unvec}$ and $\operatorname{unstack}$ are defined as the inverses of $\operatorname{vec}$ and $\operatorname{stack}$, respectively.

\section{Background and Related Work} 
\label{sec2}

\subsection{Non-Convex Optimization and Escaping Spurious Solutions}
\label{subsec:non-convex-optimization-and-escape-spurious-solutions}
Convex optimization admits a benign landscape \citep{criscitiello2026sensor} in which every local minimum is global, enabling gradient-based methods to converge reliably under mild conditions \citep{boyd2004convex,bubeck2015convex}. 
In contrast, non-convex optimization arises when the \textit{objective function} or \textit{feasible region} lacks convexity \citep{danilova2022recent}, typically inducing complex landscapes populated by multiple critical points (e.g. local minima and saddle points) \citep{jain2017nonconvex}. 
Of particular concern are \emph{spurious local minima}, namely second-order critical points that satisfy local optimality conditions but are suboptimal globally. 
Such spurious solutions are prevalent in operations research such as the pooling problem \citep{dey2015analysis} and the optimal power flow problem \citep{kocuk2016strong}, as well as in machine-learning applications such as tensor decomposition \citep{zheng2015interpolating} and neural-network training \citep{jain2017non}. They therefore pose a fundamental obstacle to reliable optimization.

Existing strategies for escaping spurious local minima without altering the optimization landscape are largely heuristic. 
Specifically, probabilistic perturbation methods inject randomness into the optimization dynamics to destabilize stationary points. Beyond stochastic gradient descent (SGD), in which minibatch sampling induces implicit gradient noise \citep{bottou2018optimization}, stochastic gradient Langevin dynamics (SGLD) adds Gaussian noise to gradient updates to enable approximate posterior sampling \citep{welling2011bayesian}. Perturbed gradient descent and its variants inject isotropic noise near critical points to provably escape saddle points in polynomial time \citep{jin2017escape,jin2021nonconvex}. However, such noise injection can degrade success rates by disrupting the deterministic gradient flow, often causing \textit{oscillations} near local minima and poor global convergence \citep{de2015global}.
Rule-based adaptation methods~\citep{wen2026fantastic}, by contrast, avoid explicit noise by heuristically modifying descent trajectories. In addition to momentum-based optimizers such as Adam \citep{kingma2014adam} and its decoupled weight-decay variant \citep{loshchilov2017decoupled}, more sophisticated strategies have been proposed. For instance, Nesterov's accelerated gradient employs a lookahead mechanism to anticipate gradient changes \citep{nesterov1983method}. ALTO \citep{zhao2025exploring}, which extends L-BFGS with curvature thresholding \citep{liu1989limited}, leverages approximate curvature information to traverse flatter regions efficiently. Despite strong empirical performance in certain regimes, such as accelerated progress through specific landscape regions, these methods often lack theoretical guarantees and may introduce unintended side effects because of their sensitivity to hyperparameters.
%
All these observations motivate the development of more effective approaches based on landscape smoothing, most notably over-parameterization.

\subsection{Typical Over-parameterization Methods in MS}
\label{subsec:typical-over-parameterization}
Over-parameterization has recently emerged as a key paradigm in non-convex optimization \citep{allen2019convergence}. By enlarging the parameter space beyond the problem’s intrinsic degrees of freedom, it can reshape the loss landscape \citep{oneto2023we}, often turning local minima into global minima or saddle points \citep{xu2018benefits}. This effect resembles interpolation \citep{liu2025understanding}, smoothing the surface and improving the attainability of optimal solutions.
Rank over-parameterization (i.e. $r_{\text{search}}>r$) under suitable RIP-type conditions can remove spurious solutions. For instance, \citep{zhang2024improved} show that if $n>r_\text{search}>r[(1+\delta_p)/(1-\delta_p)-1]^2/4$,
then every solution $X^\star$ to~\eqref{equation:matrix_objfunc} satisfies $X^\star{X^\star}^\top=M^\star$. Analogous RIP-based guarantees for the $\ell_1$ loss are given in \citep{ma2023global}.
Beyond increasing search rank, convex relaxation offers another form of over-parameterization by lifting the factorized variable ($O(nr)$ parameters) to a full matrix ($O(n^2)$ parameters) \citep{recht2010guaranteed}. As shown in \citep{yalcin2022semidefinite}, when $M^\star$ has sufficiently high rank, the RIP threshold for exact recovery via semidefinite programming can approach $1$, underscoring the strength of such lifted formulations. However, \citep{xiongover} reports that over-parameterization may make GD exponentially slower than in the exactly parameterized case, suggesting that these guarantees can be practically limited.

\subsection{Tensor Over-Parameterization: Introduction and Foundational Results}
\label{subsec:tensor-over-parameterization}
\citep{ma2023over,ma2024algorithmic} introduce a tensor-based over-parameterization framework that lifts the original decision variable from a matrix to a higher-order tensor by embedding the matrix variable $X \in \mathbb{R}^{n \times r}$ into a higher-dimensional tensor space. Specifically, they define the lifted tensor as $\mathbf{w} = \operatorname{vec}(X)^{\otimes l} \in \mathbb{R}^{nr\circ l}$, where $l$ is the lifting order.
To preserve the structure of the original matrix during lifting and recovery, a permutation tensor $\mathbf{P} \in \mathbb{R}^{n \times r \times nr}$ is introduced, satisfying
$
    \langle 
	\mathbf{P}, \operatorname{vec}(X) 
	\rangle_3 = X 
$
for any $X \in \mathbb{R}^{n \times r}$.
This yields
\begin{equation}
    \begin{aligned}
    	\mathbf{P}(\mathbf{w}) =
    	\langle 
    	\mathbf{P}^{\otimes l}, 
    	\operatorname{vec}(X)^{\otimes l}
    	\rangle_{3*[l]} = X^{\otimes l}
    	\in\mathbb{R}^{[n\times r]\circ l},
    \end{aligned}   
    \label{equation:permutation_tensor}
\end{equation}
where $\langle \cdot, \cdot \rangle_3$ denotes contraction along the third mode of $\mathbf{P}$, and $\langle \cdot, \cdot \rangle_{3*[l]}$ generalizes mode-wise contraction for order-$l$ tensors. This mapping ensures that any $X$ can be faithfully reconstructed from the tensor space using $\mathbf{P}$.
The sensing operator $\mathcal{A}$ is subsequently extended to the tensor domain to align with the lifted representation. Specifically, the original linear operator $\mathcal{A}$ is first encoded as a third-order tensor $\mathbf{A} \in \mathbb{R}^{m \times n \times n}$, where $\mathbf{A}_{ijk} = (A_i)_{jk}$ for $i \in [m]$ and $(j,k) \in [n] \times [n]$. This tensorized form is then consistently lifted to $\mathbf{A}^{\otimes l}$, enabling it to operate on higher-order tensor variables in the lifted space.

By defining the reconstruction error as $f^l(\mathbf{M}) := \|\langle\mathbf{A}^{\otimes l},\mathbf{M}\rangle - b^{\otimes l}\|_F^2$, where $\mathbf{M}\in\mathbb{R}^{[n\times n]\circ l}$ is a candidate tensor in the lifted space, the optimization problem is reformulated in the high-dimensional domain as:
\begin{equation}
    \min_{\mathbf{w} \in \mathbb{R}^{nr \circ l}} 
    h^l(\mathbf{w}) :=
    f^l
    (\langle
    \mathbf{P}(\mathbf{w}), \mathbf{P}(\mathbf{w})
    \rangle_{2*[l]}),
    \label{equation:tensor_objfunc_h}
	\tag{P2}
\end{equation}
where $\langle \mathbf{P}(\mathbf{w}), \mathbf{P}(\mathbf{w}) \rangle_{2*[l]}$ performs contraction along modes $\{2,4,\dots,2l\}$, preserving the bilinear structure of the original matrix formulation.
This reformulation enables the problem to be addressed within a higher-order tensor space, while maintaining a consistent mapping to the matrix domain (but only limited to rank-1 tensor).
The key theoretical results characterizing this lifted optimization problem are summarized below.

\begin{figure*}[h]
	\centering
	\includegraphics[width=0.48\textwidth]{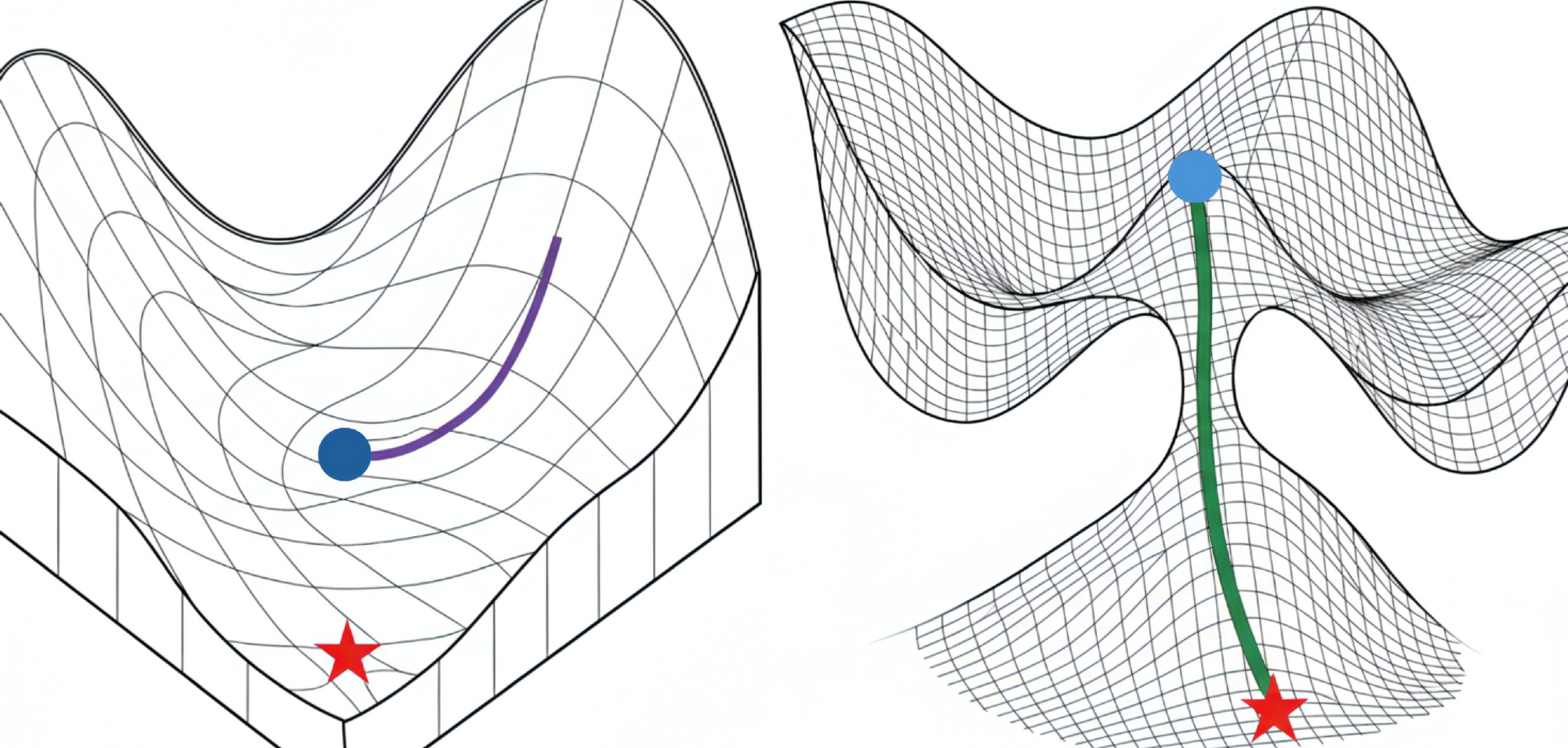}
	\caption{
		Comparison of optimization landscapes in matrix and tensor spaces. 
		Left: in matrix space, GD (purple) becomes trapped at a local minimum (dark blue), away from the global minimum (red pentagram). 
		Right: in tensor space, the corresponding point becomes a saddle (light blue), and GD (green) traverses it and converges to the global minimum. 
		This schematic is illustrative; real landscapes are far more complex.}
	\label{figure:matrix_tensor}
\end{figure*}

First, \cite[Theorem~5.4]{ma2023over} shows that tensor lifting with odd $l$ reshapes the landscape: spurious local minima in the matrix space become strict saddle points in the tensor domain, and these saddles admit explicit rank-1 escape directions of the form $\operatorname{vec}(u_n q_r^\top)^{\otimes l}$. Ideally, one would preserve a rank-1 structure to enable a clean, interpretable projection back to the matrix space; however, due to the superposition effect discussed in Section~\ref{sec1}, even an infinitesimal step along a rank-1 direction typically produces a rank-2 tensor. This difficulty is alleviated by \cite[Theorem~1]{ma2024algorithmic}, which reveals an implicit-regularization phenomenon for GD in the tensor space: along the trajectory, iterates tend to be approximately rank-1. Motivated by this, we can extract the dominant rank-1 component and operate projection to obtain an escape matrix $\check{X}$. Figure~\ref{figure:matrix_tensor} visualizes this geometric transformation and the benefits of tensor lifting.

\section{Main Results}
\label{sec3}

Tensor over-parameterization lifts a matrix-space local minimum $\hat{X}$ to the tensor $\hat{\mathbf{w}}:=\operatorname{vec}(\hat{X})^{\otimes l}$, incurring an exponential increase in dimension and making gradient-based optimization in the lifted space computationally prohibitive. To circumvent this cost, our main theorem constructs an escape point $\check{X}$ directly in the original matrix space by moving along the oracle escape direction $\operatorname{vec}(u_n q_r^\top)^{\otimes l}$ and emulating the tensor-space gradient dynamics without explicitly optimizing in the tensor domain.
For notational convenience, define $E := \mathcal{A}^*\mathcal{A}(
u_n v_r^\top + v_r u_n^\top)\in\mathbb{R}^{n\times n}$, and introduce the post-escape tensor with escape step size $\rho > 0$:
\begin{equation}
	\mathbf{w}_\mathrm{escape} :=
	\hat{\mathbf{w}} + \rho 
	\operatorname{vec}(u_n q_r^\top)^{\otimes l},
	\label{equation:w_escape}
\end{equation}
where $\lambda_n<0,u_n$ are the smallest eigenvalue and eigenvector of $\nabla f(\hat{X}\hat{X}^\top)$, and $\sigma_r,v_r,q_r$ are the smallest nonzero singular value and corresponding singular vectors of $\hat{X}$ (see Section~\ref{subsec:notation}, notation introduction).
The matrix $E$ implicitly encodes the "invisible" escape information available in the matrix space. The rank-2 tensor $\mathbf{w}_\mathrm{escape}$ corresponds to a one-step simulation of the lifted escape in Section~\ref{sec4}, and it also serves as the initialization $\tilde{\mathbf{w}}^{(0)}$ for the multi-step simulation in Section~\ref{sec5}.

\subsection{Single-step SOD Escape Mechanism}
This SOD escape mechanism directly collapses the tensor-space point $\mathbf{w}_\mathrm{escape}$ back into the matrix space. This involves the following three components:  
1) identifying a rank-1 tensor that closely approximates $\mathbf{w}_\mathrm{escape}$ under a suitable norm, i.e., ensuring $\| \operatorname{vec}(\check{X})^{\otimes l} - \mathbf{w}_\mathrm{escape} \|$ approaches zero; 
2) determining the conditions under which the $\check{X}$ can successfully escape the attraction basin of $\hat{X}$; 
3) establishing when such conditions are satisfied.
These lead to the following Theorem~\ref{theorem:single-step-sod-informal}.

\begin{figure*}[h]
	\centering
	\includegraphics[width=0.36\textwidth]{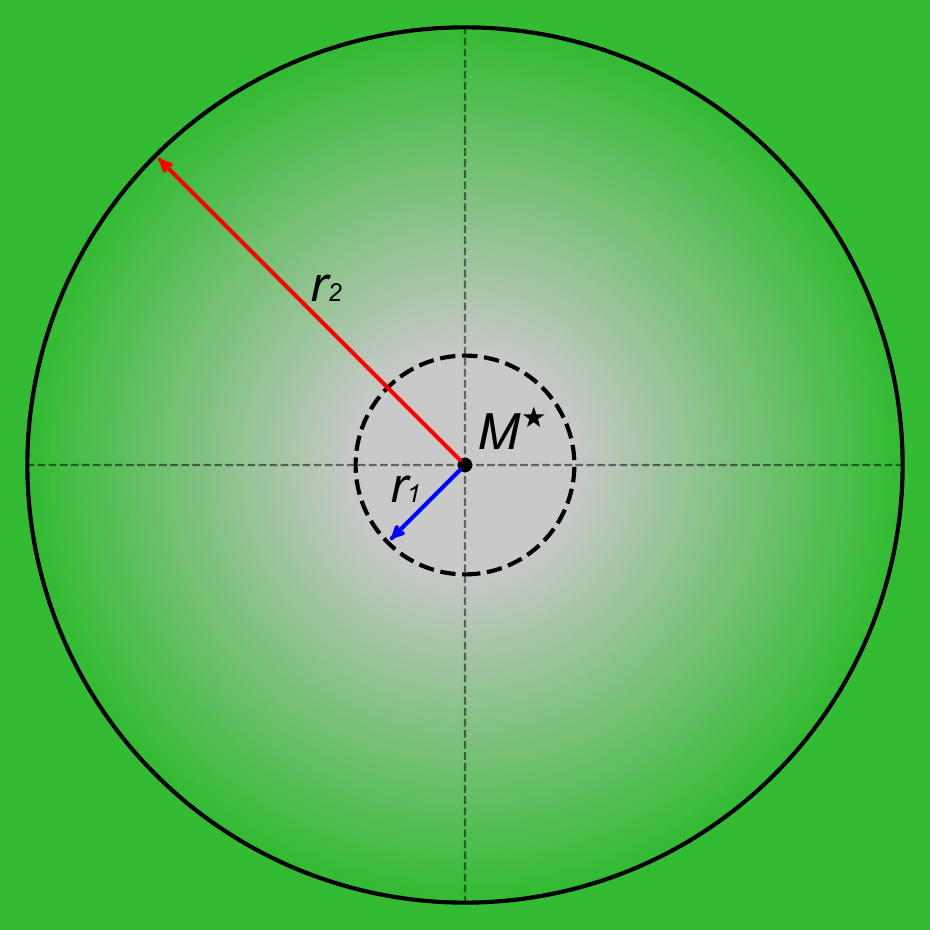}
	\caption{
		Visualization of escape regions (concentric circles) with EFS zones. The radius of the inner circle is $r_1$, defined in Equation~\eqref{definition:r_1}, and the radius of the outer circle is $r_2$, defined in Equation~\eqref{definition:r_2}. 
		The greener the region, the higher the probability $\mathbb{P}(\mathrm{EFS} > 1)$.
	}
	\label{figure:r1_r2}
\end{figure*}

We study the condition $\mathrm{EFS} > 1$ for single-step SOD escape under a symmetric Gaussian sensing ensemble (see Equation~\eqref{equation:gaussian_ensemble}).
As illustrated in Figure~\ref{figure:r1_r2}, the parameter space is partitioned by two radii $r_1$ and $r_2$ into three regimes: outside $r_2$, escape is guaranteed; inside $r_1$, escape is impossible; and between them, escape occurs with a probability that depends on both the sensing realization and the distance $\|\hat{X}\hat{X}^\top - ZZ^\top\|_F$.
Overall, when aggregating the latter two regimes, we show in Section~\ref{subsec:probabilistic-escape-analysis} that $\mathbb{P}(\mathrm{EFS} > 1) < 1/3$, implying that single-step SOD escape is unlikely in many MS instances.
This limitation fundamentally arises from the irreducible mismatch between the rank-2 escape tensor $\mathbf{w}_{\mathrm{escape}}$ and its rank-1 approximation $\operatorname{vec}(\check{X})^{\otimes l}$ under the tensor spectral norm, which incurs a non-negligible approximation error.

\begin{theorem}[Single-step SOD Escape, informal]\label{theorem:single-step-sod-informal}
	Consider the problem of escaping from a local minimum $\hat{X}$ of the objective function \eqref{equation:matrix_objfunc}. Define the \ref{equation:escape_feasibility_score} as follows:
	\begin{equation*}     
		\mathrm{EFS} = 
		\frac{-\lambda_n}{\sigma_r^2 (1+\delta_p)}
		+
		\frac{
			\langle
			E, u_n u_n^\top
			\rangle^2
		}
		{2(1+\delta_p)^2},
    \end{equation*}
	where $E := \mathcal{A}^*\mathcal{A}(u_n v_r^\top + v_r u_n^\top) \in \mathbb{R}^{n \times n}$; $(\lambda_n <0,u_n),(\sigma_r,v_r,q_r)$ are the eigenpair and singular value-vector triples of $\nabla f(\hat{X}\hat{X}^\top)$ and $\hat{X}$, respectively.
	If $\mathrm{EFS}>1$, then the matrix $\check{X} = \hat{X} + \hat{\rho} u_n q_r^\top$ serves as a valid escape point, i.e., $h(\check{X}) < h(\hat{X})$. The point $\check{X}$ is capable of escaping the attraction basin of $\hat{X}$, provided that $\hat{\rho}$ lies in the interval $\hat{\rho}_1 < \hat{\rho} < \hat{\rho}_2$, where
	\begin{equation*}
		\hat{\rho}_1, \hat{\rho}_2 
		:=
		\frac{
			-\sigma_r
			\langle
			E, u_n u_n^\top
			\rangle
		}
		{(1+\delta_p)}
		\mp
		\sigma_r\sqrt{2(\mathrm{EFS}-1)}. 
    \end{equation*}
\end{theorem}

\subsection{Multi-step SOD Escape Mechanism}
The goal of SOD escape is to identify a tensor along the lifted trajectory that admits a well-structured matrix correspondence. As discussed above, when $\mathrm{EFS}>1$, a single step from the saddle point $\hat{\mathbf{w}}$ can be directly projected back into the matrix space. However, this behavior does not occur universally across all MS instances, necessitating more robust projection strategies.
To address this, we simulate multiple steps within the tensor space while maintaining control over the trajectory to ensure that it remains easily projectable. This simulated algorithm is referred to as TPGD (see Section~\ref{sec5}). Due to the implicit regularization in \cite[Theorem~1]{ma2024algorithmic}, this iterative process gradually promotes an approximately rank-1 structure. Ultimately, this leads to a well-defined matrix escape point $\check{X}$ in the original space.
We summarize the key results in Theorem~\ref{theorem:multi-step-sod-informal}.

\begin{theorem}[Multi-step SOD Escape, informal]
\label{theorem:multi-step-sod-informal}
	Consider the goal of escaping from a local minima $\hat X$ of \eqref{equation:matrix_objfunc}. Given constants $1>\rho,\eta>0$ and a sufficiently large odd integer $l\geq3$ (representing order of simulated lifting), there exist two disjoint sets \ref{equation:u_beta}, \ref{equation:u_gamma} $\subset \mathbb{R}$ with $\sup U_\beta < \inf U_\gamma$ corresponding to two deterministic points of escape from $\hat X$ (note that $\check{X}_\beta$ is $\beta$-type with $t \in U_\beta$ and $\check{X}_\gamma$ is $\gamma$-type with $t \in U_\gamma$, both are $\check{X}$):
    \begin{equation*}
        \begin{aligned}
    		\check{X}_\beta
    		&=
    		\left\{
    		\rho^{1/l} (1 - \eta \lambda_n^l)^{t/l}
    		\right\}
    		\cdot u_n q_r^\top 
    		\\
    		\check{X}_\gamma
    		&=
    		\left\{
    		- \frac{1}{2}
    		(2\eta\rho)^{1/l}
    		\left[
    		\sum\nolimits_{\tau=0}^{t-1}
    		(1 - \eta \lambda_n^l)^\tau
    		\right]^{1/l}
    		\sigma_r
    		\right\}
    		E  \hat{X} 
    	\end{aligned}
    \end{equation*}
	such that both escape points are guaranteed to yield lower objective values, i.e., $h(\check{X}) < h(\hat{X})$, where $E := \mathcal{A}^*\mathcal{A}(u_n v_r^\top + v_r u_n^\top) \in \mathbb{R}^{n \times n}$ and $(\lambda_n <0,u_n),(\sigma_r,v_r,q_r)$ are the eigenpair and singular value-vector triples of $\nabla f(\hat{X}\hat{X}^\top)$ and $\hat{X}$, respectively. While both $\check{X}$ are easy to compute, the $\beta$-type escape point has a simpler form but higher requirement on numerical precision when $l$ is large, whereas the $\gamma$-type escape point has a complicated form but is more numerically stable.
\end{theorem}

\begin{figure*}[h]
	\centering
	\includegraphics[width=0.72\textwidth]{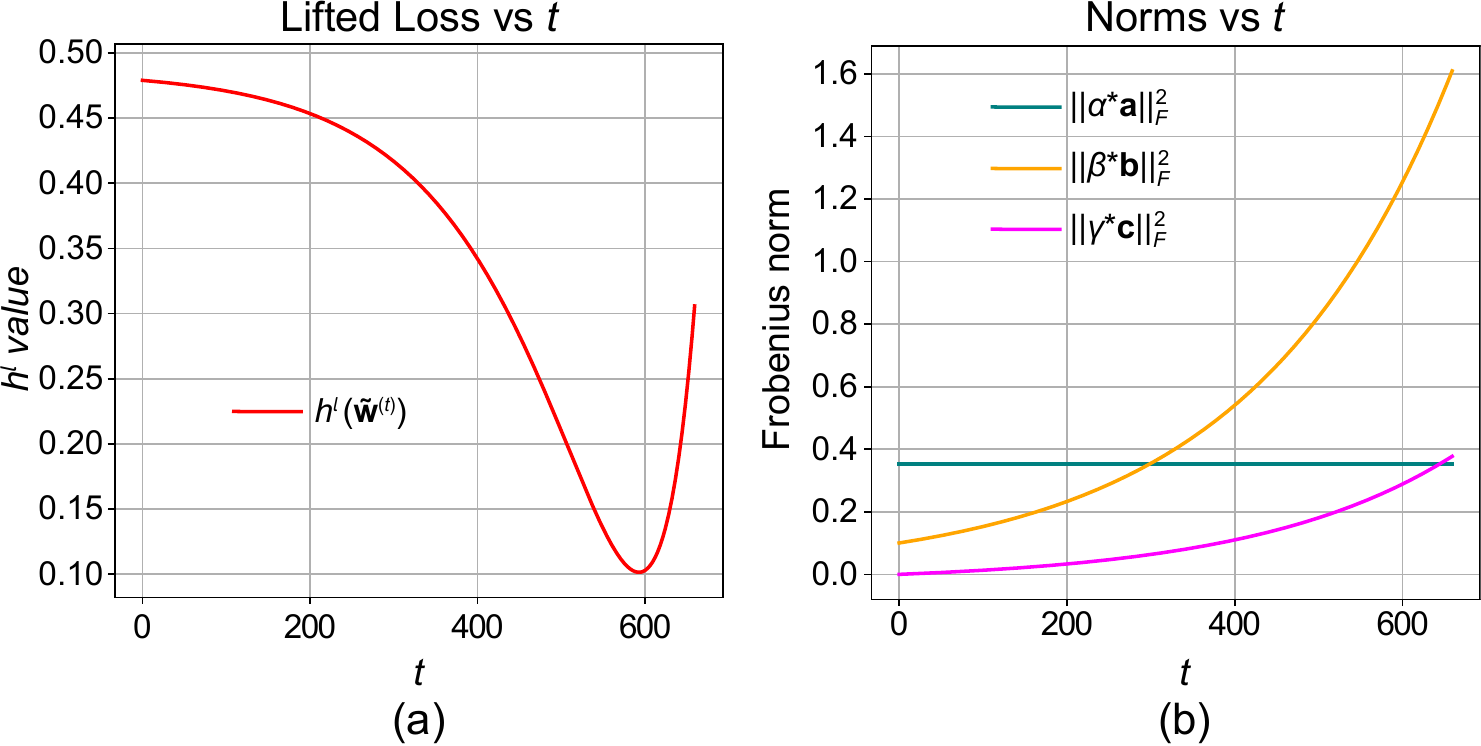}
	\caption{
		Simulation of the TPGD trajectory in a basic example.
		(a) Evolution of the lifted objective value $h^l(\tilde{\mathbf{w}}^{(t)})$ as a function of the iteration count $t$.
		The objective decreases in the early stage, while for sufficiently large $t$, truncation errors accumulate and eventually lead to an increase in the objective value.
		(b) Norms of the three components in the decomposition of $\tilde{\mathbf{w}}^{(t)}$.
		As $t$ increases, the $\mathbf{b}$-term exhibits clear norm separation from the $\mathbf{a}$- and $\mathbf{c}$-terms, indicating the emergence of a $\beta$-type escape direction.
	}
	\label{figure:simulation}
\end{figure*}

One of the main issues in choosing between the $\beta$-type and $\gamma$-type escape points is that the choice depends on the values of both $t$ and $l$; this dependence is best understood by examining the trajectory of the simulated TPGD algorithm.
Although the detailed description of TPGD is deferred to Algorithm~\ref{algorithm:tpgd}, we provide a high-level overview here. At each iteration $t$, the iterate $\tilde{\mathbf{w}}^{(t)}$ admits a closed-form decomposition into three structured components,
$
\tilde{\mathbf{w}}^{(t)}
=
\tilde{\alpha}^{(t)} \mathbf{a}
+
\tilde{\beta}^{(t)} \mathbf{b}
+
\tilde{\gamma}^{(t)} \mathbf{c},
$
as characterized in Equation~\eqref{equation:tensor_subspace} and Equation~\eqref{equation:tilde_w_t}.
The two multi-step escape mechanisms correspond to distinct dominant components in this decomposition.
Specifically, the $\beta$-type escape arises when the $\mathbf{b}$-term becomes dominant, i.e., when $\|\tilde{\beta}^{(t)}\mathbf{b}\|$ separates from the other two terms.
Projecting this dominant direction back to the matrix space yields $\check{X}$ with $t\in U_\beta$.
Analogously, the $\gamma$-type escape corresponds to dominance of the $\mathbf{c}$-term, whose projection induces $\check{X}$ with $t\in U_\gamma$.
The complementary dependence of these two regimes on $l$ and $t$ is illustrated in Figure~\ref{figure:ablation}.

From an algorithmic perspective, a practical criterion for identifying a valid escape point is as follows.
During the simulated TPGD trajectory, if (1) either the $\mathbf{b}$-term or the $\mathbf{c}$-term exhibits clear norm separation from the remaining components, and (2) the corresponding projected point $\check{X}$ satisfies
$
h^l(\operatorname{vec}(\check{X})^{\otimes l})
\approx
h^l(\tilde{\mathbf{w}}^{(t)})
<
h^l(\hat{\mathbf{w}})
$,
then $\check{X}$ can be regarded as a valid multi-step SOD escape candidate after safely projection.
This behavior is empirically demonstrated in Figure~\ref{figure:simulation} by adopting basic example in Equation~\ref{equation:basic_example}.

\section{Single-step SOD Escape (Special Case)}
\label{sec4}

We first examine the single-step case, which clearly demonstrates the fundamental ideas and introduces the essential quantities.  
The proofs of Theorem~\ref{theorem:single-step-sod-informal} can be summarized in the sketch below:

\begin{itemize}
	\item We first show that, as presented in Section~\ref{subsec:rank1-approx-for-single-step}, the rank-2 tensor $\mathbf{w}_\mathrm{escape}$ defined in Equation~\eqref{equation:w_escape} admits a rank-1 approximation under the tensor spectral norm, i.e., there exists $\check{X} = \hat{X} + \hat{\rho} u_n q_r^\top$ for $\hat{\rho} > 0$ such that
	$
        \| \operatorname{vec}(\check{X})^{\otimes l} 
        - \mathbf{w}_{\mathrm{escape}} 
        \|_{S \Join R} \to 0.
    $
    
	\item By analyzing the descent condition $h(\check{X}) < h(\hat{X})$, we derive a constraint on the admissible range of $\hat{\rho}$, motivating the definition of the \textit{Escape Feasibility Score} (EFS). $\mathrm{EFS} > 1$ certifies the existence of a valid escape direction, as presented in Section~\ref{subsec:escape-inequality-and-efs}.
	
	\item Assuming that the measurement operator $\mathcal{A}$ is drawn from a Gaussian ensemble, we conduct a probabilistic analysis (see Section~\ref{subsec:probabilistic-escape-analysis}) of the event $\mathbb{P}(\mathrm{EFS} > 1)$ to quantify the likelihood of successful escape.
\end{itemize}

\subsection{Rank-1 Approximation for Single-step Escape}
\label{subsec:rank1-approx-for-single-step}

\begin{proposition}[Closed-form Structure of Single-step SOD Point $\check{X}$]
\label{proposition:single_step_rank1_tensor_approx}
	The rank-1 tensor $\operatorname{vec}(\check{X})^{\otimes l}$, which serves as a spectral-norm approximation of the tensor-space escape point $\mathbf{w}_\mathrm{escape}$ within the subspace
    $
        R := \operatorname{span}
		\left\{
		\operatorname{vec}(u_n q_r^\top)^{\otimes l},
		\operatorname{vec}(v_r q_r^\top)^{\otimes l}
		\right\}
    $,
	has the closed-form representation
	\begin{equation}
		\check{X} = 
		\sum\nolimits_{\phi=1}^{r-1} \sigma_\phi v_\phi q_\phi^\top 
		+ \sigma_r (\beta u_n 
		+ \alpha v_r) q_r^\top.
		\label{equation:single_step_rank1approx}
	\end{equation}
	The approximation error measured by the tensor spectral norm satisfies if $\sigma_r^l \beta^l \to \rho$ and $\alpha^l \to 1$,
	\begin{equation}
        \left\|
        \operatorname{vec}(\check{X})^{\otimes l} 
        - \mathbf{w}_\mathrm{escape}
        \right\|_{S \Join R}
        \to 0,
        \label{equation:spectral_norm_approx}
    \end{equation}
where $\|\cdot\|_{S\Join R}$ denotes the tensor spectral norm with projection onto $R$ (defined in Lemma~\ref{lemma:dual_tensor_norm}).
\end{proposition}
The proof is deferred to Appendix~\ref{append:proof_of_proposition_single_step_rank1_tensor_approx} and proceeds as follows. We first show that $\check{\mathbf{w}}$ approximates $\mathbf{w}_{\text{escape}}$ along the projection directions $\operatorname{vec}(u_n q_r^\top)^{\otimes l}$ and $\operatorname{vec}(v_r q_r^\top)^{\otimes l}$, respectively.
We then demonstrate that these components lie approximately within the tensor subspace $R$ spanned by these two orthonormal directions. The approximation is thus valid under the tensor spectral norm within $R$.

\begin{remark}
    Proposition~\ref{proposition:single_step_rank1_tensor_approx} indicates that the appropriate rank-1 approximation to be further explored takes the form $\operatorname{vec}(\check{X})^{\otimes l}$, where
    \begin{equation*}
		\check{X} 
		= 
		\sum\nolimits_{\phi=1}^{r-1} \sigma_\phi v_\phi q_\phi^\top 
		+ 
		\sigma_r \left(
		\nicefrac{\rho^{1/l}}{\sigma_r}\, u_n 
		+ v_r
		\right) q_r^\top
		= 
		\hat{X} + \rho^{1/l} u_n q_r^\top
		=:
		\hat{X} + \hat{\rho} u_n q_r^\top.
    \end{equation*}
\end{remark}

\subsection{Escape Inequality and Escape Feasibility Score}
\label{subsec:escape-inequality-and-efs}

\begin{proposition}
\label{proposition:escape_inequality}
	Let the single-step SOD escape point be defined as follows (with $\hat{\rho}>0$ is the escape step amplitude):
	$
        \check{X} 
		:= 
		\hat{X} + \hat{\rho} u_n q_r^\top
    $.
	Then, a sufficient condition for achieving descent, i.e., $h(\check{X}) < h(\hat{X})$, is that $\hat{\rho}$ satisfies the following quadratic escape inequality:
    \begin{equation}
        C_1 \hat{\rho}^2 + C_2 \hat{\rho} + C_3 < 0.
        \label{equation:quadratic_inequality}
    \end{equation}
    where $C_1 := 1 + \delta_p$, $C_2 := 2 \sigma_r \langle E, u_n u_n^\top \rangle$ and $C_3 := 2 \sigma_r^2 (1 + \delta_p) + 2 \lambda_n$.
\end{proposition}
The proof is deferred to Appendix~\ref{append:proof_of_proposition_escape_inequality} and proceeds by substituting $\check{X}$ into the descent condition. Expanding the difference and collecting the terms with respect to $\hat{\rho}$ yields quadratic inequality \eqref{equation:quadratic_inequality}.

This escape inequality takes the form of a quadratic equation in $\hat{\rho}$.  
Since $1+\delta_p>0$, the corresponding parabola opens upward. The constant term $2 \sigma_r^2 \left(1 + \delta_p\right) + 2 \lambda_n > 0$ holds according to \cite[Lemma~2]{ma2023noisy}.  
Therefore, a valid real-valued solution for $\hat{\rho}$ exists if the discriminant $\Delta$ of Equation~\eqref{equation:quadratic_inequality} is positive:
\begin{equation}
    0 < \Delta 
    := 
	4 \sigma_r^2 
	\langle
	E, u_n u_n^\top
	\rangle^2
	-
	4 (1+\delta_p)
	[
	2 \sigma_r^2 (1+\delta_p) - 2 (-\lambda_n)
	].
	\label{equation:quadratic_inequality_discriminant}
\end{equation}
Rearranging the inequality yields that an escape is feasible when the \emph{Escape Feasibility Score} (EFS) exceeds 1, which is defined as follows.

\begin{definition}[Escape Feasibility Score (EFS)]
\label{definition:efs}
\begin{equation}
    \underbrace{
        \frac{-\lambda_n}{\sigma_r^2 (1+\delta_p)}
    }_{:=\mathrm{NCM}}
    +
    \underbrace{
        \frac{1}{2(1+\delta_p)^2}
        \langle
        E, u_n u_n^\top
        \rangle^2
    }_{:=\mathrm{AIC}}
    \label{equation:escape_feasibility_score}
	\tag{EFS}
\end{equation}
\end{definition}

As shown above, the EFS consists of two components: the Negative Curvature Margin (NCM) and the Alignment-Induced Curvature (AIC).
NCM quantifies the ratio between the smallest eigenvalue $\lambda_n$ ($\lambda_n<0$) of $\nabla f(\hat{X}\hat{X}^\top)$ and the smallest non-zero eigenvalue $\sigma_r^2$ of $\hat{X}\hat{X}^\top$. 
AIC characterizes the alignment between the hidden escape direction matrix $E$ and the rank-1 projection matrix $u_n u_n^\top$ formed from $\lambda_n$'s eigenvector $u_n$.
Next, we verify the deterministic part of Theorem~\ref{theorem:single-step-sod-informal}.
\begin{figure*}[h]
	\centering
	\includegraphics[width=0.60\textwidth]{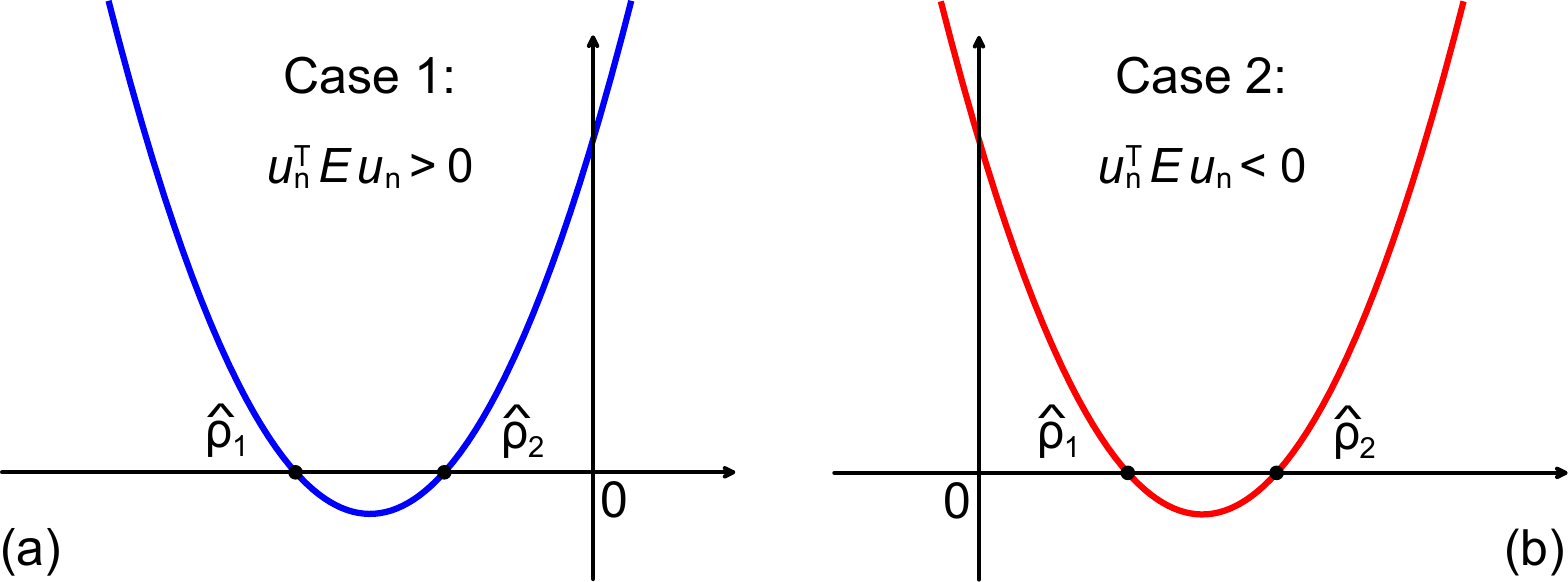}
	\caption{
		Visualization of quadratic inequality \eqref{equation:quadratic_inequality}.
	}
	\label{figure:quadratic_funcs}
\end{figure*}
\begin{proof}
	We now proceed to prove the deterministic part of Theorem~\ref{theorem:single-step-sod-formal} (as well as Theorem~\ref{theorem:single-step-sod-informal}).
	When the computed EFS value exceeds one, the discriminant satisfies $\Delta>0$, which guarantees that a real-valued $\hat{\rho}$ exists, leading to deterministic escape from the local minimum.
	
	We define the two roots of the quadratic equation~\eqref{equation:quadratic_inequality} as
    \begin{equation}
		\hat{\rho}_1 :=
		\frac{
			-2\sigma_r
			\langle
			E, u_n u_n^\top
			\rangle
			-
			\sqrt{\Delta}
		}
		{2(1+\delta_p)},
		\qquad
		\hat{\rho}_2 :=
		\frac{
			-2\sigma_r
			\langle
			E, u_n u_n^\top
			\rangle
			+
			\sqrt{\Delta}
		}
		{2(1+\delta_p)}.
		\label{equation:root1_root2}
    \end{equation}
	As depicted in Figure~\ref{figure:quadratic_funcs}, two distinct cases emerge based on the sign of the inner product $u_n^\top E u_n$:
	\begin{itemize}
		\item \textbf{Case 1:} If $\langle
		E, u_n u_n^\top
		\rangle>0$, then any $\hat{\rho}$ satisfying $\hat{\rho}_1<\hat{\rho}<\hat{\rho}_2<0$ leads to successful escape from $-\hat{X}$ via the update $\check{X} =
		-\hat{X} + \hat{\rho} u_n q_r^\top$.
		
		\item \textbf{Case 2:} If $\langle
		E, u_n u_n^\top
		\rangle < 0$, then any $\hat{\rho}$ satisfying $0<\hat{\rho}_1<\hat{\rho}<\hat{\rho}_2$ enables successful escape from $\hat{X}$ via the update $\check{X} = \hat{X} + \hat{\rho} u_n q_r^\top$.
	\end{itemize}
	This completes the proof.
\end{proof}

\subsection{Probabilistic Escape Analysis}
\label{subsec:probabilistic-escape-analysis}

\subsubsection{Gaussian Ensemble Modeling}
Consider independent symmetric measurement matrices $\{A_i\}_{i=1}^m$ drawn from a Gaussian ensemble, which is frequently applied in RIP-related proofs \citep{bah2010improved}:
\begin{equation}
    (A_i)_{jk} \sim 
    \begin{cases}
        \mathcal{N}(0, m^{-1}), & j = k, \\
        \mathcal{N}(0, (2m)^{-1}), & j < k,
    \end{cases}
    \label{equation:gaussian_ensemble}
\end{equation}
with symmetry imposed by $(A_i)_{jk} = (A_i)_{kj}$ for all $j \neq k$.
This choice ensures that, for arbitrary $M\in\mathbb{R}^{n\times n}$ we have $\mathbb{E}[\|\mathcal{A}(M)\|_2^2] = \|M\|_F^2$.
Let $P = u_n v_r^\top + v_r u_n^\top$, $Q = u_n u_n^\top$, $\psi_i = \langle A_i, P \rangle$ and $\omega_i = \langle A_i, Q \rangle$.
The AIC term in \eqref{equation:escape_feasibility_score} is obtained by defining $S_{\mathcal{A}}(P, Q) := \sum_{i=1}^{m} \psi_i \omega_i$, which yields
$
    \mathrm{AIC} 
	=
	\frac{1}{2(1+\delta_p)^2}
	S_{\mathcal{A}}^2(P,Q)
$.
The Gaussian ensemble is zero-mean, i.e. $\mathbb{E}[\psi_i] = \mathbb{E}[\omega_i] = 0$.  
By linearity of expectation and the orthogonality of $P$ and $Q$, we obtain
\begin{equation*}
    \mathbb{E}[S_{\mathcal{A}}(P, Q)]
	= 
	m \cdot \mathbb{E}[\langle A, P \rangle \langle A, Q \rangle] = 
	m \cdot m^{-1} \langle P, Q \rangle
	= 
	\langle P, Q \rangle
	=0.
\end{equation*}
Using the independence of $\{A_i\}_{i=1}^m$ and Isserlis' theorem \citep{isserlis1918formula} for Gaussian random variables, we compute the second moment as
\begin{equation*}
    \mathbb{E}[S_{\mathcal{A}}^2(P, Q)]
    =
    m \cdot \mathbb{E}[\langle A, P \rangle^2 \langle A, Q \rangle^2] 
	= 	m  \Big( 
	\mathbb{E}[\langle A, P \rangle^2]
	\mathbb{E}[\langle A, Q \rangle^2]
	+ 
	2\,\mathbb{E}[\langle A, P \rangle \langle A, Q \rangle]^2
	\Big).
\end{equation*}
Under the isotropy assumption, it follows that $\mathbb{E}[\langle A, P \rangle^2] = \|P\|_F^2 / m = 2 / m$ and $\mathbb{E}[\langle A, Q \rangle^2] = \|Q\|_F^2 / m = 1 / m$. Moreover, since $\langle P, Q \rangle = 0$, the cross term vanishes. Therefore,
\begin{equation*}
	\mathbb{E}[S_{\mathcal{A}}^2(P, Q)] = 
    m \cdot \mathbb{E}[\langle A, P \rangle^2 \langle A, Q \rangle^2] = 2m^{-1}.
\end{equation*}
To evaluate the concentration properties of $S_{\mathcal{A}}(P, Q)$, we compute its fourth moment.  
Since $\langle A, P \rangle$ and $\langle A, Q \rangle$ are independent, we have $\psi_i \sim \mathcal{N}(0, 2m^{-1})$ and $\omega_i \sim \mathcal{N}(0, m^{-1})$.  
Therefore, we obtain
$
    \mathbb{E}[\psi_i^2 \omega_i^2] 
    = \mathbb{E}[\psi_i^2] \mathbb{E}[\omega_i^2] = (2m^{-1})(m^{-1}) = 2m^{-2}
$
and 
$
    \mathbb{E}[\psi_i^4 \omega_i^4] 
    = \mathbb{E}[\psi_i^4] \mathbb{E}[\omega_i^4] = [3 (2m^{-1})^2] [3 (m^{-1})^2] = 36m^{-4}
$.
Using mutual independence across the indices, the fourth moment $\mathbb{E}[S_{\mathcal{A}}^4(P, Q)]$ becomes
\begin{equation*}
	\sum_{i=1}^m \mathbb{E}[\psi_i^4 \omega_i^4] + 3 \sum_{i \ne j} \mathbb{E}[\psi_i^2 \omega_i^2] \mathbb{E}[\psi_j^2 \omega_j^2]
	= m \cdot 36m^{-4} + 3m(m-1) \cdot (2m^{-2})^2
	= \frac{12(m + 2)}{m^3}.
\end{equation*}

\subsubsection{Statistical Lower Bound Analysis for EFS}
To obtain a probabilistic guarantee for escaping local minima $\hat{X}$, we make use of the Paley-Zygmund inequality \citep{paley1932note}.  
For any random variable $V \ge 0$ with finite second moment and any $0 < \theta < 1$, the inequality states:
\begin{equation}
	\mathbb{P}(V \geq \theta \mathbb{E}[V]) 
	\geq 
	(1 - \theta)^2 
	\frac{\mathbb{E}[V]^2}{\mathbb{E}[V^2]}.
	\label{equation:pz_inequality}
\end{equation}

The first and second moments of the AIC term can be written as
\begin{equation*}
    \begin{aligned}
        \mathbb{E}[\mathrm{AIC}] 
        &= \frac{1}{2(1+\delta_p)^2} \mathbb{E}[S_{\mathcal{A}}^2(P, Q)] = \frac{1}{(1+\delta_p)^2 m}, 
        \\
        \mathbb{E}[(\mathrm{AIC})^2] 
        &= \frac{1}{4(1+\delta_p)^4} \mathbb{E}[S_{\mathcal{A}}^4(P, Q)] = \frac{3(m + 2)}{(1+\delta_p)^4 m^3}.
    \end{aligned}
\end{equation*}

Applying the inequality~\eqref{equation:pz_inequality} to $V = \mathrm{AIC}$, we compute the moment ratio as
$
    \frac{\mathbb{E}[\mathrm{AIC}]^2}
	{\mathbb{E}[\mathrm{AIC}^2]} 
	= 
	\frac{m}{3(m + 2)}
$.
Hence, for any $0 < \theta < 1$, the Paley-Zygmund bound yields
\begin{equation*}
	\mathbb{P}\left( 
	\mathrm{AIC} \geq \theta \cdot \frac{1}{(1+\delta_p)^2 m} 
	\right) 
	\geq 
	(1 - \theta)^2 \cdot \frac{m}{3(m + 2)}.
\end{equation*}
As $m \to \infty$ and $\theta \to 0$, this lower bound approaches $1/3$. However, the bound is smaller for finite $m$, reflecting finite-sample effects in the ensemble behavior.
To guarantee escape from attraction basin of $\hat{X}$, the EFS must exceed $1$, giving us $\mathrm{AIC} > 1 - \mathrm{NCM}$.
According to \cite[Theorem~5.4]{ma2023over}, the NCM satisfies
$
	1 - \mathrm{NCM} 
	\le 
	1 - 
	\frac{1 - \delta_p}{1 + \delta_p} 
	\cdot 
	\frac{\| \hat{X}\hat{X}^\top - M^\star \|_F^2}
	{2\sigma_r^2 \operatorname{Tr}(M^\star)}
$.

To establish the lower bound on the probability of escape, we define the interval
$
    U_\theta := 
	( 0, \frac{1}{(1+\delta_p)^2 m} )
$,
which represents the range of values that $1 - \mathrm{NCM}$ can assume in comparison with the typical scale of the AIC term. Next, we verify the probabilistic part of Theorem~\ref{theorem:single-step-sod-formal}.

\begin{proof}
	We now proceed to prove the probabilistic part of Theorem~\ref{theorem:single-step-sod-formal} (as well as discussion after Theorem~\ref{theorem:single-step-sod-informal}).
	As illustrated in Figure~\ref{figure:r1_r2}, we examine three distinct regimes according to the relationship between $1 - \mathrm{NCM}$ and $U_\theta$:
	\begin{itemize}
		\item \textbf{Case 1: $1 - \mathrm{NCM} \le 0$.}  
		In this case, if we define 
        \begin{equation}
            r_2 := 2 \frac{1 + \delta_p}{1 - \delta_p} \sigma_r^2 \operatorname{Tr}(M^\star),
            \label{definition:r_2}
        \end{equation}
        then
		$\| \hat{X}\hat{X}^\top - M^\star \|_F^2 \ge r_2$.
		Since $\mathrm{AIC} \ge 0$ almost surely, $\mathrm{AIC} > 1 - \mathrm{NCM}$ always holds, leading to 
		$\mathbb{P}(\mathrm{EFS} > 1) = 1$.  
		This corresponds to iterations that are significantly away from $M^\star$, where escape is guaranteed regardless of random perturbations.
		
		\item \textbf{Case 2: $1 - \mathrm{NCM} > (1+\delta_p)^{-2}m^{-1}$.}  
		In this case, if we define 
        \begin{equation}
            r_1 := 2 
			\left[
			1 - \frac{1}{(1+\delta_p)^2 m} 
			\right]
			\frac{1 + \delta_p}{1 - \delta_p} 
			\sigma_r^2 \operatorname{Tr}(M^\star),
            \label{definition:r_1}
        \end{equation}
        then
        $\| \hat{X}\hat{X}^\top - M^\star \|_F^2 < r_1$.
		In this situation, $\hat{X}$ lies very close to $Z$.  
		The threshold required for AIC to trigger SOD escape becomes larger than its mean value, 
		so the present Paley-Zygmund argument provides no positive lower bound. Although this does not imply that the Gaussian tail probability is exactly zero,
		single-step SOD is disabled in this inner region and the algorithmic trigger probability is therefore zero by conservative design.
		
		\item \textbf{Case 3: $1 - \mathrm{NCM} \in U_\theta$.}  
		Here we parameterize the gap as
		$
			1 - \mathrm{NCM} 
			= 
			\frac{\theta}{(1+\delta_p)^2 m}
		$
        with $\theta \in (0,1)$.
		Therefore, the deviation is confined within $r_1 \le 
		\| \hat{X}\hat{X}^\top - M^\star \|_F^2 
		\le r_2$ because
        $
            \| \hat{X}\hat{X}^\top - M^\star \|_F^2 
			=
			2 \left[
			1 - \frac{\theta}{(1+\delta_p)^2 m} 
			\right]
			\frac{1 + \delta_p}{1 - \delta_p} 
			\sigma_r^2 \operatorname{Tr}(M^\star)
        $.   
		Within this bounded region, the escape probability is given by
        \begin{equation*}
			\mathbb{P}(\mathrm{AIC} > 1 - \mathrm{NCM}) 
			= \mathbb{P} 
			\left( 
				\mathrm{AIC} > 
				\frac{\theta}{(1+\delta_p)^2 m} 
			\right) 
			\geq (1 - \theta)^2 \cdot 
			\frac{m}{3(m + 2)}.
        \end{equation*}
	\end{itemize}
	This completes the proof.
\end{proof}
\begin{remark}
	When $r > r_2$, our characterization aligns with \cite[Theorem~1]{ma2024absence}.  
	Despite different derivations, both results show that $u_n q_r^\top$ is a valid descent direction in the matrix landscape, implying that $\hat{X}$ in this case is a saddle rather than a local minimum.  
	By contrast, the bound on the inner radius $r_1$ is not tight.  
	In particular, \cite[Theorem~4]{ma2024algorithmic} shows that under full tensor lifting, where all components are lifted rather than simulating over-parameterization, $r_1$ can approach $r_2/2$.  
	This gap indicates that our single-step SOD Escape captures only a subset of the geometric structure revealed by complete tensor lifting.
\end{remark}

\subsection{Formal Theorem and Limitations of the Single-Step SOD Escape}

Building on Sections~\ref{subsec:rank1-approx-for-single-step}, \ref{subsec:escape-inequality-and-efs}, and \ref{subsec:probabilistic-escape-analysis}, we can reformulate the informal version of Theorem~\ref{theorem:single-step-sod-informal} into its formal counterpart, Theorem~\ref{theorem:single-step-sod-formal}.
The single-step SOD escape strategy can be concluded as
\begin{equation}
	\hat{X}
	\xrightarrow{\text{Single-step SOD}}
	\hat{X} + \hat{\rho} u_n q_r^\top,
	\label{equation:single-step-sod}
\end{equation}

\begin{theorem}[Single-step SOD Escape, formal]\label{theorem:single-step-sod-formal}
	Consider the problem of escaping from a local minimum $\hat{X}$ of the objective function \eqref{equation:matrix_objfunc}.
	If the \ref{equation:escape_feasibility_score} defined in Definition~\ref{definition:efs} exceeds 1, then there exist two constants based on
	discriminant $\Delta$ (Equation \ref{equation:quadratic_inequality_discriminant})
    \begin{equation*}
		\hat{\rho}_1 :=
		\frac{
			-2\sigma_r
			\langle
			E, u_n u_n^\top
			\rangle
			-
			\sqrt{\Delta}
		}
		{2(1+\delta_p)}, 
		\qquad
		\hat{\rho}_2 :=
		\frac{
			-2\sigma_r
			\langle
			E, u_n u_n^\top
			\rangle
			+
			\sqrt{\Delta}
		}
		{2(1+\delta_p)},
    \end{equation*}
	such that for any $\hat{\rho}$ satisfying $\hat{\rho}_1 < \hat{\rho} < \hat{\rho}_2$, $\hat{X} + \hat{\rho} u_n q_r^\top$ is a valid escape point under spectral norm (Equation~\ref{equation:spectral_norm_approx}), ensuring that $h(\check{X}) < h(\hat{X})$. This indicates that the $\check{X}$ leaves the attraction basin of $\hat{X}$.
	Furthermore, suppose $\{A_i\}_{i=1}^m$ are drawn from a symmetric Gaussian ensemble as defined in Equation~\ref{equation:gaussian_ensemble}. Then, the escape behavior can be categorized into three cases, characterized by two concentric spheres with radii $r_1$ and $r_2$ defined as
    \begin{equation*}
		r_1 = 2 
		\left[
		1 - \frac{1}{(1+\delta_p)^2 m} 
		\right]
		\frac{1 + \delta_p}{1 - \delta_p} 
		\sigma_r^2 \operatorname{Tr}(M^\star),
		\qquad
		r_2 = 2
		\frac{1 + \delta_p}{1 - \delta_p} 
		\sigma_r^2 \operatorname{Tr}(M^\star).
    \end{equation*}
	\textbf{Case 1:}
	If $\| \hat{X}\hat{X}^\top - M^* \|_F^2 \geq r_2$, then $\mathbb{P}(\mathrm{EFS} > 1) = 1$.
	
	\textbf{Case 2:}
	If $\| \hat{X}\hat{X}^\top - M^* \|_F^2 < r_1$, then $\mathbb{P}(\mathrm{EFS} > 1) = 0$.
	
	\textbf{Case 3:}
	If $\| \hat{X}\hat{X}^\top - M^* \|_F^2 = \theta r_1 + (1-\theta) r_2$ for some $\theta \in (0,1]$, then the probability of escape satisfies
	$
		\mathbb{P}(\mathrm{EFS} > 1) 
		\geq 
		(1 - \theta)^2 \cdot \frac{m}{3(m + 2)}
	$.
		
		
\end{theorem}

\begin{proof}
	The proof follows directly from Proposition~\ref{proposition:single_step_rank1_tensor_approx}, Proposition~\ref{proposition:escape_inequality}, and the analysis (as well as its proof) presented in Section~\ref{subsec:probabilistic-escape-analysis}.
\end{proof}

While conceptually simple, single-step SOD escape mechanism has several limitations. It provides only an inexact approximation under the tensor spectral norm, which is restricted to the subspace $R$. Thus, the method strongly relies on the EFS exceeding 1. These motivate us to develop a multi-step SOD escape mechanism.

\section{Multi-step SOD Escape (General Case)}
\label{sec5}

\begin{algorithm}[h]
	\caption{Truncated Projected Gradient Descent (TPGD)}%
	\label{algorithm:tpgd}%
	\begin{algorithmic}[1]%
		\linespread{1.2}\selectfont%
		\Require Select $1 > \eta, \rho > 0$ sufficiently small. Set $\tilde{\mathbf{w}}^{(0)} \gets \mathbf{w}_\mathrm{escape}$
		\For{$t=0,1,2\ldots$}
		\State\label{state:gd-step}%
		Vanilla Gradient Descent Step: $\mathbf{v}^{(t)} \gets \tilde{\mathbf{w}}^{(t)} - \eta \nabla h^l (\tilde{\mathbf{w}}^{(t)})$
		\State\label{state:projection-step}%
		Projection Step: $\mathbf{w}^{(t+1)} \gets \Pi_S(\mathbf{v}^{(t)})$
		\State\label{state:truncation-step}%
		Truncation Step: (truncation residual is defined in \eqref{equation:tilde_w_1} and \eqref{equation:tilde_w_t+1})
        \begin{equation*}
            \tilde{\mathbf{w}}^{(t+1)} 
            \gets \mathbf{w}^{(t+1)} - \text{truncation residual}.
        \end{equation*}
		\EndFor
	\end{algorithmic}
\end{algorithm}

As explained previously, single-step SOD will only be successful when $\mathbf{w}_\mathrm{escape}$ has a low-dimensional counterpart, as guaranteed by a \ref{equation:escape_feasibility_score} score larger than 1. However, this is usually not the case, and we need to choose a $\mathbf{w}_\mathrm{escape}$ that corresponds to an iterate further along the optimization trajectory in the lifted tensor space. Nevertheless, without constraining the trajectory in the tensor space, the tensor rank can grow rapidly, making projection back to the matrix space nearly intractable. To address this issue, we carefully design a low-dimensional subspace in which the trajectory evolves, ensuring that it both decreases the lifted loss (enabling successful escape) and remains accessible to analytical projection. The subspace $S$ is formally introduced below:
\begin{equation}
	S = \operatorname{span}
	\left\{
	\underbrace{
		\operatorname{vec}(\hat{X})^{\otimes l}
	}_{:=\mathbf{a}\in\mathbb{R}^{nr\circ l}},
	\underbrace{
		\operatorname{vec}(u_n q_r^\top)^{\otimes l}
	}_{:=\mathbf{b}\in\mathbb{R}^{nr\circ l}},
	\underbrace{
		\operatorname{vec}(E \hat{X})^{\otimes l}
	}_{:=\mathbf{c}\in\mathbb{R}^{nr\circ l}}
	\right\},
	\label{equation:tensor_subspace}
\end{equation}
where $\mathbf{a},\mathbf{b},\mathbf{c}$ are order-$l$ hypercubic tensors ($l\geq 3$ is odd), and $E:=\mathcal{A}^*\mathcal{A}(
u_n v_r^\top + v_r u_n^\top)$.
The multi-step SOD escape simulates the Truncated Projected Gradient Descent (TPGD) procedure in the tensor space, originating from the matrix space. This procedure is described in Algorithm~\ref{algorithm:tpgd}.

As illustrated in Figure~\ref{figure:sod-escape}, TPGD adapts to the descent condition. Specifically, after a sufficiently large number of iterations $t$, we can guarantee that $h^l(\tilde{\mathbf{w}}^{(t)})<h^l(\hat{\mathbf{w}})$.
More importantly, TPGD yields a closed-form expression for $\tilde{\mathbf{w}}^{(t)}$, given by
$
\tilde{\mathbf{w}}^{(t)} = \tilde{\alpha}^{(t)}\mathbf{a} + \tilde{\beta}^{(t)}\mathbf{b} + \tilde{\gamma}^{(t)}\mathbf{c}
$, where $\tilde{\alpha}^{(t)}=1$, and $\tilde{\beta}^{(t)},\tilde{\gamma}^{(t)}$ are $t$-dependent coefficients. This structure enables us to deterministically amplify either (with its coefficient):
\begin{itemize}
	\item the escape component $\mathbf{b}$ identified in \cite[Theorem 5.4]{ma2023over}, or
	
	\item the principal superposition term $\mathbf{c}$ derived from the single-step analysis,
\end{itemize}
while simultaneously discarding other superposition components. 
Unlike the single-step warm-up which depends on \ref{equation:escape_feasibility_score}, our multi-step framework leverages the algebraic structure of $S$ and the closed-form coefficient dynamics to ensure successful escape. 
The entire section constitutes the proof of Theorem~\ref{theorem:multi-step-sod-informal}.
\begin{figure*}[h]
	\centering
	\includegraphics[width=1.0\textwidth]{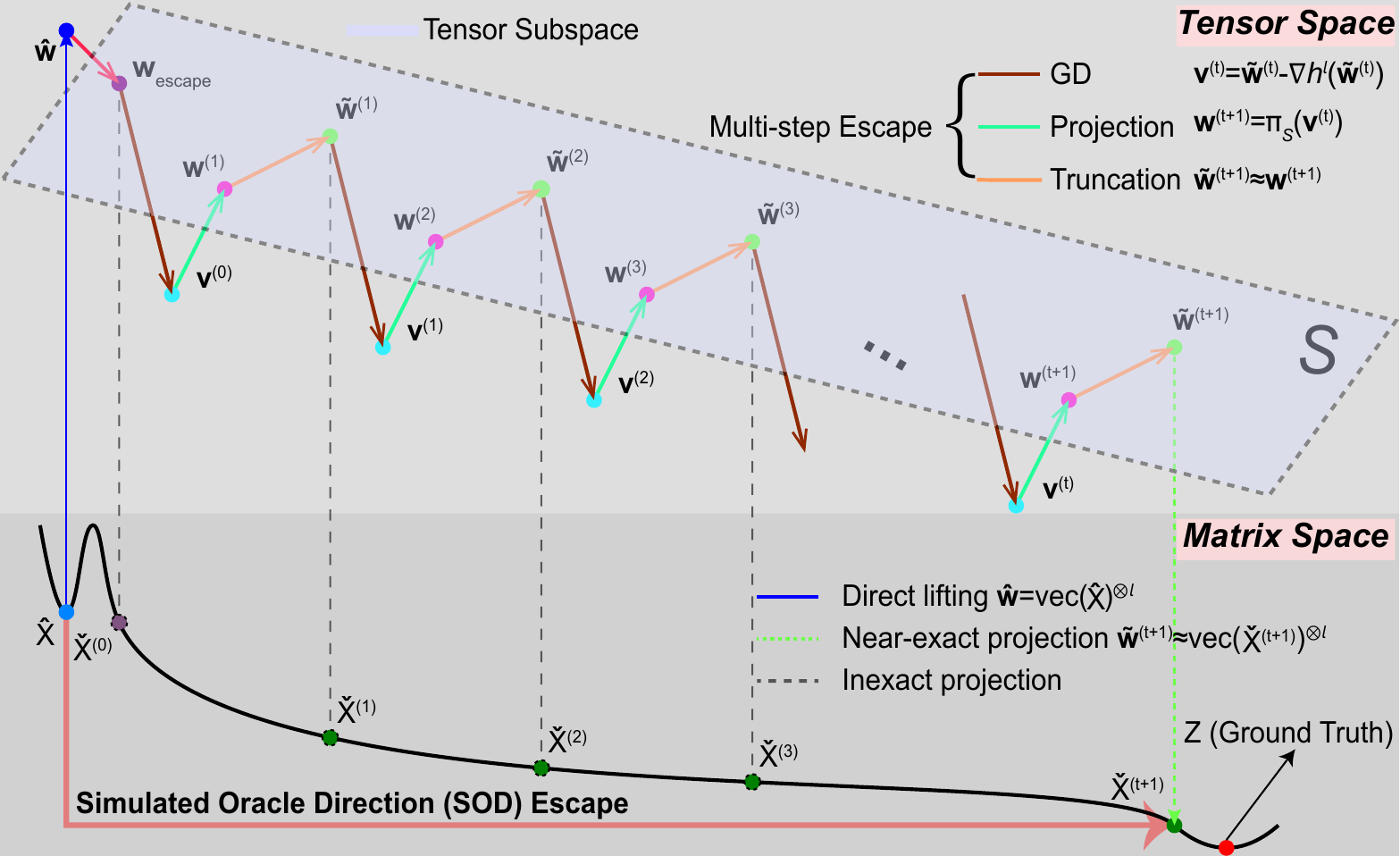}
	\caption{
		Simulating multi-step tensor-space PGD with truncation after each step, yielding a closed-form iterate $\check{X}^{(t)}$ that escapes the basin of $\hat{X}$ along the SOD direction (red transparent arrow).}
	\label{figure:sod-escape}
\end{figure*}

Section~\ref{subsec:projection-after-escape-and-PGD-preparation} introduces preliminary concepts for PGD. 
Section~\ref{subsec:closed_form_coefficient_dynamics_induced_by_tpgd} derives the closed-form expressions for the coefficient (i.e., $\mathbf{a}, \mathbf{b}, \mathbf{c}$) dynamics by using a mathematical induction framework.
The final two Section~\ref{subsec:sufficient-conditions-for-objective-decrease-under-truncation-error} and Section~\ref{subsec:dominance-criteria-for-the-coefficient-dynamics}, analyze conditions that guarantee successful escape via the multi-step SOD escape mechanism. 
In Section~\ref{subsec:sufficient-conditions-for-objective-decrease-under-truncation-error}, we establish a sufficient condition under which the objective function \eqref{equation:tensor_objfunc_h} decreases, even in the presence of truncation errors. Section~\ref{subsec:dominance-criteria-for-the-coefficient-dynamics} identifies conditions under which either $\tilde{\beta}^{(t)}\mathbf{b}$ or $\tilde{\gamma}^{(t)}\mathbf{c}$ dominates among the components $\{\tilde{\alpha}^{(t)}\mathbf{a}, \tilde{\beta}^{(t)}\mathbf{b}, \tilde{\gamma}^{(t)}\mathbf{c}\}$.

\subsection{Projection After Escape and PGD Preparation}
\label{subsec:projection-after-escape-and-PGD-preparation}

According to the definition of the subspace $S$ in Equation~\eqref{equation:tensor_subspace}, the corresponding projection matrix onto $S$ is defined as:
\begin{equation}
	P_S := C (C^\top C)^{-1} C^\top
	\in\mathbb{R}^{nrl\times nrl},
	\label{equation:projection matrix}
\end{equation}
where $C = [\operatorname{stack}(\mathbf{a}), \operatorname{stack}(\mathbf{b}), \operatorname{stack}(\mathbf{c})]
\in\mathbb{R}^{nrl\times 3}$.
Given any tensor $\mathbf{v}\in\mathbb{R}^{nr\circ l}$, its projection onto subspace $S$ is
$
	\Pi_S(\mathbf{v}) = \operatorname{unstack}
	(P_S\operatorname{stack}(\mathbf{v}))
$,

At each iteration $t$, the update rule of tensor space PGD is 
\begin{equation*}
	\mathbf{v}^{(t)} \gets \mathbf{w}^{(t)} - \eta \nabla h^l(\mathbf{w}^{(t)}),
	\qquad
	\mathbf{w}^{(t+1)} \gets \Pi_S(\mathbf{v}^{(t)}),
\end{equation*}
where $\mathbf{v}^{(t)}$ is the result of a vanilla GD on the smooth component $h^l$ with step size $\eta>0$.
The TPGD procedure is initialized from the post-escape tensor defined in Equation~\eqref{equation:w_escape}:
$
    \tilde{\mathbf{w}}^{(0)} \gets 
	\mathbf{w}_\text{escape} 
$,   
where $\rho>0$ is the escape step size. If no TPGD iterations are performed (i.e., zero steps), this reduces to the single-step SOD escape case.

To simplify the notation involved in tensor gradient derivations, we introduce the auxiliary function
\begin{equation}
	\mathcal{M}(\mathbf{s}, \mathbf{t}) 
	= 
	\langle
	\mathbf{P}(\mathbf{s}), \mathbf{P}(\mathbf{t})
	\rangle_{2*[l]}
	=
	\langle
	\langle
	\mathbf{P}^{\otimes l}, \mathbf{s}
	\rangle_{3*[l]},
	\langle
	\mathbf{P}^{\otimes l}, \mathbf{t}
	\rangle_{3*[l]}
	\rangle_{2*[l]},
    \label{equation:auxiliary-tensor-function}
\end{equation}
and define its self-pairing form as $\tilde{\mathcal{M}}(\mathbf{s}) := \mathcal{M}(\mathbf{s}, \mathbf{s})$.
Here, the permutation tensor 
$\mathbf{P}$ 
and the reshaping operator 
$\mathbf{P}(\cdot)$ 
are specified in Equation~\eqref{equation:permutation_tensor}, and the subscript notation for inner-product contractions follows Lemma~\ref{lemma:gradient-of-the-tensor-objective}.

For analytical clarity and algebraic convenience, we further define the Gram matrix 
$G = C^\top C$ associated with the tensor basis in subspace $S$ as
\begin{equation*}
	G :=
	\begin{bmatrix}
		\bar{a} & 0 & \bar{s} \\
		0 & \bar{b} & \bar{t} \\
		\bar{s} & \bar{t} & \bar{c}
	\end{bmatrix},
\end{equation*}
where $\bar{a} := \langle \mathbf{a}, \mathbf{a} \rangle$, $\bar{b} := \langle \mathbf{b}, \mathbf{b} \rangle$, $\bar{c} := \langle \mathbf{c}, \mathbf{c} \rangle$, $\bar{s} := \langle \mathbf{a}, \mathbf{c} \rangle$, $\bar{t} := \langle \mathbf{b}, \mathbf{c}
		\rangle$, $0 = \langle \mathbf{a}, \mathbf{b}
		\rangle$.
The cross-term $\langle \mathbf{a}, \mathbf{b} \rangle=0$ arises from $u_n \perp v_\phi$ for all $\phi=1,\ldots,r$.
The matrix $E$, which implicitly encodes the "invisible" escape information 
in the matrix space, is defined as
$
    E
    :=
	\mathcal{A}^*\mathcal{A} (v_r u_n^\top + u_n v_r^\top)
	=
	\mathcal{A}^*\mathcal{A} (v_r u_n^\top)
	+
	\mathcal{A}^*\mathcal{A} (u_n v_r^\top) 
	:= E_1 + E_2
$.
It follows that $E$ is symmetric and satisfies $E_1 = E_2$.

\subsection{Closed-Form Coefficient Dynamics Induced by TPGD}
\label{subsec:closed_form_coefficient_dynamics_induced_by_tpgd}

\begin{lemma}[Closed-Form Coefficient Dynamics]
\label{lemma:induction_for_the_closed_form_coefficient_dynamics}
    Suppose that $\rho$ and $\eta$ are chosen to be sufficiently small. Then the first TPGD iterate $\tilde{\mathbf{w}}^{(1)}$ takes the form
    $
        \tilde{\mathbf{w}}^{(1)}
		= 
		\mathbf{a}
		+
		\rho(1-\eta\lambda_n^l)
		\mathbf{b}
		-
		\eta\rho\sigma_r^l \frac{1}{2^{l-1}}
		\mathbf{c}
    $.
    Moreover, assume that, for an arbitrary $t \in \{1, 2, 3, \ldots\}$, the coefficients satisfy
    \begin{equation*}
        \begin{cases}
            \tilde{\alpha}^{(t)} &= 1, 
            \\
    		\tilde{\beta}^{(t)} &= 
    		\rho(1-\eta\lambda_n^l)^t,
            \\
    		\tilde{\gamma}^{(t)} &=
    		-\frac{\rho\eta}{2^{l-1}}
    		\left[
    		\sum\nolimits_{\tau=0}^{t-1}
    		(1-\eta \lambda_n^l)^\tau
    		\right]
    		\sigma_r^l.
        \end{cases}
    \end{equation*} 
	Equivalently,
    $
        \tilde{\mathbf{w}}^{(t)} = 
    	\tilde{\alpha}^{(t)} \mathbf{a} +
    	\tilde{\beta}^{(t)} \mathbf{b} +
    	\tilde{\gamma}^{(t)} \mathbf{c}
    $.
	Then the coefficients of the next TPGD iterate satisfy
    \begin{equation*}
        \begin{cases}
            \tilde{\alpha}^{(t+1)} &= 1, 
            \\
    		\tilde{\beta}^{(t+1)} &= 
    		\rho(1-\eta\lambda_n^l)^{t+1}, 
            \\
    		\tilde{\gamma}^{(t+1)} &=
    		-\frac{\rho\eta}{2^{l-1}}
    		\left[
    		\sum\nolimits_{\tau=0}^{t}
    		(1-\eta \lambda_n^l)^\tau
    		\right]
    		\sigma_r^l.
        \end{cases}
    \end{equation*}
	That is, the $(t+1)$-th TPGD iterate 
    $
        \tilde{\mathbf{w}}^{(t+1)}
		= 
		\tilde{\alpha}^{(t+1)} \mathbf{a} +
		\tilde{\beta}^{(t+1)} \mathbf{b} +
		\tilde{\gamma}^{(t+1)} \mathbf{c}
    $.
	This is an induction.
\end{lemma}

The proof is deferred to Propositions~\ref{proposition:first-tpgd-point} and~\ref{proposition:t+1-tpgd-point}. Together, these two propositions establish the induction step for the coefficient dynamics. This induction yields the following closed-form expression for $\tilde{\mathbf{w}}^{(t)}$:
\begin{equation}
	\begin{aligned}
		\tilde{\mathbf{w}}^{(t)}
		= {} & \underbrace{ \operatorname{vec}(\hat{X})^{\otimes l} }_{\mathbf{a}\text{-term }:\, \tilde{\alpha}^{(t)}\mathbf{a}}
		+ \underbrace{ \rho (1-\eta\lambda_n^l)^t \operatorname{vec}(u_n q_r^\top)^{\otimes l} }_{\mathbf{b}\text{-term }:\, \tilde{\beta}^{(t)}\mathbf{b}} + \\
		&\quad
		\underbrace{ - \frac{\eta\rho}{2^{l-1}} \left[ \sum\nolimits_{\tau=0}^{t-1} (1-\eta \lambda_n^l)^\tau \right] \sigma_r^l \operatorname{vec} ( E \hat{X} )^{\otimes l} }_{\mathbf{c}\text{-term }:\, \tilde{\gamma}^{(t)}\mathbf{c}}.
	\end{aligned}
	\label{equation:tilde_w_t}
\end{equation}

\subsection{Sufficient Conditions for Objective Decrease under Truncation Error}
\label{subsec:sufficient-conditions-for-objective-decrease-under-truncation-error}

\begin{proposition}[Sufficient Descent of TPGD]\label{proposition:sufficient-descent-of-tpgd}
	Let $\rho, \eta > 0$, and suppose the tensor objective function $h^l$ (Equation~\ref{equation:tensor_objfunc_h}) is $L$-smooth. Then there exist upper bounds on both $\rho$ and $\eta$ such that
    \begin{equation*}
        h^l(\tilde{\mathbf{w}}^{(t+1)}) < h^l(\tilde{\mathbf{w}}^{(t)}),
    \end{equation*}
	which implies each step of the TPGD algorithm guarantees a decrease in the tensor objective value.
\end{proposition}
The proof is deferred to Appendix~\ref{append:proof_of_proposition_sufficient_descent_of_tpgd}. It proceeds by building upon the results from Appendix~\ref{append:sufficient-decrease-condition-of-pgd} and thus by directly analyzing the sufficient decrease condition for the TPGD.

\subsection{Dominance Criteria for the Coefficient Dynamics}
\label{subsec:dominance-criteria-for-the-coefficient-dynamics}

In this subsection, we investigate the conditions under which one of the three terms ($\mathbf{a}$-term, $\mathbf{b}$-term, or $\mathbf{c}$-term) in Equation~\eqref{equation:tilde_w_t} becomes the leading term. Since the $\mathbf{a}$-term always corresponds to a strict saddle point in the tensor space that is lifted from a local minimum in the matrix space, it cannot serve as an escape direction. Consequently, we can rule out the possibility of the $\mathbf{a}$-term dominating the other two.
We therefore focus on the scenarios in which either the $\mathbf{b}$-term or the $\mathbf{c}$-term becomes dominant. For each case, we derive the corresponding range of the number of simulated TPGD steps $t$ required for dominance to occur, as well as the necessary order $l$ of tensor over-parameterization. Throughout this analysis, we assume that the constants $\rho$ and $\eta$ satisfy the sufficient descent condition for TPGD, as given in Equation~\eqref{equation:sufficient_decrease_TPGD_rho_eta}.

\begin{figure*}[h]
	\centering
	\includegraphics[width=0.46\textwidth]{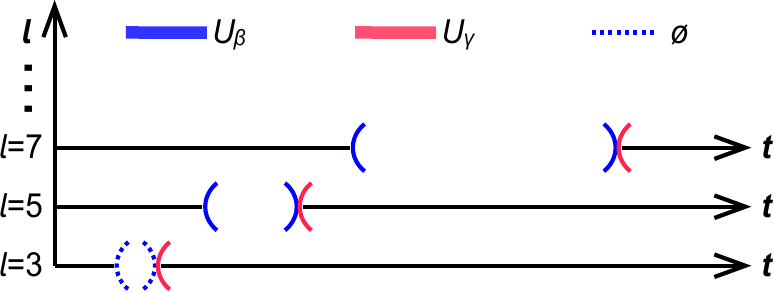}
	\caption{
		Illustration of the requirements on $l$ and $t$ for different leading terms with two distinct admissible regions \ref{equation:u_beta} and \ref{equation:u_gamma}.  
		When $l$ is small, $U_\beta$ may be empty if the condition in Equation~\eqref{equation:u_beta_validity} is not satisfied, 
		while $U_\gamma$ always exists. 
		As $l$ increases, $U_\beta$ gradually emerges based on Equation~\eqref{equation:u_beta_validality_l}, 
		allowing smaller values of $t$ for $\tilde{\beta}^{(t)}\mathbf{b}$ to dominate, 
		while larger values of $t$ for $\tilde{\gamma}^{(t)}\mathbf{c}$ to dominate.
	}
	\label{figure:ablation}
\end{figure*}

\begin{proposition}[Dominance Criteria]\label{proposition:dominance-criteria}
	There exist two disjoint intervals $U_\beta, U_\gamma \subset \mathbb{R}$ such that $\sup U_\beta < \inf U_\gamma$. When the order $l$ is small, there exists $t \in U_\gamma$ for which the term $\tilde{\gamma}^{(t)} \mathbf{c}$ dominates. In contrast, for $\tilde{\beta}^{(t)} \mathbf{b}$ to dominate, the order $l$ must be sufficiently large so that there exists $t \in U_\beta$.
\end{proposition}
The proof is deferred to Appendix~\ref{append:proof_of_proposition_dominance_criteria} and proceeds by directly comparing the magnitudes of 
$\tilde{\alpha}^{(t)} \mathbf{a}$, $\tilde{\beta}^{(t)} \mathbf{b}$, and $\tilde{\gamma}^{(t)} \mathbf{c}$. 
We summarize below two key remarks that follow from the proof.

\begin{remark}
	We observe that when $l$ is sufficiently large, $U_\beta$ is valid, and it is possible to find an iteration $t$ such that the $\mathbf{c}$-term becomes the leading component. Consequently, the associated escape point after multi-step SOD mechanism in the matrix space is given by:
	\begin{equation}
		\hat{X}
        \xrightarrow{\text{$\beta$-type}}
		\left\{
		\rho^{1/l} (1-\eta \lambda_n^l)^{t/l}
		\right\}
		\cdot u_n q_r^\top.
		\label{equation:x_check_beta}
	\end{equation}
	However, increasing $l$ may cause numerical instability due to the fractional power $1/l$. This approach is thus mainly applicable to simple cases, e.g. basic example in Equation~\eqref{equation:basic_example}, where the required value of $l$ is relatively small.
\end{remark}
\begin{remark}
	Since $\sup U_\beta < \inf U_\gamma$ and $U_\beta \cap U_\gamma = \emptyset$, as established by Equations \eqref{equation:u_beta} and \eqref{equation:u_gamma}, the terms $\tilde{\beta}^{(t)}\mathbf{b}$ and $\tilde{\gamma}^{(t)}\mathbf{c}$ cannot simultaneously serve as leading components.
	Consequently, when $\tilde{\gamma}^{(t)}\mathbf{c}$ dominates, the corresponding escape point after multi-step SOD mechanism in the matrix space is given by:
    \begin{equation}
        \begin{aligned}
            \hat{X}
            \xrightarrow{\text{$\gamma$-type}}
    		\left\{
    		- 
    		\frac{1}{2}
    		(2\eta\rho)^\frac{1}{l}
    		\left[
    		\sum\nolimits_{\tau=0}^{t-1}
    		(1-\eta \lambda_n^l)^\tau
    		\right]^{1/l}
    		\sigma_r
    		\right\}
    		E \hat{X}.
        \end{aligned}
        \label{equation:x_check_gamma}
    \end{equation}
	
	Note that when $\tilde{\gamma}^{(t)}\mathbf{c}$ becomes the leading term, it is not necessary for $l$ to exceed a specific lower bound. Instead, it only needs to satisfy the condition stated in \cite[Theorem 5.4]{ma2023over}.
	Hence, this scenario is particularly more practical.
\end{remark}

\subsection{Formal Theorem and Limitations of the Multi-Step SOD Escape}

Building on the analyses from Section~\ref{subsec:projection-after-escape-and-PGD-preparation}, Appendix~\ref{append:the-first-TPGD-step}, Appendix~\ref{append:the-t+1-TPGD-step}, Section~\ref{subsec:sufficient-conditions-for-objective-decrease-under-truncation-error} and Section~\ref{subsec:dominance-criteria-for-the-coefficient-dynamics}, we can reformulate the informal version of Theorem~\ref{theorem:multi-step-sod-informal} into its formal counterpart, Theorem~\ref{theorem:multi-step-sod-formal}.

It is important to note that the multi-step SOD escape mechanism is established by using oracle from tensor over-parameterization (proposed in \cite[Theorem~5.4]{ma2023over}). 
Therefore, Theorem~\ref{theorem:multi-step-sod-formal} naturally inherits certain conditions on the distance between $\hat{X}\hat{X}^\top$ and $M^\star$, 
which must be satisfied 
\begin{equation}
    r_2
	\ge
	\|\hat{X}\hat{X}^\top-M^\star\|_F^2
	\ge
	\frac{r_2}{2},
	\label{equation:strict_saddle_condition}
\end{equation}
where $r_2$ is defined in Equation~\eqref{definition:r_2}.
Additionally, the lifting order $l$ must satisfy the condition:
\begin{equation} 
	l > 
	\left\{
	1 -
	\log_2
	\left[
	\frac{
		2(1+\delta_p)
		\sigma_r^2\operatorname{Tr}(M^\star)
	}
	{(1-\delta_p) \|\hat{X}\hat{X}^\top-M^\star\|_F^2}
	\right]
	\right\}^{-1}.
	\label{equation:l_condition}
\end{equation}
This highlights the key limitation: the multi-step SOD escape mechanism is only applicable when $\hat{X}$ satisfies Equation~\eqref{equation:strict_saddle_condition}. A second limitation arises from the potential numerical instability introduced by the power of $1/l$, especially when $l$ is large.

\begin{theorem}[Multi-step SOD Escape, formal]
	\label{theorem:multi-step-sod-formal}
	Given a local minimum $\hat{X}$ of the objective function \eqref{equation:matrix_objfunc}, which satisfies the condition in Equation~\eqref{equation:strict_saddle_condition}, we aim to find a valid point $\check{X}$ that escapes the attraction basin of $\hat{X}$, i.e., $h(\check{X}) < h(\hat{X})$. 
	Let the lifting order $l$ satisfy the condition in Equation~\eqref{equation:l_condition}, and let the two step-sizes $\eta, \rho \in (0, 1)$ satisfy 
	Equation~\eqref{equation:sufficient_decrease_TPGD_rho_eta}.
	Then, the following holds:

	\textbf{Case 1:}
	If $l$ is sufficiently large such that $\rho > \rho_\text{min}$ (defined in Equation~\eqref{equation:u_beta_validity}), and the number of simulation steps $t \in U_\beta$ (defined in Equation~\eqref{equation:u_beta}) is relatively small, then the valid escape point
	\begin{equation}
		\check{X} = 
		\left\{
		\rho^{1/l} (1-\eta \lambda_n^l)^{t/l}
		\right\}
		\cdot u_n q_r^\top.
		\tag{$\beta$-type}
	\end{equation}
	However, this point may exhibit numerical instability due to expensive requirement of $l$.
		

	\textbf{Case 2:}
	If $l$ is not so large, a larger number of simulation steps $t \in U_\gamma$ (defined in Equation~\eqref{equation:u_gamma}) may be required. In this case, the valid escape point is
	\begin{equation}
		\check{X} = 
		\left\{
		- 
		\frac{1}{2}
		(2\eta\rho)^\frac{1}{l}
		\left[
		\sum\nolimits_{\tau=0}^{t-1}
		(1-\eta \lambda_n^l)^\tau
		\right]^{1/l}
		\sigma_r
		\right\}
		\cdot
		E \hat{X},
		\tag{$\gamma$-type}
	\end{equation}
	which exhibits greater numerical stability due to smaller $l$.
\end{theorem}
\begin{proof}
	The proof proceeds step by step, building upon the results established in 
    Lemma~\ref{lemma:induction_for_the_closed_form_coefficient_dynamics}, Proposition~\ref{proposition:sufficient-descent-of-tpgd} and Proposition~\ref{proposition:dominance-criteria}.
\end{proof}

\section{Numerical Experiments}
\label{sec6}

Our experiments focus on the core multi-step SOD escape mechanism. We first present two numerical case studies (see Appendix~\ref{append:case-1} and \ref{append:case-2}): (i) perturbed matrix completion (PMC), which uses a challenging PMC instance introduced in \citep{yalccin2022factorization} that contains many deep local minima; and (ii) a real-world case study~\citep{xu2024implicit}, in which the sensing matrices are obtained from a real system and spurious local minima arise during simulation.
We then conduct success-rate studies (see Section~\ref{subsection:success-rate-comparison}) on larger problem instances by running additional trials. Although these experiments focus on medium- to large-scale settings, they demonstrate that the proposed SOD escape mechanism can deterministically and verifiably escape spurious local minima in realistic matrix sensing tasks. To illustrate the applicability of SOD-Escape beyond matrix sensing, we further conduct an illustrative neural-network experiment (see Section~\ref{subsec:an_illustrative_neural_network_experiment}) based on training a quadratic neural network.
Additional numerical experiments can be found in Appendix~\ref{append:additional_details_of_section_6}.

\subsection{Two Case Studies}
\label{subsec:two-case-studies}

We consider the PMC case (see Appendix~\ref{append:case-1}) following the experimental protocol in \citep{ma2023over}. By setting $n = 3$ and $\epsilon = 0.3$, $z = \pm(1, 0, 1)$ serve as ground-truth factorizations, whereas $\hat{x} = \pm(1, 0, -1)$ are two primary spurious solutions.We adopt small initialization \citep{stoger2021small} for GD: draw $\tilde{x}$ from the standard normal distribution and set $x_0 = \zeta \tilde{x}$ with $\zeta = 0.01$. We fix the random seed so that GD converges to the neighborhood of one spurious solution. We run GD for $6000$ iterations prior to activating the multi-step SOD escape mechanism, and for an additional $1000$ iterations afterward; the step size is $0.001$ throughout.
As shown in Figure~\ref{figure:exp2}(f), under this setting GD converges to the vicinity of the spurious solution $(-1, 0, 1)$ rather than to either ground truth. 
Consistently, the curves in the first-row panels (a)–(c) also support this observation, prior to the application of SOD escape mechanism.
After applying the multi-step SOD, all three curves ultimately approach zero, indicating convergence to the neighborhood of $(1, 0, 1)$.
The second-row panels (d) and (e) detail the internal procedure of the multi-step SOD escape mechanism ($l = 11$). We ensure that after a finite number of simulated steps $t$, the tensor objective decreases, i.e., $h^{l}(\tilde{\mathbf{w}}^{(t)}) < h^{l}(\hat{\mathbf{w}})$ (as shown in panel (d)), where $\hat{\mathbf{w}}$ denotes the lift of the matrix-space spurious point. Panel (e) shows that, at iteration $59000$ when the multi-step SOD escape mechanism is triggered, the $\mathbf{c}$-term in Equation~\eqref{equation:tilde_w_t} dominates the other two and becomes the leading component; projecting this term back to the matrix space yields the desired closed-form escape point $\check{x}$. 

\begin{figure*}[h]
	\centering
	\includegraphics[width=0.92\textwidth]{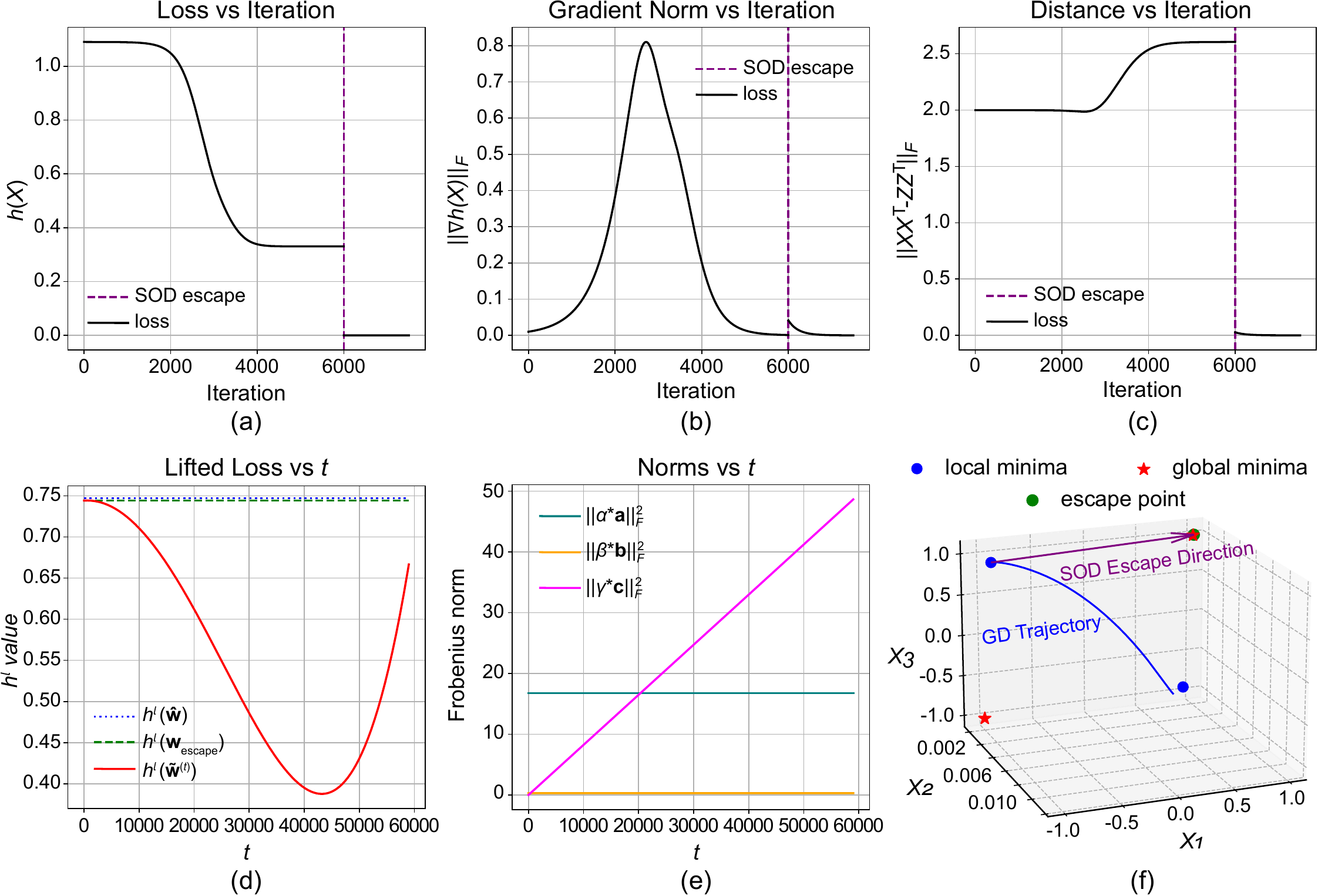}
	\caption{
		Experimental results of the PMC case. 
		\textbf{(a)}-\textbf{(c)}: Matrix objective function value / Frobenius norm of the gradient / Distance to ground truth vs. iteration. 
		\textbf{(d)}-\textbf{(e)}: Approximate tensor objective function value / Frobenius norms of the three tensors within subspace $S$ vs. the number of TPGD steps. 
		\textbf{(f)}: Three-dimensional representation of the trajectory of the optimization path.
	}
	\label{figure:exp2}
\end{figure*}

We consider another real-world MS case (see Appendix~\ref{append:case-2}) aimed at recovering a rank-1 low-rank matrix. The ground-truth factor 
$z = [1,0,0]^\top$.
In this numerical example, we randomly select the initialization and fix the random seed so that GD with step size $0.1$ converges to the local minimum $\hat{x}$. Figure~\ref{figure:exp3}(f) visualizes the 3D GD trajectory: prior to invoking the multi-step SOD escape mechanism, $100$ GD steps drive the iterate to the local minimum; the mechanism then computes an escaped point $\check{x}$, from which another $100$ GD steps with step size $0.1$ reach the ground truth $M^\star = zz^\top$. The first-row panels (a)--(c) corroborate this behavior.
As in the previous experiments, panels (d) and (e) depict the multi-step SOD escape procedure, which simulates TPGD in the lifted tensor space. With lifting order $l=5$ and $k=150{,}000$ simulated steps, we simultaneously achieve $h^{l}(\tilde{\mathbf{w}}^{(t)}) < h^{l}(\hat{\mathbf{w}})$ and observe that the $\mathbf{c}$-term dominates $\mathbf{a}$-term and $\mathbf{b}$-term, becoming the leading component. We then collapse $\tilde{\gamma}^{(t)} \mathbf{c}$ back to the matrix space to obtain the closed-form escape point $\check{x}$. As seen in panel (f), after applying the multi-step SOD escape mechanism, GD starting from $\check{x}$ travels a non-negligible distance and converges to the ground truth. This example further validates the effectiveness of the multi-step SOD escape mechanism. 

\begin{figure*}[h]
	\centering
	\includegraphics[width=0.92\textwidth]{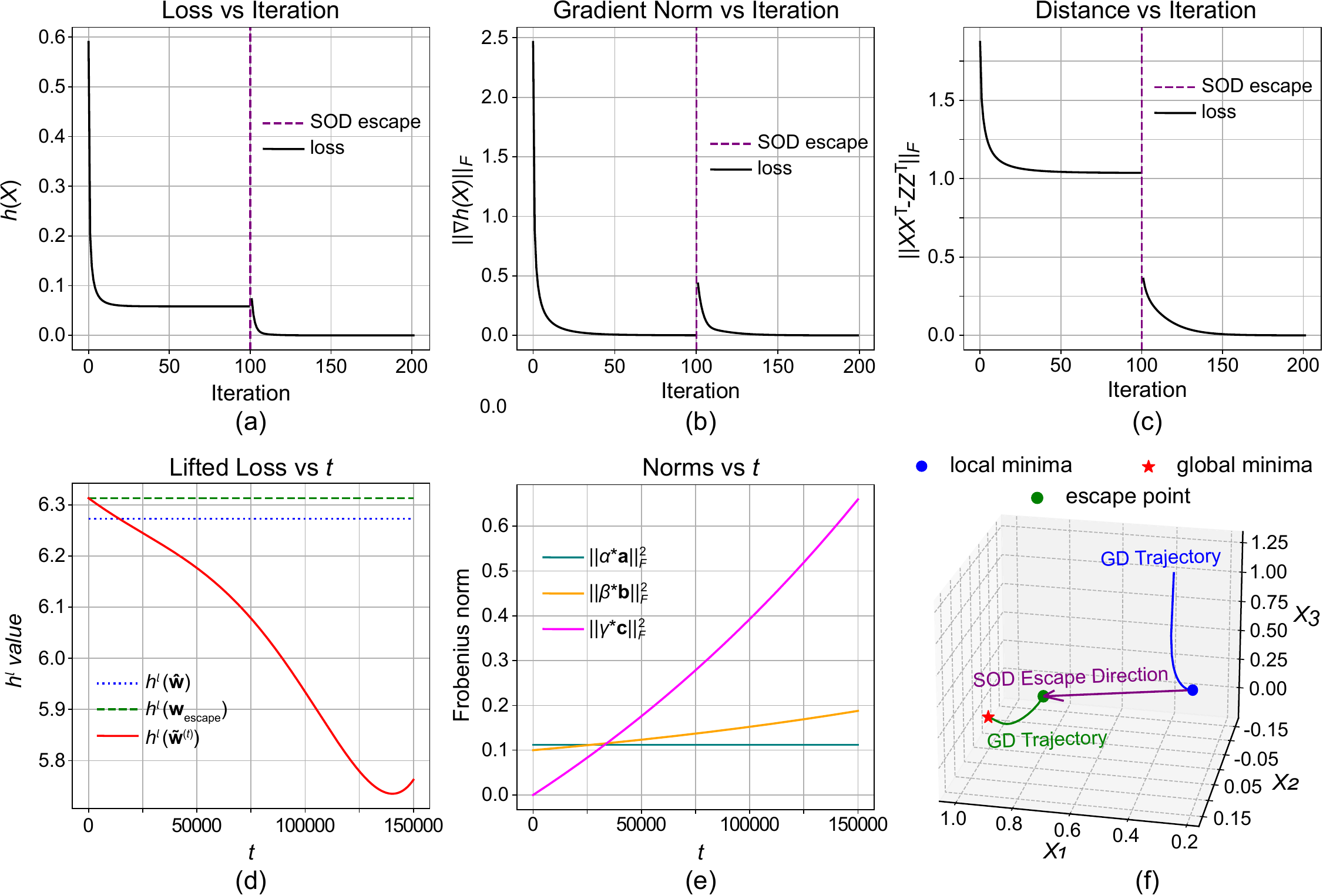}
	\caption{
		Experimental results of the real-world case. 
		\textbf{(a)}-\textbf{(c)}: Matrix objective function value / Frobenius norm of the gradient / Distance to ground truth vs. iteration. 
		\textbf{(d)}-\textbf{(e)}: Approximate tensor objective function value / Frobenius norms of the three tensors within subspace $S$ vs. the number of TPGD steps. 
		\textbf{(f)}: Three-dimensional representation of the trajectory of the optimization path.
	}
	\label{figure:exp3}
\end{figure*}

\subsection{Success Rate Comparison}
\label{subsection:success-rate-comparison}

\begin{figure*}[h]
	\centering
	\includegraphics[width=0.78\textwidth]{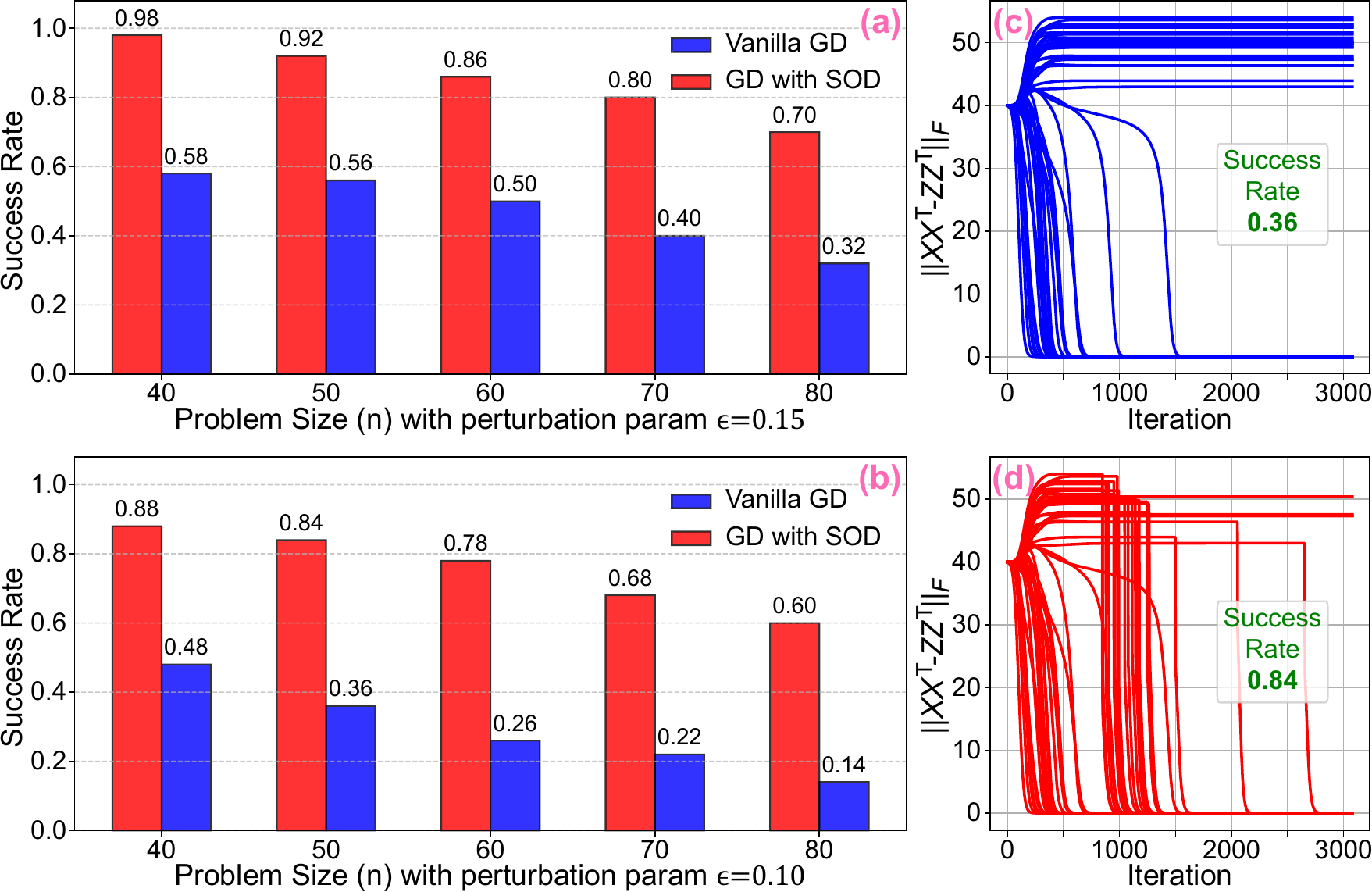}
	\caption{
		Panels (a) and (b) show the success rates under different problem sizes $n$ and different perturbation parameters $\epsilon$. 
		Panels (c) and (d) present the trajectories of $\| XX^\top - ZZ^\top \|_F$ over 50 trials across iterations, where (c) corresponds to GD and (d) corresponds to GD with SOD.
	}
	\label{figure:numerical_study}
\end{figure*}

In this study, we continue to use the PMC problem as our MS task. 
We evaluate the effects of increasing the problem size, characterized by the dimension $n\times n$ of the ground-truth matrix $M^\star$, as well as the impact of the perturbation magnitude $\epsilon$, on the convergence success rate. 
A trial is considered successful if the final solution satisfies $\|XX^\top - M^\star\|_F<0.02$, and the success rate is defined as the proportion of successful trials out of 50. The reason we use success rate instead of average reconstruction error is that failed trails will have very large errors, skewing the statistics by a large margin, distracting readers from the true message.
We test five different matrix sizes: $n=40,50,60,70,80$. For each size, we run 50 independent trials. The perturbation levels considered are $\epsilon=0.15$ and $\epsilon=0.10$. According to \citep{yalccin2022factorization}, smaller $\epsilon$ generally results in a larger RIP constant $\delta_p$, which in turn leads to a more complex optimization landscape \citep{zhang2019sharp}.
We evaluate the performance of two methods: vanilla GD and GD enhanced with our proposed multi-step SOD escape mechanism.
As shown in Figure~\ref{figure:numerical_study}, the SOD escape mechanism significantly increases the success rate of GD by effectively escaping many local minima. However, as the problem size increases or $\epsilon$ decreases, the optimization landscape becomes more intricate, leading to a general decrease in success rates for both methods.
Despite this decline, our method maintains a clear advantage over approaches that optimize entirely in the tensor space: optimization in tensor space is computationally infeasible for medium- to large-scale instances (e.g. $n>10$), as the variable $\mathbf{w}=\operatorname{vec}(X)^{\otimes l}$ grows in size to $(nr)^l$ with $X\in\mathbb{R}^{n\times r}$.

\subsection{An Illustrative Neural-Network Experiment}
\label{subsec:an_illustrative_neural_network_experiment}

We next provide an illustrative experiment showing that the proposed SOD-escape mechanism may also be useful beyond the standard MS formulation. 
In particular, we consider a neural-network (NN) training problem whose structure is closely related to BM-factorized MS~\citet{li2018algorithmic}. 
%
Motivated by this connection, we construct a one-hidden-layer quadratic neural network (QNN) and evaluate the proposed escape strategy in a binary classification task based on the Iris dataset \citep{fisher1936use}. 
We focus on distinguishing Versicolor from Virginica, 
where the training loss does not immediately decrease to zero and exhibits local basins where standard gradient descent stagnates. This makes the task suitable for examining whether SOD-escape can identify an effective descent direction after the baseline gradient method has stalled.

\begin{figure*}[h]
	\centering
	\includegraphics[width=0.85\textwidth]{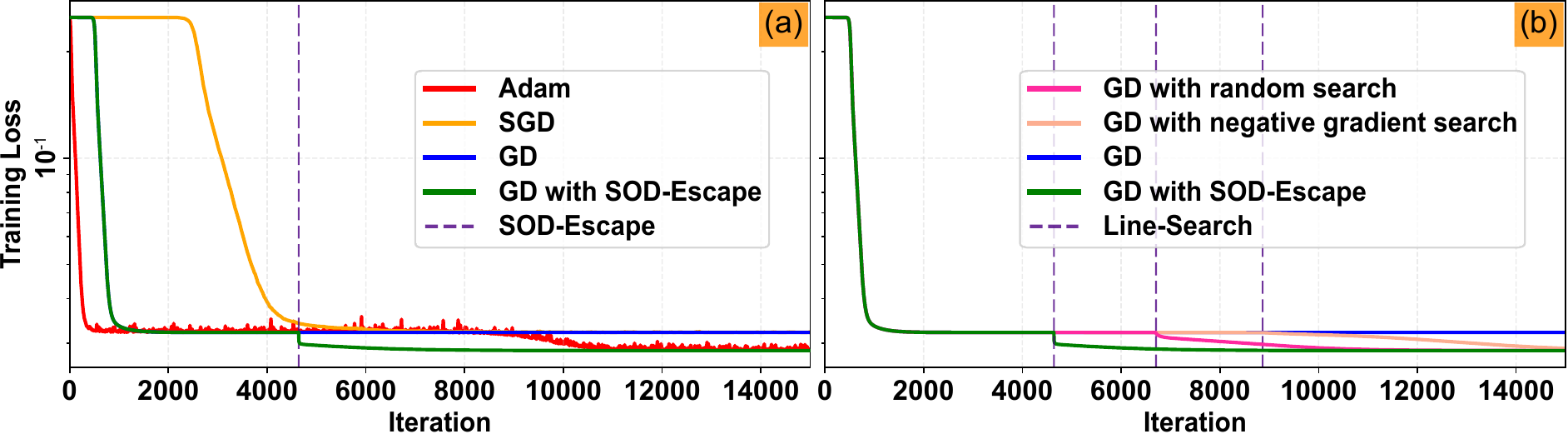}
	\caption{
		Training-loss curves for the QNN on the classification task. 
		Panel (a) compares GD-SOD with stochastic optimizers.
		Panel (b) compares different escape strategies after GD stagnates.
	}
	\label{figure:quadratic_nn}
\end{figure*}

Figure~\ref{figure:quadratic_nn} reports the training-loss curves. The baseline method is standard gradient descent, whereas our method applies the SOD-escape procedure once stagnation of GD is detected.
Panel~(a) compares GD-SOD with stochastic optimization methods, including SGD and Adam. In this experiment, both GD and SGD stagnate at a nonzero training loss and fail to escape the local basin within the plotted iterations. Adam is able to keep decreasing the loss, but the decrease is gradual. In contrast, once GD-SOD detects stagnation, the SOD-escape step produces a sharp drop in the training loss, reaching in a single escape step a loss level that Adam attains only after many additional iterations. 
%
Panel~(b) further investigates the role of the SOD direction as a deterministic line-search direction. In our construction, the candidate points in \eqref{equation:x_check_beta} and \eqref{equation:x_check_gamma} differ only through scalar coefficients multiplying prescribed matrix directions.
Therefore, the SOD-escape step can be viewed as a structured line search over a carefully chosen 
family of candidate escaping points. We compare this strategy with two alternatives: a random line search and a line search along the negative-gradient direction. Although these alternative searches can also find points with lower training loss, their improvements are substantially more gradual and require additional gradient iterations to reach comparable loss values. By contrast, the SOD-guided search produces a much larger immediate loss reduction.

Overall, this experiment provides a proof-of-concept demonstration that the proposed SOD-escape mechanism can be effective in a NN training problem.
While the present experiment focuses on a QNN motivated by the MS equivalence, 
it suggests a promising direction for extending structured escape mechanisms to broader classes of NN training problems via possibly approximate designs.

\section{Conclusions}
\label{sec7}
We proposed Simulated Oracle Direction (SOD) Escape, a deterministic framework for escaping spurious local minima in nonconvex matrix sensing via simulated tensor lifting—capturing the escape geometry of over-parameterized tensor spaces while operating entirely in the original matrix domain and avoiding explicit lifting costs. We provided two complementary analyses: a single-step mechanism that introduces the Escape Feasibility Score (EFS) to certify when a deterministic descent direction exists, and a multi-step mechanism that simulates truncated projected gradient descent (TPGD) within a structured subspace to produce closed-form escape points with explicit parameter dependence and provable objective decrease. Numerical experiments corroborate the theory, showing that SOD reliably escapes local minima and subsequently converges to the global solution across diverse instances. Overall, SOD offers a practical route to deterministic nonconvex optimization by converting over-parameterization insights into computable escape rules; future work may extend this simulated-lifting principle to broader nonconvex models and improve robustness to numerical precision at high lifting orders.

\bibliographystyle{unsrtnat}
\bibliography{reference}

\appendix

\renewcommand{\contentsname}{Appendix Contents} 
\setcounter{tocdepth}{2} 
\startcontents[appendix] 
\printcontents[appendix]{}{1}{\section*{Appendix Contents}} 

\section{Additional Details of Preliminary Knowledge}
\label{append:additional_details_of_section_preliminary_knowledge}
\subsection{Tensor Algebra Basics}

\begin{definition}[Tensor]
\label{definition:tensor}
	As a generalization of how vectors parametrize finite-dimensional vector spaces, tensors are parametrized using arrays formed from the Cartesian product of such spaces, as discussed in \citep{comon2008symmetric}. Specifically, an $l$-way array is defined as:
	\begin{equation*}
		\mathbf{s} = 
		\{s_{i_1 i_2 \dots i_l} 
		\mid 
		1 \leq i_k \leq n_k, 
		1 \leq k \leq l\} 
		\in 
		\mathbb{R}^{n_1 \times \cdots \times n_l}.
	\end{equation*}
	We treat tensors and arrays as synonymous, since there exists a natural isomorphism between them. 
	Furthermore, if $n_1 = \cdots = n_l = n$, we refer to the tensor (or array) as an $l$-order, $n$-dimensional tensor. For notational convenience, we use $\mathbb{R}^{n\circ l}$, where $n\circ l := n \times \cdots \times n$, repeated $l$ times. 
\end{definition}

\begin{definition}[Symmetric Tensor]
\label{definition:symmetric_tensor}
	Analogous to the definition of symmetric matrices, an order-$l$ tensor $\mathbf{s}$ with equal dimensions (i.e., $n_1 = \cdots = n_l$; also referred to as a hyper-cubic tensor) is said to be symmetric if its entries remain invariant under any permutation of indices:
	\begin{equation*}
		s_{i_{\sigma(1)} \cdots i_{\sigma(l)}} 
		= 
		s_{i_1 \cdots i_l},
	\end{equation*}
	for any $\sigma \in \mathcal{G}_l$, $i_1, \dots, i_l \in \{1, \dots, n\}$, where $\mathcal{G}_l$ denotes the symmetric group of all permutations on the set $\{1, \dots, l\}$. 
	The space of symmetric tensors is denoted by $\mathcal{S}^l(\mathbb{R}^n)$.
\end{definition}

\begin{definition}[Rank of Tensor]
\label{definition:rank_of_tensor}
	The rank of a cubic tensor $\mathbf{s} \in \mathbb{R}^{n\circ l}$ is defined as
	\begin{equation*}
		\operatorname{rank}(\mathbf{s}) = 
		\min 
		\left\{
		r \mid \mathbf{s} 
		= 
		\sum_{i=1}^r 
		a_i \otimes b_i \otimes c_i \otimes \cdots
		\right\},
	\end{equation*}
    where $a_i,b_i,c_i,\dots \in \mathbb{R}^n$ are finitely many vectors.
	Furthermore, as noted in \citep{kolda2015numerical}, if $\mathbf{s}$ is a symmetric tensor, it can be decomposed as
	\begin{equation*}
		\mathbf{s} = 
        \sum_{i=1}^r \varsigma_i a_i \otimes \cdots \otimes a_i
        := 
        \sum_{i=1}^r \varsigma_i a_i^{\otimes l},
	\end{equation*}
	where the rank is defined as the number of nonzero $\varsigma_i$, analogous to the rank of symmetric matrices.
	A central concept in this work is that of rank-1 tensors. A tensor $\mathbf{s}$ is rank-1 if and only if
	$
        \mathbf{s} = a^{\otimes l}
    $
	for some vector $a \in \mathbb{R}^n$.
\end{definition}

\begin{definition}[Tensor Multiplication]
\label{definition:tensor_multiplication}
	The outer product of two tensors, denoted by $\otimes$, is an operation that combines two tensors to produce a higher-order tensor. Specifically, the outer product of tensors $\mathbf{s}$ and $\mathbf{t}$, of orders $p$ and $q$ respectively, yields a tensor of order $p+q$, denoted as $\mathbf{o} = \mathbf{s} \otimes \mathbf{t}$, where
    $
        o_{i_1\cdots i_p j_1 \dots j_q} 
		= 
		s_{i_1\cdots i_p} t_{j_1 \cdots j_q}
    $.
	When both tensors have the same dimension, the outer product is defined as
	\begin{equation*}
		\otimes: \
		\mathbb{R}^{n \circ p} 
		\times 
		\mathbb{R}^{n \circ q} 
		\to \mathbb{R}^{n \circ (p+q)}.
	\end{equation*}
	For notational convenience, repeated outer products of a vector $a$ are written as
	\begin{equation*}
		\underbrace{
			a \otimes \cdots \otimes a
		}_{l \ \text{times}}
		\mathrel{:=}  a^{\otimes l}.
	\end{equation*}
	We also define the inner product between tensors. The mode-$l$ inner product between two tensors, assuming they share the same size in the $l$-th mode, is denoted by $\langle \mathbf{s}, \mathbf{t} \rangle_l$. Without loss of generality, assuming $l = 1$, the mode-$1$ inner product is defined as
	\begin{equation*}
		\left[ 
		\langle \mathbf{s}, \mathbf{t} \rangle_1 \right]_{i_2\cdots i_p j_2 \cdots j_q} 
		= 
		\sum_{\iota=1}^{n_1} 
		s_{\iota i_2\cdots i_p} 
		t_{\iota j_2 \cdots j_q}.
	\end{equation*}
	In $\langle \cdot, \cdot \rangle_l$, the summation is performed over the $l$-th mode of both tensors, where the total number of elements is $n_l$. 
	This definition can be naturally extended to multi-mode inner products by summing over several modes, denoted as $\langle \mathbf{s}, \mathbf{t} \rangle_{l_1, \cdots, l_m}$. 
	Such inner products are also referred to as tensor contractions, where the contraction is performed over modes $l_1, \cdots, l_m$.
\end{definition}

\begin{lemma}[Tensor Outer Product Identity]
\label{lemma:kronecker_identity}
	Let $P,Q,U,V$ be arbitrary matrices of compatible dimensions. Then the following identity holds:
	\begin{equation*}
		\langle 
		P \otimes Q, U \otimes V 
		\rangle_{2,4} = 
		PU \otimes QV.
	\end{equation*}
	A commonly used special case is:
	\begin{equation*}
		\langle P^{\otimes l}, Q^{\otimes l}\rangle
		= 
		\langle P, Q\rangle^l
		= 
		\operatorname{Tr}(P^\top Q)^l.
	\end{equation*}
\end{lemma}
\begin{proof}
	Please refer to Section 10.2 of \citep{petersen2008matrix} for a detailed derivation. This identity corresponds to the tensor outer product formulation, where the tensor outer product is equivalent to the unreshaped version of the Kronecker product.
\end{proof}

\begin{definition}[Tensor Norm]
\label{definition:tensor_norm}
	We adopt norm definitions consistent with those proposed in \citep{qi2019tensor}.
	For a cubic tensor $\mathbf{s} \in \mathbb{R}^{n \circ l}$, the nuclear norm $\|\cdot\|_*$ is defined as 
    \begin{equation*}
        \|\mathbf{s}\|_* = \inf \Biggl\{ 
        \sum_{j=1}^{s_m} |\varsigma_j| : 
        \mathbf{s} = \sum_{j=1}^{s_m} 
        \varsigma_j s_j^{\otimes l} 
        \Biggr\},
    \end{equation*}
    where $\|s_j\|_2 = 1$ and $s_j \in \mathbb{R}^{n}$.
    The spectral norm $\|\cdot\|_S$ is then defined as follows:
	\begin{equation*}
		\|\mathbf{s}\|_S 
		= \sup \Biggl\{ |\langle \mathbf{s}, a^{\otimes l} \rangle| : \|a\|_2 = 1, \, a \in \mathbb{R}^{n} \Biggr\}.
	\end{equation*}
\end{definition}

\begin{lemma}[Dual Relationship Between Tensor Norms] 
\label{lemma:dual_tensor_norm}
	The spectral norm $\|\cdot\|_S$ is the dual norm of the nuclear norm $\|\cdot\|_*$.
	That is, for any tensor $\mathbf{s}$, the following identity holds:
	\begin{equation*}
		\|\mathbf{s}\|_S 
		= 
		\sup_{\|\mathbf{t}\|_* \leq 1} 
		| \langle 
		\mathbf{s}, \mathbf{t} 
		\rangle |, 
	\end{equation*}
	where $\mathbf{t}$ is a tensor of the same dimensions as $\mathbf{s}$.
	In particular, the \textit{subspace spectral norm} of $\mathbf{s}$ over a subspace $R$ is defined as
	\begin{equation*}
		\|\mathbf{s}\|_{S \Join R}
		:= 
		\sup_{\mathbf{t} \in R,\ \|\mathbf{t}\|_* \leq 1} 
		| \langle 
		\mathbf{s}, \mathbf{t} 
		\rangle |.
	\end{equation*}
\end{lemma}
\begin{proof}
	Please refer to \cite[Lemma 21]{lim2013blind} for a detailed derivation..
\end{proof}

\subsection{Matrix Sensing Basics}

\begin{lemma}[Gradient of the Tensor Objective Function]
\label{lemma:gradient-of-the-tensor-objective}
	If $\mathbf{w} \in \mathbb{R}^{nr \circ l}$, then the gradient of $h^l(\mathbf{w})$ is given by: 
    \begin{equation}
        \nabla h^l (\mathbf{w}) 
		= 
		\Big\langle
		\langle
		\mathbf{A}_r^{\otimes l},
		\mathbf{w}
		\rangle_{[3l-1]},
		\langle
		\mathbf{A}^{\otimes l},
		\tilde{\mathcal{M}} (\mathbf{w})
		- 
		\tilde{\mathcal{M}} (\operatorname{vec}(Z)^{\otimes l})
		\rangle_{[3l-1,3l]}
		\Big\rangle_{[2l-1]},
		\label{equation:tensor_gradient}
    \end{equation}
	where $\mathbf{A}_r$ is defined as $I_r \oslash_{2,3} \mathbf{A} \in \mathbb{R}^{m \times nr \times nr}$, and $\oslash_{2,3}$ denotes the Kronecker product applied only to the last two dimensions of $\mathbf{A}$. 
	The operator $\tilde{\mathcal{M}}$ is defined in Equation~\eqref{equation:auxiliary-tensor-function}.
	The subscript $[3l-1, 3l]$ denotes tensor contraction along modes $2, 3, 5, 6, \ldots, 3l-1, 3l$; the interpretation of other subscripts follows similarly. For example, $2*[l]$ denotes the sequence $2,4,6,8,\ldots,2l$.
\end{lemma}
\begin{proof}
	Please refer to the proof of \cite[Lemma 5.2]{ma2023over}.
\end{proof}

\begin{lemma}
\label{lemma:upper_bound_A_star_A}
	For an arbitrary matrix $M\in\mathbb{R}^{n\times n}$, we have
	\begin{equation*}
		\|\mathcal{A}^*\mathcal{A}(M)\|_2^2
		\leq
		m \Xi^2 (1+\delta_p) \|M\|_F^2,
	\end{equation*}
	where $\Xi^2:=\max\{\|A_1\|_2^2, \cdots, \|A_m\|_2^2\}$. 
\end{lemma}
\begin{proof}
	By the definition of $\mathcal{A}^*\mathcal{A}$, we have
    \begin{equation*}
		\|\mathcal{A}^*\mathcal{A}(M)\|_2^2 
		= \Biggl\| \sum_{i=1}^{m} \langle A_i, M \rangle A_i \Biggr\|_2^2 
		\leq \left( \sum_{i=1}^{m} |\langle A_i, M \rangle|^2 \right) \left( \sum_{i=1}^{m} \|A_i\|_2^2 \right) 
		\leq m \Xi^2 \sum_{i=1}^{m} \langle A_i, M \rangle^2.
    \end{equation*}
	By the RIP property (please refer to Definition~\ref{definition:rip}),
    \begin{equation*}
        m \Xi^2 \sum_{i=1}^{m}
        \langle A_i, M\rangle^2
        =
        m \Xi^2
        \|\mathcal{A}(M)\|_2^2
        \leq
        m \Xi^2 (1+\delta_p) \|M\|_F^2.
    \end{equation*}
	This concludes the proof.
\end{proof}

\section{Additional Details of Section~\ref{sec4}}
\label{append:additional_details_of_section_4}
\subsection{Proof of Proposition~\ref{proposition:single_step_rank1_tensor_approx}}
\label{append:proof_of_proposition_single_step_rank1_tensor_approx}

\begin{proof}
	We begin by examining two key projection directions that characterize the approximation structure.
	We then show that any direction in the subspace spanned by them can be approximated.

	\proofstep{Orthogonality Between Tensor Projection Directions}
	Using first-order optimality condition, we have $u_n \perp v_\phi$ for all $\phi = 1, \dots, r$ and this yields:
    $
        \operatorname{vec}(u_n q_r^\top)^{\otimes l}
		\perp 
		\operatorname{vec}(v_r q_r^\top)^{\otimes l}
    $.
	Such two tensors thus form orthogonal basis vectors in the spectral subspace $R$, providing the foundation for subsequent analysis of the rank-1 approximation.
	
	\proofstep{Approximation in the Projection Direction of $\operatorname{vec}(u_n q_r^\top)^{\otimes l}$}
	To establish this result, we evaluate the approximation condition:
    \begin{equation*}
        \langle
		\operatorname{vec}(\check{X})^{\otimes l},
		\operatorname{vec}(u_n q_r^\top)^{\otimes l}
		\rangle
		\rightarrow
		\langle
		\hat{\mathbf{w}} + 
		\rho \operatorname{vec}(u_n q_r^\top)^{\otimes l},
		\operatorname{vec}(u_n q_r^\top)^{\otimes l}
		\rangle.
	\end{equation*}
	This is equivalent to approximating
	$
    	\langle
    	\check{X}^{\otimes l}, 
    	(u_n q_r^\top)^{\otimes l}
    	\rangle
	$ by
	$
    	\langle
    	\hat{X}^{\otimes l} + 
    	\rho (u_n q_r^\top)^{\otimes l},
    	(u_n q_r^\top)^{\otimes l}
    	\rangle
	$.
	By substituting the expression of $\check{X}$ and using Lemma~\ref{lemma:kronecker_identity}, we compute 
    $
        \langle
		\check{X}^{\otimes l}, 
		(u_n q_r^\top)^{\otimes l}
		\rangle
    $ as:
    \begin{equation*}
		\operatorname{Tr}
		\left\{
		\left[
		\sum\nolimits_{\phi=1}^{r-1}
		\sigma_{\phi} v_{\phi} q_{\phi}^\top
		+ \sigma_r
		(\alpha v_r + \beta u_n) q_r^\top
		\right]
		q_r u_n^\top
		\right\}^l
		=
		\operatorname{Tr}
		\left\{
		\sigma_r
		(\alpha v_r + \beta u_n) u_n^\top
		\right\}^l
		=
		\sigma_r^l \beta^l.
    \end{equation*}
	Similarly, we compute:
    \begin{equation*}
		\langle
		\hat{X}^{\otimes l} + 
		\rho (u_n q_r^\top)^{\otimes l},
		(u_n q_r^\top)^{\otimes l}
		\rangle
		=
		\operatorname{Tr}
		(
		\hat{X} q_r u_n^\top
		)^l
		+
		\rho
		\operatorname{Tr}
		(
		u_n q_r^\top q_r u_n^\top
		)^l
		= \rho.
    \end{equation*}
	Therefore, to satisfy the approximation condition, it must hold that $\sigma_r^l \beta^l \to \rho$.
	
	\proofstep{Approximation in the Projection Direction of $\operatorname{vec}(v_r q_r^\top)^{\otimes l}$}
	To establish this result, we evaluate the approximation condition:
    \begin{equation*}
        \langle
		\operatorname{vec}(\check{X})^{\otimes l},
		\operatorname{vec}(v_r q_r^\top)^{\otimes l}
		\rangle
		\rightarrow
		\langle
		\hat{\mathbf{w}} + 
		\rho \operatorname{vec}(u_n q_r^\top)^{\otimes l},
		\operatorname{vec}(v_r q_r^\top)^{\otimes l}
		\rangle.
	\end{equation*}
	The condition for this approximation can be derived through a similar argument, which leads to $\alpha^{l} \rightarrow 1$. 
	Since $l$ is odd (please refer to \cite[Theorem 5.4]{ma2023over}), it follows that $\alpha \rightarrow 1$.
	
	\proofstep{Approximation Guarantee in Terms of Tensor Spectral Norm}
	We now present the formal proof of Proposition~\ref{proposition:single_step_rank1_tensor_approx}.
	Define the residual tensor $\mathbf{d}_p$ as:
	\begin{equation*}
		\mathbf{d}_p := 
		\operatorname{vec}(\check{X})^{\otimes l}
		- 
		\left[
		\hat{\mathbf{w}} + \rho 
		\operatorname{vec}(u_n q_r^\top)^{\otimes l}
		\right].
	\end{equation*}
	By the definitions of the tensor spectral norm and nuclear norm (please refer to Definition~\ref{definition:tensor_norm}), 
	and since $u_n$, $v_r$, and $q_r$ are unit vectors, we have:
	$
        \|
		\operatorname{vec}(u_n q_r^\top)^{\otimes l}
		\|_* =
		\|
		\operatorname{vec}(v_r q_r^\top)^{\otimes l}
		\|_* = 1
    $.
    
	Let $\mathbf{d}_s\in R$ be the dual tensor defined in Lemma~\ref{lemma:dual_tensor_norm} which satisfies the conditions given in \cite[Lemma 10]{ma2024algorithmic}, and:
	\begin{equation*}
		\|\mathbf{d}_p\|_{S\Join R} = 
		|\langle \mathbf{d}_p, \mathbf{d}_s \rangle|, 
		\quad
		\|\mathbf{d}_s\|_* \leq 1, 
		\quad
		\mathbf{d}_s \in \mathbb{R}^{nr\circ l}.
	\end{equation*}
	Since $\operatorname{vec}(u_n q_r^\top)^{\otimes l}$ and $\operatorname{vec}(v_r q_r^\top)^{\otimes l}$ form an orthogonal basis of $R$, $\mathbf{d}_s$ must belong to their span. Therefore, there exist scalars $\lambda_{u_n}$ and $\lambda_{v_r}$ such that
    \begin{equation*}
        \mathbf{d}_s = \lambda_{u_n} 
		\operatorname{vec}(u_n q_r^\top)^{\otimes l}
		+ \lambda_{v_r} 
		\operatorname{vec}(v_r q_r^\top)^{\otimes l},
    \end{equation*}
	with $\left|\lambda_{u_n}\right| + \left|\lambda_{v_r}\right|\leq 1$.
	Using the approximation results along the projection directions $\operatorname{vec}(u_n q_r^\top)^{\otimes l}$ and $\operatorname{vec}(v_r q_r^\top)^{\otimes l}$, we compute $\langle \mathbf{d}_p, \mathbf{d}_s \rangle$ as:
    \begin{equation*}
		\langle \mathbf{d}_p, \mathbf{d}_s \rangle
		=
		\lambda_{u_n} 
		\langle \operatorname{vec}(\check{X})^{\otimes l} - \mathbf{w}_\mathrm{escape}, 
		\operatorname{vec}(u_n q_r^\top)^{\otimes l} \rangle
		+
		\lambda_{v_r} \langle \operatorname{vec}(\check{X})^{\otimes l} - \mathbf{w}_\mathrm{escape}, 
		\operatorname{vec}(v_r q_r^\top)^{\otimes l} \rangle,
    \end{equation*}
    where $\mathbf{w}_\mathrm{escape}$ is defined in Equation~\eqref{equation:w_escape}.
	Thus, the spectral norm of $\mathbf{d}_p$ is bounded by:
    \begin{equation*}
		\|\mathbf{d}_p\|_{S\Join R} 
		= |\langle \mathbf{d}_p, \mathbf{d}_s \rangle| 
		\leq
		|\lambda_{u_n}| \cdot |\sigma_r^l\beta^l - \rho| + |\lambda_{v_r}| \cdot \sigma_r^l \cdot |\alpha^l - 1|.
    \end{equation*}
	As $\operatorname{vec}(\check{X})^{\otimes l}$ becomes aligned with $\operatorname{vec}(u_n q_r^\top)^{\otimes l}$ and $\operatorname{vec}(v_r q_r^\top)^{\otimes l}$, we have $\sigma_r^l\beta^l \rightarrow \rho$ and $\alpha^l \rightarrow 1$, yielding $\|\mathbf{d}_p\|_{S\Join R} \rightarrow 0$.
	This concludes the proof.
\end{proof}

\subsection{Proof of Proposition~\ref{proposition:escape_inequality}}
\label{append:proof_of_proposition_escape_inequality}
\begin{proof}
	Our goal is to establish that
	$
	\|\mathcal{A} (\hat{X}\hat{X}^\top - M^\star)\|_2^2 
	>
	\|\mathcal{A} (\check{X}\check{X}^\top - M^\star)\|_2^2,
	$
	which can be rewritten as
	$
		\langle
		\mathcal{A}^*\mathcal{A} (\hat{X}\hat{X}^\top - M^\star)
		, \hat{X}\hat{X}^\top - M^\star
		\rangle>
		\langle
		\mathcal{A}^*\mathcal{A} (\check{X}\check{X}^\top - M^\star), \check{X}\check{X}^\top - M^\star
		\rangle
	  $.  
	Since the gradient vanishes at $\hat{X}$, we have $\mathcal{A}^*\mathcal{A}(\hat{X}\hat{X}^\top - M^\star)\hat{X} = 0$. Thus,
    \begin{equation}
		\underbrace{
			\langle
			\mathcal{A}^*\mathcal{A} (\check{X}\check{X}^\top - \hat{X}\hat{X}^\top), M^\star
			\rangle
		}_{:=S}
		>
		\underbrace{
			\langle
			\mathcal{A}^*\mathcal{A} (\check{X}\check{X}^\top - M^\star), \check{X}\check{X}^\top
			\rangle
		}_{:=T}.
        \label{equation:target_inequality}
    \end{equation}
	From the definition $\check{X} = \hat{X} + \hat{\rho} u_n q_r^\top$, we compute $\check{X}\check{X}^\top$ as
    \begin{equation*}
        \begin{aligned}
            &[ \hat{X} + \hat{\rho} (u_n q_r^\top) ] [ \hat{X} + \hat{\rho} (u_n q_r^\top) ]^\top 
    		= \hat{X} \hat{X}^\top + \hat{\rho} \hat{X} q_r u_n^\top + \hat{\rho} u_n q_r^\top \hat{X}^\top + \hat{\rho}^2 u_n q_r^\top q_r u_n^\top
    		\\
			=&\hat{X} \hat{X}^\top + \hat{\rho} \sigma_r ( v_r u_n^\top + u_n v_r^\top ) + \hat{\rho}^2 u_n u_n^\top.
        \end{aligned}
    \end{equation*}
	where we have used $\hat{X} q_r = \sigma_r v_r$ and $q_r^\top \hat{X}^\top = \sigma_r v_r^\top$. Substituting this decomposition into Equation~\eqref{equation:target_inequality}, we proceed with term-by-term expansion.

    \begin{table}[h]
		\centering
		\renewcommand{\arraystretch}{1.2} 
		\caption{Definitions of the nine interaction terms in the expansion of $T$.}
		\label{table:T_term_decomposition}
		\begin{tabular}{ll}
			\hline
			\textbf{Term} & \textbf{Definition} \\ \hline
			$T_1$ & $\langle \mathcal{A}^*\mathcal{A}(\hat{X} \hat{X}^\top - M^\star), \hat{X} \hat{X}^\top \rangle$ \\
			$T_2$ & $\langle \mathcal{A}^*\mathcal{A}(\hat{X} \hat{X}^\top - M^\star), \hat{\rho} \sigma_r (v_r u_n^\top + u_n v_r^\top) \rangle$ \\
			$T_3$ & $\langle \mathcal{A}^*\mathcal{A}(\hat{X} \hat{X}^\top - M^\star), \hat{\rho}^2 u_n u_n^\top \rangle$ \\
			$T_4$ & $\langle \hat{\rho} \sigma_r E, \hat{X} \hat{X}^\top \rangle$ \\
			$T_5$ & $\left\| \mathcal{A} ( \hat{\rho} \sigma_r (v_r u_n^\top + u_n v_r^\top) ) \right\|_2^2$ \\
			$T_6$ & $\langle \mathcal{A}^*\mathcal{A}(\hat{\rho} \sigma_r (v_r u_n^\top + u_n v_r^\top)), \hat{\rho}^2 u_n u_n^\top \rangle$ \\
			$T_7$ & $\langle \mathcal{A}^*\mathcal{A}(\hat{\rho}^2 u_n u_n^\top), \hat{X} \hat{X}^\top \rangle$ \\
			$T_8$ & $\langle \mathcal{A}^*\mathcal{A}(\hat{\rho}^2 u_n u_n^\top), \hat{\rho} \sigma_r (v_r u_n^\top + u_n v_r^\top) \rangle$ \\
			$T_9$ & $\langle \mathcal{A}^*\mathcal{A}(\hat{\rho}^2 u_n u_n^\top), \hat{\rho}^2 u_n u_n^\top \rangle$ \\ \hline
		\end{tabular}
	\end{table}
	
	\paragraph{Expansion of $S$ Term and $T$ Term.}
	The $S$ term decomposes as
    \begin{equation*}
		S 
		= 
		\langle
		\mathcal{A}^*\mathcal{A}
		(
		\hat{\rho} \sigma_r (v_r u_n^\top + u_n v_r^\top)
		+
		\hat{\rho}^2 u_n u_n^\top
		)
		, M^\star
		\rangle
		=
		\underbrace{
			\hat{\rho} \sigma_r
			\langle
			E, M^\star
			\rangle
		}_{:=S_1}
		+
		\underbrace{
			\hat{\rho}^2
			\langle
			\mathcal{A}^*\mathcal{A} (u_n u_n^\top)
			, M^\star
			\rangle
		}_{:=S_2}.
    \end{equation*}
	The $T$ term decomposes as a sum of nine interactions, the specific details of which are presented in Table~\ref{table:T_term_decomposition}.
    We repeatedly use the identity
    $
        \langle \mathcal{A} (M), \mathcal{A} (N) \rangle
        =
        \langle \mathcal{A}^*\mathcal{A} (M), N \rangle
        =
        \langle M, \mathcal{A}^*\mathcal{A} (N) \rangle
    $.
	
	\paragraph{Subcomponent Analysis.}
	Using first-order optimality and orthogonality properties, we summarize the simplifications in Table~\ref{table:subterm-analysis}.
    \begin{table}[h]
		\renewcommand{\arraystretch}{1.2}
		\centering
		\caption{Simplified expressions of all subterms.}
		\label{table:subterm-analysis}
		\begin{tabular}{cc}
			\toprule
			Term & Simplified Expression \\
			\midrule
			$T_1$ & 
			$\nabla f(\hat{X} \hat{X}^\top) 
			\hat{X} \hat{X}^\top = 0$ \\
			$T_2$ & 
			$2\hat{\rho} \sigma_r \lambda_n
			v_r^\top u_n = 0$ \\
			$T_5$ & 
			$\hat{\rho}^2 \sigma_r^2
			\|\mathcal{A}(v_r u_n^\top + u_n v_r^\top)
			\|_2^2$ \\
			$T_9$ & 
			$\hat{\rho}^4
			\|\mathcal{A}(u_n u_n^\top)
			\|_2^2$ \\
			$T_6+T_8$ & 
			$2 \hat{\rho}^3 \sigma_r
			\langle
			E, u_n u_n^\top
			\rangle$ \\
			$S_1-T_4$ & 
			$- \hat{\rho} \sigma_r
			\langle
			\nabla f (\hat{X} \hat{X}^\top), 
			v_r u_n^\top + u_n v_r^\top
			\rangle = 0$ \\
			$S_2-T_7$ & 
			$-\hat{\rho}^2
			\langle
			\nabla f (\hat{X} \hat{X}^\top), u_n u_n^\top
			\rangle
			= - \hat{\rho}^2 \lambda_n$ \\
			$-T_3$ & 
			$-\hat{\rho}^2
			\langle
			\nabla f (\hat{X} \hat{X}^\top),
			u_n u_n^\top
			\rangle
			= -
			\hat{\rho}^2 \lambda_n$ \\
			\bottomrule
		\end{tabular}
	\end{table}
	
	\paragraph{Derivation of the Escape Inequality.}
	Combining all corresponding terms yields the fact that
    \begin{equation*}
        (S_1 - T_4) + (S_2 - T_7) + (-T_3) > 
        T_1 + T_2 + T_5 + (T_6 + T_8) + T_9
    \end{equation*}
    leads to
    \begin{equation*}
		-2\lambda_n 
		>
		\sigma_r^2 \|\mathcal{A}(v_r u_n^\top + u_n v_r^\top)\|_2^2 
		+
		2\hat{\rho}\sigma_r \langle E, u_n u_n^\top \rangle 
		+
		\hat{\rho}^2 \|\mathcal{A}(u_n u_n^\top)\|_2^2.
    \end{equation*}
	Applying the RIP condition (please refer to Definition~\ref{definition:rip}) gives the strengthened inequality:
    \begin{equation*}
        -2\lambda_n
		>
		2\sigma_r^2(1 + \delta_p)
		+
		2\hat{\rho}\sigma_r
		\langle E, u_n u_n^\top \rangle
		+
		(1 + \delta_p)\hat{\rho}^2,
    \end{equation*}
	where $\|u_n v_r^\top + v_r u_n^\top\|_F^2 = 2$ and $\|u_n u_n^\top\|_F^2 = 1$ are used.  
	This directly corresponds to the quadratic form in $\hat{\rho}$ stated in Proposition~\ref{proposition:escape_inequality}, concluding the proof.
\end{proof}

\section{Additional Details of Section~\ref{sec5}}
\label{append:additional_details_of_section_5}
\subsection{The First Truncated Projected Gradient Descent Step}
\label{append:the-first-TPGD-step}

\begin{proposition}[The First TPGD Point]\label{proposition:first-tpgd-point}
	If $\rho$ is sufficiently small, then the first TPGD point $\tilde{\mathbf{w}}^{(1)}$ has the following form:
	\begin{equation*}
		\tilde{\mathbf{w}}^{(1)}
		= 
		\mathbf{a}
		+
		\rho(1-\eta\lambda_n^l)
		\mathbf{b}
		-
		\eta\rho\sigma_r^l \frac{1}{2^{l-1}}
		\mathbf{c}.
	\end{equation*}
\end{proposition}
The proof follows by combining the derivations in Appendix~\ref{append:tensor_gradient_at_the_first_pgd_step} and Appendix~\ref{append:first_pgd_point_and_its_truncated_approximation}.

\subsubsection{Tensor Gradient at the First PGD Step}
\label{append:tensor_gradient_at_the_first_pgd_step}
By substituting the initialization point $\tilde{\mathbf{w}}^{(0)}$ into $\tilde{\mathcal{M}}(\cdot)$, we obtain:
\begin{equation*}
	\begin{aligned}
		\tilde{\mathcal{M}} (\mathbf{w}^{(0)})
    	=& \langle \langle \mathbf{P}^{\otimes l}, \hat{\mathbf{w}} + \rho \operatorname{vec} (u_n q_r^\top)^{\otimes l} \rangle_{3*[l]}, \langle \mathbf{P}^{\otimes l}, \hat{\mathbf{w}} + \rho \operatorname{vec} (u_n q_r^\top)^{\otimes l} \rangle_{3*[l]} \rangle_{2*[l]}  
    	\\
		=& (\hat{X}\hat{X}^\top)^{\otimes l} + \rho \sigma_r^l (v_r u_n^\top)^{\otimes l} + \rho \sigma_r^l (u_n v_r^\top)^{\otimes l} + \rho^2 (u_n u_n^\top)^{\otimes l}.
	\end{aligned}
\end{equation*}

We now proceed to compute the tensor gradient $\nabla h^l (\tilde{\mathbf{w}}^{(0)})$ based on Lemma~\ref{lemma:gradient-of-the-tensor-objective}. Specifically, it becomes
$
    \left\langle 
    \langle \mathbf{A}_r^{\otimes l}, \tilde{\mathbf{w}}^{(0)} \rangle_{[3l-1]}, \langle \mathbf{A}^{\otimes l}, \tilde{\mathcal{M}} (\tilde{\mathbf{w}}^{(0)}) - \tilde{\mathcal{M}} (\operatorname{vec}(Z)^{\otimes l}) \rangle_{[3l-1,3l]} 
    \right\rangle_{[2l-1]}
$, which expands as
\begin{equation*}
    \begin{aligned}
        \Big\langle \langle \mathbf{A}_r^{\otimes l}, \hat{\mathbf{w}} + \rho \operatorname{vec}(u_n q_r^\top)^{\otimes l} \rangle_{[3l-1]}, 
        &\langle 
        \mathbf{A}^{\otimes l}, \tilde{\mathcal{M}} (\hat{\mathbf{w}}) - \tilde{\mathcal{M}} (\operatorname{vec}(Z)^{\otimes l}) 
        \rangle_{[3l-1,3l]} +
        \\
        &\langle 
        \mathbf{A}^{\otimes l}, 
        \rho \sigma_r^l (v_r u_n^\top)^{\otimes l} +
        \rho \sigma_r^l (u_n v_r^\top)^{\otimes l} +
		\rho^2 (u_n u_n^\top)^{\otimes l} 
        \rangle_{[3l-1,3l]}
        \Big\rangle_{[2l-1]}.
    \end{aligned}
\end{equation*}
This gradient can be naturally decomposed into several sub-components.

\paragraph{The First Sub-Term $H_1$.}
\begin{equation*}
	H_1 = 
	\Big\langle
	\langle
	\mathbf{A}_r^{\otimes l},
	\hat{\mathbf{w}}
	\rangle_{[3l-1]}, 
	\langle
	\mathbf{A}^{\otimes l},
	\tilde{\mathcal{M}} (\hat{\mathbf{w}})
	- 
	\tilde{\mathcal{M}} (\operatorname{vec}(Z)^{\otimes l})
	\rangle_{[3l-1,3l]}
	\Big\rangle_{[2l-1]}.
\end{equation*}   
$H_1 = 0$ because $H_1=\nabla h^l (\hat{\mathbf{w}})$ where $\mathbf{a}=\hat{\mathbf{w}}=\operatorname{vec}(\hat{X})^{\otimes l}$ is a saddle point after tensor lifting.

\paragraph{The Second Sub-Term $H_2$.}
\begin{equation*}
	H_2 := \Bigl\langle 
	\langle 
	\mathbf{A}_r^{\otimes l}, \hat{\mathbf{w}} 
	\rangle_{[3l-1]}, 
	\langle 
	\mathbf{A}^{\otimes l}, 
	\rho \sigma_r^l (v_r u_n^\top)^{\otimes l} +
	\rho \sigma_r^l (u_n v_r^\top)^{\otimes l} +
	\rho^2 (u_n u_n^\top)^{\otimes l} 
	\rangle_{[3l-1,3l]}
	\Bigr\rangle_{[2l-1]}
\end{equation*}
which we decompose as $H_2 := H_2^1 + H_2^2 + H_2^3$. 
For the first term, we compute $H_2^1$ as:
\begin{equation*}
    \begin{aligned}
        &\Bigl\langle \langle \mathbf{A}_r^{\otimes l}, \hat{\mathbf{w}} \rangle_{[3l-1]}, \langle \mathbf{A}^{\otimes l}, \rho \sigma_r^l (v_r u_n^\top)^{\otimes l} \rangle_{[3l-1,3l]} \Bigr\rangle_{[2l-1]} \\
        =& \rho \sigma_r^l 
        \left\langle 
        \begin{bmatrix} 
        A_1 \hat{X} \\ \vdots \\ A_m \hat{X} \end{bmatrix}^{\otimes l}, 
        \begin{bmatrix} 
        \langle A_1, v_r u_n^\top \rangle \\ \vdots \\ \langle A_m, v_r u_n^\top \rangle 
        \end{bmatrix}^{\otimes l} 
        \right\rangle_{[2l-1]} 
        = \rho \sigma_r^l 
        \operatorname{vec}(E_1 \hat{X})^{\otimes l}.
    \end{aligned}
\end{equation*}
Analogously, we can obtain
\begin{equation*}
	H_2^2 =
	\rho \sigma_r^l
	\operatorname{vec}(E_2 \hat{X})^{\otimes l},
	\qquad
	H_2^3 = 
	\rho^2
	\operatorname{vec}\left(
	\sum_{i=1}^{m}
	\langle A_{i}, u_n u_n^\top \rangle
	A_{i} \hat{X}
	\right)^{\otimes l}.
\end{equation*}
Since the sensing matrices $A_i$ are symmetric, we have $\langle
A_{i}, u_n v_r^\top\rangle=\langle
A_{i}, v_r u_n^\top\rangle$, implying $H_2^1=H_2^2$. Therefore, the second sub-term simplifies to:
\begin{equation*}
	H_2 
    =
	\rho \sigma_r^l
	\frac{1}{2^{l-1}}
	\operatorname{vec}(E \hat{X})^{\otimes l} 
	+
	\rho^2
	\operatorname{vec}\left(
	\sum_{i=1}^{m}
	\langle A_{i}, u_n u_n^\top \rangle
	A_{i} \hat{X}
	\right)^{\otimes l}.
\end{equation*}

\paragraph{The Third Sub-Term $H_3$.}
Through direct expansion, we compute $H_3$ as:
\begin{equation*}
    \begin{aligned}
        &\Bigl\langle 
        \langle \mathbf{A}_r^{\otimes l}, \rho \operatorname{vec}(u_n q_r^\top)^{\otimes l} \rangle_{[3l-1]}, 
        \langle \mathbf{A}^{\otimes l}, \tilde{\mathcal{M}} (\hat{\mathbf{w}}) - \tilde{\mathcal{M}} (\operatorname{vec}(Z)^{\otimes l}) \rangle_{[3l-1,3l]} 
        \Bigr\rangle_{[2l-1]} 
        \\
        =& \rho \left\langle 
        \begin{bmatrix} A_1 u_n q_r^\top \\ \vdots \\ A_m u_n q_r^\top 
        \end{bmatrix}^{\otimes l}, 
        \begin{bmatrix} \langle A_1, \hat{X}\hat{X}^\top - ZZ^\top \rangle \\ \vdots \\ \langle A_m, \hat{X}\hat{X}^\top - ZZ^\top \rangle \end{bmatrix}^{\otimes l} 
        \right\rangle_{[2l-1]} 
        \\
        =& \rho \operatorname{vec} \left( \nabla f(\hat{X}\hat{X}^\top) u_n q_r^\top \right)^{\otimes l} 
        = \rho \lambda_n^l \operatorname{vec}(u_n q_r^\top)^{\otimes l}.
    \end{aligned}
\end{equation*}
Here, $\lambda_n<0$ denotes the smallest eigenvalue of the matrix gradient evaluated at $\hat{X}$, and $u_n$ is the corresponding eigenvector. 

\paragraph{The Fourth Sub-Term $H_4$.}
Applying analogous logic to the preceding derivations, we can compute $H_4$ as:
\begin{equation*}
    \Bigl\langle 
    \langle \mathbf{A}_r^{\otimes l}, 
    \rho \operatorname{vec}(u_n q_r^\top)^{\otimes l} \rangle_{[3l-1]}, 
    \langle 
    \mathbf{A}^{\otimes l}, 
    \rho \sigma_r^l (v_r u_n^\top)^{\otimes l} + 
    \rho \sigma_r^l (u_n v_r^\top)^{\otimes l} + 
    \rho^2 (u_n u_n^\top)^{\otimes l} 
    \rangle_{[3l-1,3l]} 
    \Bigr\rangle_{[2l-1]},
\end{equation*}
which can be further computed as
\begin{equation*}
	H_4 =
	\rho^2 \sigma_r^l \frac{1}{2^{l-1}} 
	\operatorname{vec}
	\left( E u_n q_r^\top \right)^{\otimes l} 
	+
	\rho^3 \operatorname{vec}
	\left( 
	\sum_{i=1}^{m} 
	\langle A_{i}, u_n u_n^\top \rangle 
	A_{i} u_n q_r^\top 
	\right)^{\otimes l}.
\end{equation*}

\subsubsection{First PGD Point and Its Truncated Approximation}
\label{append:first_pgd_point_and_its_truncated_approximation}
The intermediate point $\mathbf{v}^{(0)}$, obtained after one step of GD, is given by:
\begin{equation*}
    \begin{aligned}
        \mathbf{v}^{(0)}
		=&
		\tilde{\mathbf{w}}^{(0)} - \eta \nabla h^l (\tilde{\mathbf{w}}^{(0)}) 
        = \operatorname{vec} (\hat{X})^{\otimes l} + \rho \operatorname{vec} (u_n q_r^\top)^{\otimes l} - \eta \left[ H_2 + H_3 + H_4 \right]
        \\
		=& \operatorname{vec}(\hat{X})^{\otimes l} 
        + \rho (1-\eta \lambda_n^l) 
        \operatorname{vec}\left(u_n q_r^\top\right)^{\otimes l} 
        - \eta \rho \sigma_r^l \frac{1}{2^{l-1}} \operatorname{vec}(E \hat{X})^{\otimes l}
        + O(\rho^2),
    \end{aligned}
\end{equation*}
or equivalently,
$
    \mathbf{a} + \rho (1-\eta \lambda_n^l) \mathbf{b} - \eta \rho \sigma_r^l \frac{1}{2^{l-1}} \mathbf{c} + O(\rho^2)
$.

Since the projection operator $\Pi_S$ maps any point onto the subspace $S=\operatorname{span}\{\mathbf{a},\mathbf{b},\mathbf{c}\}$, and the vectors $\operatorname{vec}(\mathbf{a}),\operatorname{vec}(\mathbf{b}),\operatorname{vec}(\mathbf{c})\in\mathbb{R}^{nrl}$ are linearly independent by construction, the projected point $\mathbf{w}^{(1)}\in S$ can be computed as:
$\mathbf{w}^{(1)} = \Pi_S (\mathbf{v}^{(0)})$
and it can therefore be written uniquely as $\alpha^{(1)} \mathbf{a} + \beta^{(1)} \mathbf{b} + \gamma^{(1)} \mathbf{c}$.

The coefficients $\alpha^{(1)}$, $\beta^{(1)}$, and $\gamma^{(1)}$ satisfy the following linear relation:
\begin{equation*}
	G\begin{bmatrix}
		\alpha^{(1)} \\
		\beta^{(1)} \\
		\gamma^{(1)}
	\end{bmatrix}
	= C^\top 
	\operatorname{stack}(\mathbf{v}^{(0)}),
\end{equation*}
where $\mathbf{v}^{(0)} = \operatorname{vec}(\mathbf{a} + \rho (1-\eta \lambda_n^l) \mathbf{b} - \eta \rho \sigma_r^l \frac{1}{2^{l-1}} \mathbf{c}) + O(\rho^2)$.
Consequently, these coefficients must solve the associated least-squares system:
\begin{equation*}
    \begin{cases}
        \bar{a}\alpha^{(1)} + \bar{s}\gamma^{(1)} 
        &= \langle \mathbf{a}, \mathbf{v}^{(0)} \rangle,
        \\
        \bar{b}\beta^{(1)} + \bar{t}\gamma^{(1)} 
        &= \langle \mathbf{b}, \mathbf{v}^{(0)} \rangle,
        \\
        \bar{s} \alpha^{(1)} + \bar{t}\beta^{(1)} + \bar{c}\gamma^{(1)} 
        &= \langle \mathbf{c}, \mathbf{v}^{(0)} \rangle.
    \end{cases}
\end{equation*}
where 
$
    \langle \mathbf{a}, \mathbf{v}^{(0)} \rangle 
    = \bar{a} - \eta\rho\sigma_r^l \frac{1}{2^{l-1}} \bar{s} + \langle \mathbf{a}, O(\rho^2) \rangle
$,
$
    \langle \mathbf{b}, \mathbf{v}^{(0)} \rangle 
    = \rho(1-\eta\lambda_n^l) \bar{b} -\eta\rho\sigma_r^l \frac{1}{2^{l-1}} \bar{t} + \langle \mathbf{b}, O(\rho^2) \rangle
$,
and
$
    \langle \mathbf{c}, \mathbf{v}^{(0)} \rangle 
    = \bar{s} + \rho(1-\eta\lambda_n^l) \bar{t} -\eta\rho\sigma_r^l \frac{1}{2^{l-1}} \bar{c} + \langle \mathbf{c}, O(\rho^2) \rangle
$.

By performing a Schur complement elimination on $\alpha^{(1)}$ and $\beta^{(1)}$, we isolate $\gamma^{(1)}$ as:
\begin{equation*}
	\gamma^{(1)} = 
	-\eta\rho\sigma_r^l \frac{1}{2^{l-1}}
	+ 
	\underbrace{
		\frac{
			\left\langle 
			\mathbf{c} - 
			(\nicefrac{\bar{s}}{\bar{a}}) \mathbf{a}
			- 
			(\nicefrac{\bar{t}}{\bar{b}}) \mathbf{b}
			, O(\rho^2) 
			\right\rangle
		}{
			\left\langle 
			\mathbf{c} - 
			(\nicefrac{\bar{s}}{\bar{a}}) \mathbf{a}
			- 
			(\nicefrac{\bar{t}}{\bar{b}}) \mathbf{b}
			, \mathbf{c}
			\right\rangle
	}}_{:=\tilde{O}_\gamma(\rho^2)}.
\end{equation*}
Ignoring higher-order $O(\rho^2)$ corrections, we approximate $\gamma^{(1)}$ by its leading-order component $\tilde{\gamma}^{(1)} :=
-\eta\rho\sigma_r^l \frac{1}{2^{l-1}}$.
In a similar manner, the approximations for $\alpha^{(1)}$ and $\beta^{(1)}$ are:
\begin{equation*}
    \begin{aligned}
        \alpha^{(1)} &=
    	1 + \tilde{O}_\alpha(\rho^2)
    	\approx 
    	\tilde{\alpha}^{(1)},
        \\
    	\beta^{(1)} &=
    	\rho(1-\eta\lambda_n^l) + \tilde{O}_\beta(\rho^2)
    	\approx 
    	\tilde{\beta}^{(1)}.
    \end{aligned}
\end{equation*}
where $\tilde{\alpha}^{(1)} := 1$ and $\tilde{\beta}^{(1)} := \rho(1-\eta\lambda_n^l)$.

These findings reveal that, up to first order in $\rho$, the coefficients 
$\alpha^{(1)}$, $\beta^{(1)}$, and $\gamma^{(1)}$ correspond exactly to the directional components of $\mathbf{v}^{(0)}$ along 
$\mathbf{a}$, $\mathbf{b}$, and $\mathbf{c}$ respectively. Thus,
\begin{equation}
    \mathbf{w}^{(1)} =
    \tilde{\mathbf{w}}^{(1)} + \text{truncation residual},
    \label{equation:tilde_w_1}
\end{equation}
where 
$
    \tilde{\mathbf{w}}^{(1)} = 
    \tilde{\alpha}^{(1)} \mathbf{a} + \tilde{\beta}^{(1)} \mathbf{b} + \tilde{\gamma}^{(1)} \mathbf{c}
$,
and truncation residual at the first TPGD iteration is thus
$
    \tilde{O}_\alpha(\rho^2) \cdot \mathbf{a} + \tilde{O}_\beta(\rho^2) \cdot \mathbf{b} + \tilde{O}_\gamma(\rho^2) \cdot \mathbf{c}
$.
We define $\tilde{\mathbf{w}}^{(1)}$ as the \emph{first TPGD point}, obtained by discarding 
the high-order $\tilde{O}(\rho^2)$ contributions from $\mathbf{w}^{(1)}$. This completes the proof of Proposition~\ref{proposition:first-tpgd-point}.

\subsection{The $(t+1)$-th Truncated Projected Gradient Descent Step}
\label{append:the-t+1-TPGD-step}

\begin{proposition}[The $(t+1)$-th TPGD Point]\label{proposition:t+1-tpgd-point}
	Suppose that $\rho$ and $\eta$ are sufficiently small. Assume that for an arbitrary $t \in \{1, 2, 3, \ldots\}$, the following coefficients hold:
    \begin{equation*}
        \begin{cases}
            \tilde{\alpha}^{(t)} &= 1, 
            \\
    		\tilde{\beta}^{(t)} &= 
    		\rho(1-\eta\lambda_n^l)^t,
            \\
    		\tilde{\gamma}^{(t)} &=
    		-\frac{\rho\eta}{2^{l-1}}
    		\left[
    		\sum\nolimits_{\tau=0}^{t-1}
    		(1-\eta \lambda_n^l)^\tau
    		\right]
    		\sigma_r^l.
        \end{cases}
    \end{equation*}  
	This implies that $\tilde{\mathbf{w}}^{(t)} = 
	\tilde{\alpha}^{(t)} \mathbf{a} +
	\tilde{\beta}^{(t)} \mathbf{b} +
	\tilde{\gamma}^{(t)} \mathbf{c}$.
	Then the recurrence relations for the next TPGD step are given by:
    \begin{equation*}
        \begin{cases}
            \tilde{\alpha}^{(t+1)} &= 1, 
            \\
    		\tilde{\beta}^{(t+1)} &= 
    		\rho(1-\eta\lambda_n^l)^{t+1}, 
            \\
    		\tilde{\gamma}^{(t+1)} &=
    		-\frac{\rho\eta}{2^{l-1}}
    		\left[
    		\sum\nolimits_{\tau=0}^{t}
    		(1-\eta \lambda_n^l)^\tau
    		\right]
    		\sigma_r^l.
        \end{cases}
    \end{equation*}
	In other words, the $(t+1)$-th TPGD point $\tilde{\mathbf{w}}^{(t+1)}$ has the form:
	$
		\tilde{\mathbf{w}}^{(t+1)}
		= 
		\tilde{\alpha}^{(t+1)} \mathbf{a} +
		\tilde{\beta}^{(t+1)} \mathbf{b} +
		\tilde{\gamma}^{(t+1)} \mathbf{c}
	$.
\end{proposition}

\begin{proof}
	The proof follows from arguments analogous to those used in the proof of 
	Proposition~\ref{proposition:first-tpgd-point}.
	
	\paragraph{Tensor Gradient at the $(t+1)$-th PGD Step.}
	By substituting the expression of $\tilde{\mathbf{w}}^{(t)}$ into the function $\tilde{\mathcal{M}}(\cdot)$, we compute
    \begin{equation*}
        \begin{aligned}
			&\Bigl\langle \langle \mathbf{P}^{\otimes l}, \tilde{\alpha}^{(t)} \mathbf{a} + \tilde{\beta}^{(t)} \mathbf{b} + \tilde{\gamma}^{(t)} \mathbf{c} \rangle_{3*[l]}, \, \langle \mathbf{P}^{\otimes l}, \tilde{\alpha}^{(t)} \mathbf{a} + \tilde{\beta}^{(t)} \mathbf{b} + \tilde{\gamma}^{(t)} \mathbf{c} \rangle_{3*[l]} \Bigr\rangle_{2*[l]}
			\\
			=&
            (\hat{X}\hat{X}^\top)^{\otimes l} + 
			\tilde{\beta}^{(t)} \sigma_r^l (v_r u_n^\top)^{\otimes l} + 
			\tilde{\beta}^{(t)} \sigma_r^l (u_n v_r^\top)^{\otimes l} +
			(\tilde{\beta}^{(t)})^2 (u_n u_n^\top)^{\otimes l} 
            \\
            +&
            \tilde{\gamma}^{(t)} (\hat{X}\hat{X}^\top E)^{\otimes l} + 
			\tilde{\gamma}^{(t)} (E \hat{X}\hat{X}^\top)^{\otimes l} +
			(\tilde{\gamma}^{(t)})^2 (E \hat{X}\hat{X}^\top E)^{\otimes l} 
            \\
            +&
            (\tilde{\beta}^{(t)} \tilde{\gamma}^{(t)}) \sigma_r^l (u_n v_r^\top E)^{\otimes l} + (\tilde{\gamma}^{(t)} \tilde{\beta}^{(t)}) \sigma_r^l (E v_r u_n^\top)^{\otimes l}
        \end{aligned}
    \end{equation*}
	where we directly use the induction assumption that $\tilde{\alpha}^{(t)}=1$.
	
	Next, we compute the gradient $\nabla h^l (\tilde{\mathbf{w}}^{(t)})$ as follows by using Lemma~\ref{lemma:gradient-of-the-tensor-objective} again:
    \begin{equation*}
        \begin{aligned}
            \Bigl\langle 
            \langle 
            \mathbf{A}_r^{\otimes l}, 
            \mathbf{a} + \tilde{\beta}^{(t)} \mathbf{b} + \tilde{\gamma}^{(t)} \mathbf{c} 
            \rangle_{[3l-1]}, 
            &\langle 
            \mathbf{A}^{\otimes l}, \tilde{\mathcal{M}} (\hat{\mathbf{w}}) - \tilde{\mathcal{M}} (\operatorname{vec}(Z)^{\otimes l}) 
            \rangle_{[3l-1,3l]} +
            \\
    		&\langle 
            \mathbf{A}^{\otimes l}, 
            \tilde{\beta}^{(t)} \sigma_r^l (v_r u_n^\top)^{\otimes l} + 
            \tilde{\beta}^{(t)} \sigma_r^l (u_n v_r^\top)^{\otimes l} +
            (\tilde{\beta}^{(t)})^2 (u_n u_n^\top)^{\otimes l}
            \rangle_{[3l-1,3l]} +
            \\
    		&\langle 
            \mathbf{A}^{\otimes l}, 
            \tilde{\gamma}^{(t)} (\hat{X}\hat{X}^\top Y)^{\otimes l} + 
            \tilde{\gamma}^{(t)} (Y \hat{X}\hat{X}^\top)^{\otimes l} +
            (\tilde{\gamma}^{(t)})^2 (Y \hat{X} \hat{X}^\top Y)^{\otimes l} 
            \rangle_{[3l-1,3l]} +
            \\
    		&\langle 
            \mathbf{A}^{\otimes l}, 
            (\tilde{\beta}^{(t)} \tilde{\gamma}^{(t)}) \sigma_r^l (u_n v_r^\top Y)^{\otimes l} +
            (\tilde{\gamma}^{(t)} \tilde{\beta}^{(t)}) \sigma_r^l (Y v_r u_n^\top)^{\otimes l}
            \rangle_{[3l-1,3l]}
            \Bigr\rangle_{[2l-1]}.
        \end{aligned}
    \end{equation*}
	
	This gradient naturally decomposes into 18 sub-components, which are summarized in Table~\ref{table:tensor_gradient_terms}. We focus on the most significant contributions, $H_1^1$ and $H_2^1$, which are computed as follows:
    \begin{equation*}
		\begin{aligned}
			H_1^1 =&
            \Bigl\langle 
            \langle 
            \mathbf{A}_r^{\otimes l}, 
            \tilde{\beta}^{(t)} \mathbf{b} 
            \rangle_{[3l-1]}, 
            \langle 
            \mathbf{A}^{\otimes l}, 
            \tilde{\mathcal{M}} (\hat{\mathbf{w}}) - \tilde{\mathcal{M}} 
            (\operatorname{vec}(Z)^{\otimes l}) 
            \rangle_{[3l-1,3l]} 
            \Bigr\rangle_{[2l-1]} 
			\\
            =&
			\tilde{\beta}^{(t)} \operatorname{vec} 
			( \nabla f(\hat{X}\hat{X}^\top) 
			u_n q_r^\top )^{\otimes l} 
			= \tilde{\beta}^{(t)} \lambda_n^l \operatorname{vec}(u_n q_r^\top)^{\otimes l}.
	        \\
			H_2^1 =& 
			\Bigl\langle 
            \langle 
            \mathbf{A}_r^{\otimes l}, \mathbf{a} 
            \rangle_{[3l-1]}, 
            \langle 
            \mathbf{A}^{\otimes l}, 
            \tilde{\beta}^{(t)} 
            \sigma_r^l (v_r u_n^\top)^{\otimes l} + 
            \tilde{\beta}^{(t)} 
            \sigma_r^l (u_n v_r^\top)^{\otimes l} 
            \rangle_{[3l-1,3l]} 
            \Bigr\rangle_{[2l-1]} 
			\\
            =&
			\tilde{\beta}^{(t)} \sigma_r^l 
			\operatorname{vec} ( E_1 \hat{X} )^{\otimes l} 
			+ \tilde{\beta}^{(t)} \sigma_r^l 
			\operatorname{vec} ( E_2 \hat{X} )^{\otimes l} 
			= \tilde{\beta}^{(t)} \sigma_r^l \frac{1}{2^{l-1}} \operatorname{vec} (E \hat{X})^{\otimes l}.
		\end{aligned}
	\end{equation*}
	The derivations of these sub-terms follow the same procedure as in Appendix~\ref{append:the-first-TPGD-step}, and are therefore omitted here.
	\begin{table*}[h]
		\centering
		\caption{Expressions of the 18 Sub-terms in Tensor Gradient}
		\label{table:tensor_gradient_terms}
		\resizebox{1.0\linewidth}{!} 
		{
			\begin{tabular}{cc}
				\toprule
				Term & Expression \\
				\midrule
				$0$ 
				& 
				$\langle
				\langle
				\mathbf{A}_r^{\otimes l},
				\mathbf{a}
				\rangle_{[3l-1]},
				\langle
				\mathbf{A}^{\otimes l},
				\tilde{\mathcal{M}} (\hat{\mathbf{w}})
				- 
				\tilde{\mathcal{M}} (\operatorname{vec}(Z)^{\otimes l})
				\rangle_{[3l-1,3l]}
				\rangle_{[2l-1]}$ 
				\\
				$H_1^1$ 
				& 
				$\langle
				\langle
				\mathbf{A}_r^{\otimes l},
				\tilde{\beta}^{(t)} \mathbf{b}
				\rangle_{[3l-1]},
				\langle
				\mathbf{A}^{\otimes l},
				\tilde{\mathcal{M}} (\hat{\mathbf{w}})
				- 
				\tilde{\mathcal{M}} (\operatorname{vec}(Z)^{\otimes l})
				\rangle_{[3l-1,3l]}
				\rangle_{[2l-1]}$ 
				\\
				$H_1^2$ 
				& 
				$\langle
				\langle
				\mathbf{A}_r^{\otimes l},
				\tilde{\gamma}^{(t)} \mathbf{c}
				\rangle_{[3l-1]},
				\langle
				\mathbf{A}^{\otimes l},
				\tilde{\mathcal{M}} (\hat{\mathbf{w}})
				- 
				\tilde{\mathcal{M}} (\operatorname{vec}(Z)^{\otimes l})
				\rangle_{[3l-1,3l]}
				\rangle_{[2l-1]}$
				\\
				$H_2^1$ 
				& 
				$\langle
				\langle
				\mathbf{A}_r^{\otimes l},
				\mathbf{a}
				\rangle_{[3l-1]},
				\langle
				\mathbf{A}^{\otimes l},
				\tilde{\beta}^{(t)} \sigma_r^l (v_r u_n^\top)^{\otimes l}
				+
				\tilde{\beta}^{(t)} \sigma_r^l (u_n v_r^\top)^{\otimes l}
				\rangle_{[3l-1,3l]}
				\rangle_{[2l-1]}$ 
				\\
				$H_2^2$ 
				& 
				$\langle
				\langle
				\mathbf{A}_r^{\otimes l},
				\mathbf{a}
				\rangle_{[3l-1]},
				\langle
				\mathbf{A}^{\otimes l},
				(\tilde{\beta}^{(t)})^2 (u_n u_n^\top)^{\otimes l}
				\rangle_{[3l-1,3l]}
				\rangle_{[2l-1]}$ 
				\\
				$H_2^3$ 
				& 
				$\langle
				\langle
				\mathbf{A}_r^{\otimes l},
				\tilde{\beta}^{(t)} \mathbf{b}
				\rangle_{[3l-1]},
				\langle
				\mathbf{A}^{\otimes l},
				\tilde{\beta}^{(t)} \sigma_r^l (v_r u_n^\top)^{\otimes l}
				+
				\tilde{\beta}^{(t)} \sigma_r^l (u_n v_r^\top)^{\otimes l}
				\rangle_{[3l-1,3l]}
				\rangle_{[2l-1]}$
				\\
				$H_2^4$ 
				& 
				$\langle
				\langle
				\mathbf{A}_r^{\otimes l},
				\tilde{\beta}^{(t)} \mathbf{b}
				\rangle_{[3l-1]},
				\langle
				\mathbf{A}^{\otimes l},
				(\tilde{\beta}^{(t)})^2 (u_n u_n^\top)^{\otimes l}
				\rangle_{[3l-1,3l]}
				\rangle_{[2l-1]}$
				\\
				$H_2^5$ 
				& 
				$\langle
				\langle
				\mathbf{A}_r^{\otimes l},
				\tilde{\gamma}^{(t)} \mathbf{c}
				\rangle_{[3l-1]},
				\langle
				\mathbf{A}^{\otimes l},
				\tilde{\beta}^{(t)} \sigma_r^l (v_r u_n^\top)^{\otimes l}
				+
				\tilde{\beta}^{(t)} \sigma_r^l (u_n v_r^\top)^{\otimes l}
				\rangle_{[3l-1,3l]}
				\rangle_{[2l-1]}$
				\\
				$H_2^6$ 
				& 
				$\langle
				\langle
				\mathbf{A}_r^{\otimes l},
				\tilde{\gamma}^{(t)} \mathbf{c}
				\rangle_{[3l-1]},
				\langle
				\mathbf{A}^{\otimes l},
				(\tilde{\beta}^{(t)})^2 (u_n u_n^\top)^{\otimes l}
				\rangle_{[3l-1,3l]}
				\rangle_{[2l-1]}$ 
				\\
				$H_3^1$ 
				& 
				$\langle
				\langle
				\mathbf{A}_r^{\otimes l},
				\mathbf{a}
				\rangle_{[3l-1]},
				\langle
				\mathbf{A}^{\otimes l},
				\tilde{\gamma}^{(t)} (\hat{X}\hat{X}^\top E)^{\otimes l}
				+
				\tilde{\gamma}^{(t)} (E \hat{X}\hat{X}^\top)^{\otimes l}
				\rangle_{[3l-1,3l]}
				\rangle_{[2l-1]}$ 
				\\
				$H_3^2$ 
				& 
				$\langle
				\langle
				\mathbf{A}_r^{\otimes l},
				\mathbf{a}
				\rangle_{[3l-1]},
				\langle
				\mathbf{A}^{\otimes l},
				(\tilde{\gamma}^{(t)})^2 (E \hat{X} \hat{X}^\top E)^{\otimes l}
				\rangle_{[3l-1,3l]}
				\rangle_{[2l-1]}$ 
				\\
				$H_3^3$ 
				& 
				$\langle
				\langle
				\mathbf{A}_r^{\otimes l},
				\tilde{\beta}^{(t)} \mathbf{b}
				\rangle_{[3l-1]},
				\langle
				\mathbf{A}^{\otimes l},
				\tilde{\gamma}^{(t)} (\hat{X}\hat{X}^\top E)^{\otimes l}
				+
				\tilde{\gamma}^{(t)} (E \hat{X}\hat{X}^\top)^{\otimes l}
				\rangle_{[3l-1,3l]}
				\rangle_{[2l-1]}$ 
				\\
				$H_3^4$ 
				& 
				$\langle
				\langle
				\mathbf{A}_r^{\otimes l},
				\tilde{\beta}^{(t)} \mathbf{b}
				\rangle_{[3l-1]},
				\langle
				\mathbf{A}^{\otimes l},
				(\tilde{\gamma}^{(t)})^2 (E \hat{X} \hat{X}^\top E)^{\otimes l}
				\rangle_{[3l-1,3l]}
				\rangle_{[2l-1]}$
				\\
				$H_3^5$ 
				& 
				$\langle
				\langle
				\mathbf{A}_r^{\otimes l},
				\tilde{\gamma}^{(t)} \mathbf{c}
				\rangle_{[3l-1]},
				\langle
				\mathbf{A}^{\otimes l},
				\tilde{\gamma}^{(t)} (\hat{X}\hat{X}^\top E)^{\otimes l}
				+
				\tilde{\gamma}^{(t)} (E \hat{X}\hat{X}^\top)^{\otimes l}
				\rangle_{[3l-1,3l]}
				\rangle_{[2l-1]}$  
				\\
				$H_3^6$ 
				& 
				$\langle
				\langle
				\mathbf{A}_r^{\otimes l},
				\tilde{\gamma}^{(t)} \mathbf{c}
				\rangle_{[3l-1]},
				\langle
				\mathbf{A}^{\otimes l},
				(\tilde{\gamma}^{(t)})^2 (E \hat{X} \hat{X}^\top E)^{\otimes l}
				\rangle_{[3l-1,3l]}
				\rangle_{[2l-1]}$ 
				\\
				$H_4^1$ 
				& 
				$\langle
				\langle
				\mathbf{A}_r^{\otimes l},
				\mathbf{a}
				\rangle_{[3l-1]},
				\langle
				\mathbf{A}^{\otimes l},
				(\tilde{\beta}^{(t)} \tilde{\gamma}^{(t)}) \sigma_r^l (u_n v_r^\top E)^{\otimes l}
				+
				(\tilde{\gamma}^{(t)} \tilde{\beta}^{(t)}) \sigma_r^l (E v_r u_n^\top)^{\otimes l}
				\rangle_{[3l-1,3l]}
				\rangle_{[2l-1]}$
				\\
				$H_4^2$ 
				& 
				$\langle
				\langle
				\mathbf{A}_r^{\otimes l},
				\tilde{\beta}^{(t)} \mathbf{b}
				\rangle_{[3l-1]},
				\langle
				\mathbf{A}^{\otimes l},
				(\tilde{\beta}^{(t)} \tilde{\gamma}^{(t)}) \sigma_r^l (u_n v_r^\top E)^{\otimes l}
				+
				(\tilde{\gamma}^{(t)} \tilde{\beta}^{(t)}) \sigma_r^l (E v_r u_n^\top)^{\otimes l}
				\rangle_{[3l-1,3l]}
				\rangle_{[2l-1]}$
				\\
				$H_4^3$ 
				& 
				$\langle
				\langle
				\mathbf{A}_r^{\otimes l},
				\tilde{\gamma}^{(t)} \mathbf{c}
				\rangle_{[3l-1]},
				\langle
				\mathbf{A}^{\otimes l},
				(\tilde{\beta}^{(t)} \tilde{\gamma}^{(t)}) \sigma_r^l (u_n v_r^\top E)^{\otimes l}
				+
				(\tilde{\gamma}^{(t)} \tilde{\beta}^{(t)}) \sigma_r^l (E v_r u_n^\top)^{\otimes l}
				\rangle_{[3l-1,3l]}
				\rangle_{[2l-1]}$
				\\
				\bottomrule
			\end{tabular}
		}
	\end{table*}
	
	\paragraph{$(t+1)$-th PGD Point and Its Truncated Approximation.}
	The intermediate point $\mathbf{v}^{(t)}$, obtained after one step of GD, is given by:
    \begin{equation*}
        \mathbf{v}^{(t)} 
		= \tilde{\mathbf{w}}^{(t)} - \eta \nabla h^l (\tilde{\mathbf{w}}^{(t)}) 
		= \operatorname{vec}(\hat{X})^{\otimes l} + \tilde{\beta}^{(t)} \operatorname{vec}(u_n q_r^\top)^{\otimes l} + \tilde{\gamma}^{(t)} \operatorname{vec}(E\hat{X})^{\otimes l}
        - \eta \cdot
        \text{sub-terms}.
	\end{equation*}
    The sub-terms are
    $
        H_1^1 + H_1^2 + H_2^1 + H_2^2 + H_2^3 + H_2^4 + H_2^5 + H_2^6 + H_3^1 + H_3^2 + H_3^3 + H_3^4 + H_3^5 + H_3^6 + H_4^1 + H_4^2 + H_4^3
    $.
    Rearranging these terms makes $\mathbf{v}^{(t)}$ become
    \begin{equation}
		\operatorname{vec}(\hat{X})^{\otimes l} +
		[ \tilde{\beta}^{(t)} 
		\operatorname{vec}(u_n q_r^\top)^{\otimes l} 
		- \eta H_1^1 ] 
		+ [ \tilde{\gamma}^{(t)} \operatorname{vec}(E\hat{X})^{\otimes l} - \eta H_2^1 ] 
		+ \text{remaining terms},
        \label{equation:remaining_terms}
    \end{equation}
	where $\tilde{\alpha}^{(t+1)}$ remains equal to $1$. Furthermore, the following identities hold:
    \begin{equation*}
        \begin{aligned}
            \tilde{\beta}^{(t)} \operatorname{vec}(u_n q_r^\top)^{\otimes l} - \eta H_1^1 
    		&=  
            \tilde{\beta}^{(t+1)} \mathbf{b}, 
            \\
    		\tilde{\gamma}^{(t)} \operatorname{vec}(E\hat{X})^{\otimes l} - \eta H_2^1 
    		&=  
            \tilde{\gamma}^{(t+1)} \mathbf{c}.
        \end{aligned}
    \end{equation*}
    
	Table~\ref{table:remaining_terms} enumerates the remainder components from Equation~\eqref{equation:remaining_terms}.  
	\begin{table*}[h]
		\centering
		\caption{Expressions and Asymptotic Orders of the Remaining Terms in $\mathbf{v}^{(t)}$}
		\label{table:remaining_terms}
		\resizebox{1.0\linewidth}{!} 
		{
			\begin{tabular}{ccc}
				\toprule
				Term & Expression & Order \\
				\midrule
				$-\eta H_1^2$ 
				& 
				$\frac{\eta^2\rho}{2^{l-1}}
				[
				\sum_{\tau=0}^{t-1}
				(1-\eta \lambda_n^l)^\tau
				]
				\sigma_r^l
				\operatorname{vec}
				(\nabla f(\hat{X}\hat{X}^\top) E \hat{X})^{\otimes l}$ 
				& 
				$O(\rho\eta^2)$ \\
				$-\eta H_2^2$ 
				& 
				$-\eta\rho^2(1-\eta\lambda_n^l)^{2t}
				\operatorname{vec}
				(
				\mathcal{A}^*\mathcal{A} (u_n u_n^\top) \hat{X}
				)^{\otimes l}$ 
				& 
				$O(\rho^2\eta)$ \\
				$-\eta H_2^3$ 
				& 
				$-\eta\rho^2(1-\eta\lambda_n^l)^{2t}
				\sigma_r^l\frac{1}{2^{l-1}}
				\operatorname{vec}
				(E u_n q_r^\top)^{\otimes l}$ 
				& 
				$O(\rho^2\eta)$ \\
				$-\eta H_2^4$ 
				& 
				$-\eta\rho^3(1-\eta\lambda_n^l)^{3t}
				\operatorname{vec}
				(
				\mathcal{A}^*\mathcal{A} (u_n u_n^\top) u_n q_r^\top
				)^{\otimes l}$ 
				& 
				$O(\rho^3\eta)$ \\
				$-\eta H_2^5$ 
				& 
				$\frac{\eta^2\rho^2}{2^{2l-2}}
				[
				\sum_{\tau=0}^{t-1}
				(1-\eta \lambda_n^l)^{\tau+t}
				]
				\sigma_r^{2l}
				\operatorname{vec}
				(
				\mathcal{A}^*\mathcal{A} (v_r u_n^\top + u_n v_r^\top) E \hat{X}
				)^{\otimes l}$ 
				& 
				$O(\rho^2\eta^2)$ \\
				$-\eta H_2^6$ 
				& 
				$\frac{\rho^3\eta^2}{2^{l-1}}
				[
				\sum_{\tau=0}^{t-1}
				(1-\eta \lambda_n^l)^{\tau+2t}
				]
				\sigma_r^l
				\operatorname{vec}
				(
				\mathcal{A}^*\mathcal{A} (u_n u_n^\top) E \hat{X}
				)^{\otimes l}$ 
				& 
				$O(\rho^3\eta^2)$ \\
				$-\eta H_3^1$ 
				& 
				$\frac{\rho\eta^2}{2^{2l-2}}
				[
				\sum_{\tau=0}^{t-1}
				(1-\eta \lambda_n^l)^\tau
				]
				\sigma_r^l
				\operatorname{vec}
				(
				\mathcal{A}^*\mathcal{A} (\hat{X} \hat{X}^\top E + E \hat{X} \hat{X}^\top) \hat{X}
				)^{\otimes l}$ 
				& 
				$O(\rho\eta^2)$ \\
				$-\eta H_3^2$ 
				& 
				$-\frac{\rho^2\eta^3}{2^{2l-2}}
				[
				\sum_{\tau=0}^{t-1}
				(1-\eta \lambda_n^l)^\tau
				]^2
				\sigma_r^{2l}
				\operatorname{vec}
				(
				\mathcal{A}^*\mathcal{A} (E \hat{X} \hat{X}^\top E) \hat{X}
				)^{\otimes l}$ 
				& 
				$O(\rho^2\eta^3)$ \\
				$-\eta H_3^3$ 
				& 
				$\frac{\rho^2\eta^2}{2^{2l-2}}
				\left[
				\sum_{\tau=0}^{t-1}
				(1-\eta \lambda_n^l)^{\tau+t}
				\right]
				\sigma_r^l
				\operatorname{vec}
				(
				\mathcal{A}^*\mathcal{A} (\hat{X} \hat{X}^\top E + E \hat{X} \hat{X}^\top) u_n q_r^\top
				)^{\otimes l}$ 
				& 
				$O(\rho^2\eta^2)$ \\
				$-\eta H_3^4$ 
				& 
				$-\frac{\rho^3\eta^3}{2^{2l-2}}
				[
				\sum_{\tau=0}^{t-1}
				(1-\eta \lambda_n^l)^\tau
				]^2
				(1-\eta \lambda_n^l)^t
				\sigma_r^{2l}
				\operatorname{vec}
				(
				\mathcal{A}^*\mathcal{A} (E \hat{X} \hat{X}^\top E) u_n q_r^\top
				)^{\otimes l}$ 
				& 
				$O(\rho^3\eta^3)$ \\
				$-\eta H_3^5$ 
				& 
				$-\frac{\rho^2\eta^3}{2^{3l-3}}
				[
				\sum_{\tau=0}^{t-1}
				(1-\eta \lambda_n^l)^\tau
				]^2
				\sigma_r^{2l}
				\operatorname{vec}
				(
				\mathcal{A}^*\mathcal{A} (\hat{X} \hat{X}^\top E + E \hat{X} \hat{X}^\top) E \hat{X}
				)^{\otimes l}$ 
				& 
				$O(\rho^2\eta^3)$ \\
				$-\eta H_3^6$ 
				& 
				$\frac{\rho^3\eta^4}{2^{3l-3}}
				[
				\sum_{\tau=0}^{t-1}
				(1-\eta \lambda_n^l)^\tau
				]^3
				\sigma_r^{3l}
				\operatorname{vec}
				(
				\mathcal{A}^*\mathcal{A} (E \hat{X} \hat{X}^\top E) E \hat{X}
				)^{\otimes l}$ 
				&
				$O(\rho^3\eta^4)$ \\
				$-\eta H_4^1$ 
				& 
				$\frac{\rho^2\eta^2}{2^{2l-2}}
				[
				\sum_{\tau=0}^{t-1}
				(1-\eta \lambda_n^l)^{\tau+t}
				]
				\sigma_r^{2l}
				\operatorname{vec}
				(
				\mathcal{A}^*\mathcal{A} (u_n v_r^\top E + E v_r u_n^\top) \hat{X}
				)^{\otimes l}$ 
				& 
				$O(\rho^2\eta^2)$ \\
				$-\eta H_4^2$ 
				& 
				$\frac{\rho^3\eta^2}{2^{2l-2}}
				[
				\sum_{\tau=0}^{t-1}
				(1-\eta \lambda_n^l)^{\tau+2t}
				]
				\sigma_r^{2l}
				\operatorname{vec}
				(
				\mathcal{A}^*\mathcal{A} (u_n v_r^\top E + E v_r u_n^\top) u_n q_r^\top
				)^{\otimes l}$ 
				& 
				$O(\rho^3\eta^2)$ \\
				$-\eta H_4^3$ 
				& 
				$-\frac{\rho^4\eta^3}{2^{3l-3}}
				[
				\sum_{\tau=0}^{t-1}
				(1-\eta \lambda_n^l)^\tau
				]^2
				(1-\eta \lambda_n^l)^{2t}
				\sigma_r^{3l}
				\operatorname{vec}
				(
				\mathcal{A}^*\mathcal{A} (u_n v_r^\top E + E v_r u_n^\top) E \hat{X}
				)^{\otimes l}$ 
				& 
				$O(\rho^4\eta^3)$ \\
				\bottomrule
			\end{tabular}
		}
	\end{table*}
	
	Apart from the 15 remaining terms which are $O((\rho+\eta)\rho\eta)$, we retain only the sub-terms $H_1^1$ and $H_2^1$ when $\eta$ and $\rho$ are sufficiently small.
	By aggregating all $15$ residual terms, we conclude:
	\begin{equation*}
		\mathbf{v}^{(t)} = 
		\tilde{\alpha}^{(t+1)} \mathbf{a} +
		\tilde{\beta}^{(t+1)} \mathbf{b} +
		\tilde{\gamma}^{(t+1)} \mathbf{c} +
		O((\rho+\eta)\rho\eta).
	\end{equation*}
	Analogous to the initial TPGD iteration, the $(t+1)$-th point $\mathbf{w}^{(t+1)}$ is derived as:
    \begin{equation*}
        \mathbf{w}^{(t+1)} 
		= \Pi_S (\mathbf{v}^{(t)}) = 
        \operatorname{unstack}
        \left(
            P_S \operatorname{stack}
            \left( \tilde{\alpha}^{(t+1)} \mathbf{a} + \tilde{\beta}^{(t+1)} \mathbf{b} + \tilde{\gamma}^{(t+1)} \mathbf{c} + O((\rho+\eta)\rho\eta) \right)
        \right).
	\end{equation*}
    Because $\mathbf{w}^{(t+1)} \in S$, it can therefore be written uniquely as
	$
        \alpha^{(t+1)} \mathbf{a} + \beta^{(t+1)} \mathbf{b} + \gamma^{(t+1)} \mathbf{c}
    $.
	The coefficients $\alpha^{(t+1)},\beta^{(t+1)}$, and $\gamma^{(t+1)}$ are obtained by solving the following least-squares projection:
	\begin{equation*}
		G\begin{bmatrix}
			\alpha^{(t+1)} \\
			\beta^{(t+1)} \\
			\gamma^{(t+1)}
		\end{bmatrix}
		= C^\top 
		\operatorname{stack}
		\left(
		\mathbf{v}^{(t)}
		\right).
	\end{equation*}
	To exemplify the computation, consider $\gamma^{(t+1)}$. Eliminating $\alpha^{(t+1)}$ and $\beta^{(t+1)}$, the Schur complement leads to:
    \begin{equation*}
        \begin{aligned}
            &\gamma^{(t+1)} 
    		= -\frac{\eta\rho}{2^{l-1}} \left[ \sum_{\tau=0}^t (1-\eta \lambda_n^l)^\tau \right] \sigma_r^l \mathbf{v}^{(t)} 
            \\
    		+& 
            \underbrace{ \frac{ \left\langle \mathbf{c} - (\nicefrac{\bar{s}}{\bar{a}}) \mathbf{a} - (\nicefrac{\bar{t}}{\bar{b}}) \mathbf{b} , O((\rho+\eta)\rho\eta) \right\rangle }{ \left\langle \mathbf{c} - (\nicefrac{\bar{s}}{\bar{a}}) \mathbf{a} - (\nicefrac{\bar{t}}{\bar{b}}) \mathbf{b} , \mathbf{c} \right\rangle } }_{:=\tilde{O}_\gamma((\rho+\eta)\rho\eta)}.
        \end{aligned}
    \end{equation*}
	Similarly, the other two coefficients are given by:
    \begin{equation*}
        \begin{aligned}
            \alpha^{(t+1)} 
    		&= 1 + \tilde{O}_\alpha((\rho+\eta)\rho\eta) \approx \tilde{\alpha}^{(t+1)}
            \\
    		\beta^{(t+1)} 
    		&= \rho(1-\eta\lambda_n^l)^{t+1} + \tilde{O}_\beta((\rho+\eta)\rho\eta) \approx \tilde{\beta}^{(t+1)}.
        \end{aligned}
    \end{equation*}
    where $\tilde{\alpha}^{(t+1)} := 1$ and $\tilde{\beta}^{(t+1)} := \rho(1-\eta\lambda_n^l)^{t+1}$.
    
	This result indicates that, up to an error of order $\tilde{O}((\rho+\eta)\rho\eta)$, the coefficients $\alpha^{(t+1)},\beta^{(t+1)}$ and $\gamma^{(t+1)}$ coincide with the coefficients of $\mathbf{a},\mathbf{b}$ and $\mathbf{c}$ in the intermediate tensor $\mathbf{v}^{(t)}$, respectively. Hence, $\mathbf{w}^{(t+1)}$ can be reformulated as:
	\begin{equation}
		\tilde{\mathbf{w}}^{(t+1)} + \text{truncation residual}.
		\label{equation:tilde_w_t+1}
	\end{equation} 
    where truncation residual at the $(t+1)$-th TPGD iteration is thus
    $
        \tilde{O}_\alpha((\rho+\eta)\rho\eta) \cdot \mathbf{a} + \tilde{O}_\beta((\rho+\eta)\rho\eta) \cdot \mathbf{b} + \tilde{O}_\gamma((\rho+\eta)\rho\eta) \cdot \mathbf{c}
    $.
	The tensor $\tilde{\mathbf{w}}^{(t+1)}$ is referred to as the \textit{$(t+1)$-th TPGD point}, as it neglects higher-order error contributions.
	This completes the proof of Proposition~\ref{proposition:t+1-tpgd-point}.
\end{proof}

\subsection{Sufficient Decrease Condition of PGD}
\label{append:sufficient-decrease-condition-of-pgd}
We introduce the generic projected gradient descent (PGD) algorithm in tensor space \citep{chen2019non}, as its structure is largely analogous to that of PGD in matrix space. In this work, the PGD algorithm in tensor space updates the solution iteratively as follows:
\begin{equation*}
	\mathbf{w}^{(t+1)} = 
	\Pi_S(\tilde{\mathbf{w}}^{(t)} 
	- \eta \nabla h^l (\tilde{\mathbf{w}}^{(t)}))
	:=
	\Pi_S(\mathbf{v}^{(t)}),
\end{equation*}
where $\Pi_S(\cdot)$ denotes the projection operator onto subspace $S$, and $\eta > 0$ is the step size.

We assume that the tensor objective function $h^l(\cdot)$ is $L$-smooth, i.e., its gradient is Lipschitz continuous with constant $L > 0$, then for any $\mathbf{u}, \mathbf{v} \in \mathbb{R}^{(nr)\circ l}$
\begin{equation}
	\|
	\nabla h^l (\mathbf{u}) - \nabla h^l (\mathbf{v})
	\|_F
	\leq L
	\|\mathbf{u} - \mathbf{v}\|_F.
	\label{equation:tensor_L_smooth}
\end{equation}
By the standard descent lemma for $L$-smooth functions, we have
\begin{equation*}
	h^l(\mathbf{w}^{(t+1)})
	\leq
	h^l(\tilde{\mathbf{w}}^{(t)})
	+
	\langle
	\nabla h^l (\tilde{\mathbf{w}}^{(t)}),
	\mathbf{w}^{(t+1)} - \tilde{\mathbf{w}}^{(t)}
	\rangle
	+
	\frac{L}{2}
	\|\mathbf{w}^{(t+1)} - \tilde{\mathbf{w}}^{(t)}\|_F^2.
\end{equation*}
To proceed, we utilize the optimality condition of the projection operator. Specifically, for any $\mathbf{v} \in \mathbb{R}^{(nr)\circ l}$, the projection $\mathbf{w} = \Pi_S(\mathbf{v})$ satisfies the following variational inequality:
\begin{equation*}
	\langle
	\mathbf{v} - \mathbf{w}, \mathbf{u} - \mathbf{w}
	\rangle
	\leq 0, \quad \forall\, \mathbf{u} \in S.
\end{equation*}
Applying this condition to $(t+1)$-th iteration, with $\mathbf{v} = \tilde{\mathbf{w}}^{(t)} - \eta \nabla h^l(\tilde{\mathbf{w}}^{(t)})$, $\mathbf{w} = \mathbf{w}^{(t+1)}$, and $\mathbf{u} = \tilde{\mathbf{w}}^{(t)} \in S$, we obtain
$
    \langle
	\tilde{\mathbf{w}}^{(t)} 
	- 
	\eta \nabla h^l (\tilde{\mathbf{w}}^{(t)})
	-
	\mathbf{w}^{(t+1)},
	\tilde{\mathbf{w}}^{(t)} - \mathbf{w}^{(t+1)}
	\rangle
	\leq 0
$.
This yields
\begin{equation*}
	\langle
	\nabla h^l (\tilde{\mathbf{w}}^{(t)}),
	\mathbf{w}^{(t+1)} - \tilde{\mathbf{w}}^{(t)}
	\rangle
	\leq 
	-\frac{1}{\eta}
	\|
	\mathbf{w}^{(t+1)} - \tilde{\mathbf{w}}^{(t)}
	\|_F^2.
\end{equation*}
Substituting this inequality into the smoothness-based upper bound, we obtain:
\begin{equation}
	\begin{aligned}
		h^l(\mathbf{w}^{(t+1)}) 
		&\leq h^l(\tilde{\mathbf{w}}^{(t)}) - \frac{1}{\eta} \|\mathbf{w}^{(t+1)} - \tilde{\mathbf{w}}^{(t)}\|_F^2
        + \frac{L}{2} \|\mathbf{w}^{(t+1)} - \tilde{\mathbf{w}}^{(t)}\|_F^2 \\
		&\leq h^l(\tilde{\mathbf{w}}^{(t)}) - \left( \nicefrac{1}{\eta} - \nicefrac{L}{2} \right) \|\mathbf{w}^{(t+1)} - \tilde{\mathbf{w}}^{(t)}\|_F^2 
		=: h^l(\tilde{\mathbf{w}}^{(t)}) - \Delta^{(t+1)}.
	\end{aligned}
	\label{equation:sufficient_decrease_PGD_loss}
\end{equation}
Here, we define the intermediate descent magnitude $\Delta^{(t+1)}$.
Therefore, a sufficient decrease condition in the tensor objective function is guaranteed when the step size $\eta$ satisfies
\begin{equation}
	\Delta^{(t+1)} > 0
	\quad \Rightarrow \quad
	\eta < 2/L,
	\label{equation:sufficient_decrease_PGD_eta}
\end{equation}
which ensures the objective value strictly decreases at each projection step.

\subsection{Proof of Proposition~\ref{proposition:sufficient-descent-of-tpgd}}
\label{append:proof_of_proposition_sufficient_descent_of_tpgd}
\begin{proof}
	By building upon the results from Appendix~\ref{append:sufficient-decrease-condition-of-pgd}, we directly analyze the sufficient decrease condition for the TPGD. Let $\tilde{\mathbf{w}}^{(t+1)}$ denote the truncated point after projection and post-processing. Leveraging the $L$-smoothness of $h^l(\cdot)$, we apply the standard upper bound again:
    \begin{equation*}
        \begin{aligned}
            h^l(\tilde{\mathbf{w}}^{(t+1)}) 
            - h^l(\mathbf{w}^{(t+1)})
            \leq& \langle 
            \nabla h^l (\mathbf{w}^{(t+1)}), \tilde{\mathbf{w}}^{(t+1)} - \mathbf{w}^{(t+1)} \rangle 
            +\frac{L}{2} 
            \|\tilde{\mathbf{w}}^{(t+1)} - \mathbf{w}^{(t+1)}\|_F^2 
            \\
            \leq& \|\nabla h^l (\mathbf{w}^{(t+1)})\|_F \|\tilde{\mathbf{w}}^{(t+1)} - \mathbf{w}^{(t+1)}\|_F 
            +\frac{L}{2} 
            \|\tilde{\mathbf{w}}^{(t+1)} - \mathbf{w}^{(t+1)}\|_F^2.
        \end{aligned}
    \end{equation*}
	Based on the truncation structure, the deviation between $\tilde{\mathbf{w}}^{(t+1)}$ and $\mathbf{w}^{(t+1)}$ is
    \begin{equation*}
        \mathbf{w}^{(t+1)} - \tilde{\mathbf{w}}^{(t+1)}
        = \text{truncation residual},
    \end{equation*}
    where the truncation residual is defined in Equation~\eqref{equation:tilde_w_1} and Equation~\eqref{equation:tilde_w_t+1}.
	Consequently, we obtain the following norm bound:
	\begin{equation*}
		\|
		\tilde{\mathbf{w}}^{(t+1)} - \mathbf{w}^{(t+1)}
		\|_F
		\leq
		\begin{cases}
			C \rho^2,
			& t = 0, \\
			C (\rho+\eta)\rho\eta,
			& t \geq 1.
		\end{cases}
	\end{equation*}
	where $C = \tilde{O}(1) \cdot (\|\mathbf{a}\|_F + \|\mathbf{b}\|_F + \|\mathbf{c}\|_F)$ is a constant.
	Substituting into the previous inequality, we further obtain:
	\begin{equation}
        h^l(\tilde{\mathbf{w}}^{(t+1)}) - h^l(\mathbf{w}^{(t+1)}) \leq 
        \begin{cases}
            G^{(t+1)} C \rho^2 + \frac{L}{2} C^2 \rho^4, & t = 0, 
            \\
            G^{(t+1)} C (\rho+\eta)\rho\eta + 
        	\frac{L}{2} C^2 (\rho+\eta)^2 \rho^2 \eta^2, & t \geq 1,
        \end{cases}
        \label{equation:sufficient_decrease_truncation_loss}
    \end{equation}
    where we define $G^{(t+1)} := \| \nabla h^l(\mathbf{w}^{(t+1)}) \|_F$.
	Adding the two inequalities \eqref{equation:sufficient_decrease_PGD_loss} and \eqref{equation:sufficient_decrease_truncation_loss}, 
	we arrive at the upper bound (where $\Delta^{(t+1)}$ is defined in \eqref{equation:sufficient_decrease_PGD_loss}):
    \begin{equation*}
        \begin{aligned}
			h^l(\tilde{\mathbf{w}}^{(t+1)}) - h^l(\tilde{\mathbf{w}}^{(t)}) 
        	\leq
            -\Delta^{(t+1)} 
            +
            \begin{cases}
                G^{(t+1)} C \rho^2 + \frac{L}{2} C^2 \rho^4, & t = 0, 
                \\
                G^{(t+1)} C (\rho+\eta)\rho\eta + 
        		\frac{L}{2} C^2 (\rho+\eta)^2 \rho^2 \eta^2, & t \geq 1.
            \end{cases}
        \end{aligned}
    \end{equation*}
	To ensure sufficient decrease, we require:
    \begin{equation*}
        \max \left\{ 
        G^{(t+1)} C \rho^2 + \frac{L}{2} C^2 \rho^4, \, G^{(t+1)} C (\rho+\eta)\rho\eta + 
        \frac{L}{2} C^2 (\rho+\eta)^2 \rho^2 \eta^2 
        \right\} < \Delta^{(t+1)}.
	\end{equation*}
	Assuming that both $\rho$ and $\eta$ are sufficiently small, we neglect higher-order terms $O(\rho^3)$ and $O(\eta^3)$, obtaining an approximate condition:
	\begin{equation*}
		\max
		\{
		G^{(t+1)} C \rho^2,
		G^{(t+1)} C (\rho+\eta)\rho\eta
		\}
		<
		\Delta^{(t+1)}.
	\end{equation*}
	From the first inequality, we derive a bound on $\rho$:
	$
	G^{(t+1)} C \rho^2 < \Delta^{(t+1)}
	$
	yields
	$
	\rho < \sqrt{\frac{\Delta^{(t+1)}}{C\cdot G^{(t+1)}}}
	$.
	Substituting this into the second term:
    \begin{equation*}
		G^{(t+1)} C (\rho+\eta)\rho\eta 
		= 
		G^{(t+1)} C \rho^2\eta + G^{(t+1)} C \rho\eta^2
		< 
		\Delta^{(t+1)} \eta 
		+ \sqrt{C \cdot G^{(t+1)}\Delta^{(t+1)}} \eta^2.
    \end{equation*}
	To satisfy the sufficient decrease condition, we require that
	\begin{equation*}
		\Delta^{(t+1)} \eta
		+
		\sqrt{C\cdot G^{(t+1)}\Delta^{(t+1)}}
		\eta^2
		<
		\Delta^{(t+1)}.
	\end{equation*}
	Solving this quadratic inequality with respect to $\eta$, we obtain:
    \begin{equation*}
        \eta < \hat{\eta} :=
        \frac{
			-\Delta^{(t+1)} +
			\sqrt{
				{\Delta^{(t+1)}}^2
				+
				4\sqrt{C\cdot G^{(t+1)}\Delta^{(t+1)}}
				\Delta^{(t+1)}
			}
		}
		{2\sqrt{C\cdot G^{(t+1)}\Delta^{(t+1)}}}.
    \end{equation*}
	In conclusion, to ensure the objective function $h^l$ consistently decreases at each TPGD step $t$, i.e.,
	$
		h^l(\tilde{\mathbf{w}}^{(t+1)}) < h^l(\tilde{\mathbf{w}}^{(t)})
	$,
	it suffices to choose PGD step size $\eta$ and escape step size $\rho$ satisfying:
	\begin{equation}
		0< \rho <
		\sqrt{
			\frac{\Delta^{(t+1)}}{C\cdot G^{(t+1)}}
		},
		\quad
		0< \eta <
		\min
		\left\{
		\hat{\eta}, \frac{2}{L}
		\right\}.
		\label{equation:sufficient_decrease_TPGD_rho_eta}
	\end{equation}
	This completes the proof.
\end{proof}

\subsection{Proof of Proposition~\ref{proposition:dominance-criteria}}
\label{append:proof_of_proposition_dominance_criteria}

\subsubsection{Dominance of $\tilde{\beta}^{(t)}\mathbf{b}$}
\label{append:dominance-of-b-term}
We begin by examining the condition:
\begin{equation*}
	\|\tilde{\beta}^{(t)} \mathbf{b}\|_F
	>
	\max\{
	\|\tilde{\alpha}^{(t)} \mathbf{a}\|_F,
	\|\tilde{\gamma}^{(t)} \mathbf{c}\|_F
	\},
\end{equation*}
which implies that the component $\tilde{\beta}^{(t)} \mathbf{b}$ dominates the other two in magnitude.

\paragraph{Condition for $\tilde{\beta}^{(t)} \mathbf{b}$ to dominate $\tilde{\alpha}^{(t)} \mathbf{a}$.}
Consider the ratio:
\begin{equation*}
	\begin{aligned}
		\frac{\|\tilde{\beta}^{(t)} \mathbf{b}\|_F}
		{\|\tilde{\alpha}^{(t)} \mathbf{a}\|_F}
		=
		\frac{|\tilde{\beta}^{(t)}|}
		{|\tilde{\alpha}^{(t)}|}
		\cdot
		\frac{\|\mathbf{b}\|_F}
		{\|\mathbf{a}\|_F}
		=
		\frac{|\tilde{\beta}^{(t)}|}
		{|\tilde{\alpha}^{(t)}|}
		\cdot
		\frac{\|u_n q_r^\top\|_F^l}
		{\|\hat{X}\|_F^l}.
	\end{aligned}
\end{equation*}
Substituting the expressions for $\tilde{\beta}^{(t)}$ and $\tilde{\alpha}^{(t)}$, and using the fact that both $u_n$ and $q_r$ are unit vectors, the ratio simplifies to:
\begin{equation*}
    \frac{\|\tilde{\beta}^{(t)} \mathbf{b}\|_F}
    {\|\tilde{\alpha}^{(t)} \mathbf{a}\|_F}
    =
    \frac{\rho (1-\eta\lambda_n^l)^t}
    {\|\hat{X}\|_F^l}
    > 1
	\quad\Rightarrow\quad
    t > \frac{\log\|\hat{X}\|_F^l -\log\rho} 
    {\log(1-\eta\lambda_n^l)}.
\end{equation*}
where the inequality holds under the assumptions $\rho,\eta>0$, $\lambda_n<0$, and odd $l$.

\paragraph{Condition for $\tilde{\beta}^{(t)} \mathbf{b}$ to dominate $\tilde{\gamma}^{(t)} \mathbf{c}$.}
Now consider:
\begin{equation*}     
    \begin{aligned}
		\frac{\|\tilde{\beta}^{(t)} \mathbf{b}\|_F}
		{\|\tilde{\gamma}^{(t)} \mathbf{c}\|_F}
		=
		\frac{|\tilde{\beta}^{(t)}|}
		{|\tilde{\gamma}^{(t)}|}
		\cdot
		\frac{\|\mathbf{b}\|_F}
		{\|\mathbf{c}\|_F}
		=
		\frac{|\tilde{\beta}^{(t)}|}
		{|\tilde{\gamma}^{(t)}|}
		\cdot
		\frac{\|u_n q_r^\top\|_F^l}
		{\|E \hat{X}\|_F^l}.
    \end{aligned}     
\end{equation*}
Substituting the expressions for $\tilde{\beta}^{(t)}$ and $\tilde{\gamma}^{(t)}$, and using the fact that both $u_n$ and $q_r$ are unit vectors, the ratio simplifies to:
\begin{equation*}
    \frac{\|\tilde{\beta}^{(t)} \mathbf{b}\|_F}{\|\tilde{\gamma}^{(t)} \mathbf{c}\|_F} 
	= \frac{\rho (1-\eta\lambda_n^l)^t}
	{\frac{\eta\rho}{2^{l-1}} 
	\left[ \sum_{\tau=0}^{t-1} (1-\eta \lambda_n^l)^\tau \right] \sigma_r^l} 
	\frac{1}{\|E \hat{X}\|_F^l} > 1,
\end{equation*}
which yields
\begin{equation*}
    t < \frac{-1}{\log(1-\eta\lambda_n^l)} \log \left[ 1 - 2^{l-1} \frac{(-\lambda_n)^l}{\sigma_r^l} \frac{1}{\|E \hat{X}\|_F^l} \right],
\end{equation*}
where the inequality holds under the assumptions $\rho,\eta>0$, $\lambda_n<0$, and odd $l$.

Combining both requirements, we define the interval $U_\beta$ below in which $\tilde{\beta}^{(t)} \mathbf{b}$ dominates both $\tilde{\alpha}^{(t)} \mathbf{a}$ and $\tilde{\gamma}^{(t)} \mathbf{c}$, i.e. $t \in U_\beta \cap (0, +\infty)$:
\begin{equation}
	U_\beta 
	:= \frac{1}{\log(1-\eta\lambda_n^l)} \Bigg( \log\frac{\|\hat{X}\|_F^l}{\rho}, 
	-\log \left[ 1-2^{l-1} \frac{(-\lambda_n)^l}{\sigma_r^l} \frac{1}{\|E \hat{X}\|_F^l} \right] \Bigg).
    \label{equation:u_beta}
\end{equation}

\paragraph{Validation of the Range $U_\beta$.}
A valid interval $U_\beta$ must satisfy the condition that its left endpoint is less than its right endpoint. This implies:
\begin{equation}
	1 > \rho >
	\|\hat{X}\|_F^l
	-2^{l-1}
	\frac{(-\lambda_n)^l}{\sigma_r^l}
	\frac{\|\hat{X}\|_F^l}
	{
		\|E \hat{X}\|_F^l
	} =: \rho_\text{min}.
	\label{equation:u_beta_validity}
\end{equation}
We provide an upper bound for the denominator by using operator norms from Lemma~\ref{lemma:upper_bound_A_star_A} where $\Xi^2:=\max\{\|A_1\|_2^2, \cdots, \|A_m\|_2^2\}$. Specifically, 
\begin{equation*}
	\|E \hat{X}\|_F^l 
	\leq \|E\|_2^l \|\hat{X}\|_F^l
	\leq [m(1+\delta_p)]^{\frac{l}{2}} \Xi^l \|v_r u_n^\top + u_n v_r^\top\|_F^l \|\hat{X}\|_F^l
	=
	[2m(1+\delta_p)]^{\frac{l}{2}} \Xi^l \|\hat{X}\|_F^l.
\end{equation*}
Substituting this into the inequality, we require
\begin{equation*}
	1 > 
	\|\hat{X}\|_F^l - 2^{l-1} \frac{(-\lambda_n)^l}{\sigma_r^l} \frac{\|\hat{X}\|_F^l}{[2m(1+\delta_p)]^{\frac{l}{2}} \Xi^l \|\hat{X}\|_F^l}
	= (\sigma_1^2 + \dots + \sigma_r^2)^\frac{l}{2} - \frac{2^{\frac{l}{2}} (-\lambda_n)^l}{2[m(1+\delta_p)]^{\frac{l}{2}}\Xi^l\sigma_r^l}
\end{equation*}
in order for the condition $\rho<1$ to hold, where $\sigma_1 \geq \cdots \geq \sigma_r>0$.

Typically, if we have $-\lambda_n>\sqrt{m(1+\delta_p)/2}\Xi\sigma_r$, then for sufficiently large $l$, it follows that
$
	-\lambda_n >
	2^{\frac{1}{l}-\frac{1}{2}}
	\sqrt{m(1+\delta_p)}\Xi\sigma_r
$.
Consequently, we have:
\begin{equation*}
	\frac{2^{\frac{l}{2}} (-\lambda_n)^l}
	{2[m(1+\delta_p)]^{\frac{l}{2}}\Xi^l\sigma_r^l}
	\approx
	1+
	\frac{2^{\frac{l}{2}} (-\lambda_n)^l}
	{2[m(1+\delta_p)]^{\frac{l}{2}}\Xi^l\sigma_r^l}
	>
	(\sigma_1^2+\cdots+\sigma_r^2)^\frac{l}{2}.
\end{equation*}
Therefore, the range $U_\beta$ is valid provided that (sufficient condition):
\begin{equation}
	l >
	\left\{
	\log_2 
	\frac{\sqrt{2}(-\lambda_n)}
	{
		\sqrt{m(1+\delta_p)
			(\sigma_1^2+\cdots+\sigma_r^2)}
		\Xi\sigma_r
	}
	\right\}^{-1}.
	\label{equation:u_beta_validality_l}
\end{equation}

\subsubsection{Dominance of $\tilde{\gamma}^{(t)}\mathbf{c}$}
\label{append:dominance-of-c-term}

Similar to the deduction in Appendix~\ref{append:dominance-of-b-term}, we can derive that when $t \in U_\gamma \cap (0, +\infty)$, $\tilde{\gamma}^{(t)} \mathbf{c}$ dominates both $\tilde{\alpha}^{(t)} \mathbf{a}$ and $\tilde{\beta}^{(t)} \mathbf{b}$.
The interval $U_\gamma$ is defined as
\begin{equation}
    \begin{aligned}
        U_\gamma
    	:= \Biggl( &\frac{1}{\log(1-\eta\lambda_n^l)} \cdot
        \\
        \max \Biggl\{ &\log\left[ 1 + 2^{l-1} \frac{\|\hat{X}\|_F^l}{\rho} \frac{(-\lambda_n)^l}{\sigma_r^l} \frac{1}{\|E \hat{X}\|_F^l} \right], 
        \\
    	-&\log \left[ 1 - 2^{l-1} \frac{(-\lambda_n)^l}{\sigma_r^l} \frac{1}{\|E \hat{X}\|_F^l} \right] \Biggr\}, \, +\infty \Biggr). 
    \end{aligned}
    \label{equation:u_gamma}
\end{equation}


\section{Additional Details of Section~\ref{sec6}}
\label{append:additional_details_of_section_6}
\subsection{Hyperparameter Analysis of Multi-step SOD Escape}
\begin{table*}[h]
	\caption{Hyperparameter Analysis of $l$ and $t$ based on Real-world MS Example}
	\label{table:ablation}
	\centering
	\resizebox{1.0\linewidth}{!} 
	{
		\begin{tabular}{ccccccc}
			\toprule
			\multicolumn{7}{c}{$l=3\Rightarrow\rho_\text{min}=0.208$}  \\ 
			\hline
			\multicolumn{1}{c|}{$U$} & 
			\multicolumn{1}{c|}{$U_\beta=\emptyset$} & 
			\multicolumn{5}{c}{$U_\gamma=(2006.17, +\infty)$}  \\
			\multicolumn{1}{c|}{$t$} & 
			\multicolumn{1}{c|}{1000} & 
			5000 & 
			33500 & 
			100000 & 
			500000 & 
			2000000  \\ 
			\hline
			\multicolumn{1}{c|}{$\check{X}$} & 
			\multicolumn{1}{c|}{\ding{55}\textsuperscript{1}} & 
			\checkmark & 
			\checkmark & 
			\checkmark & 
			\checkmark & 
			\ding{55}\textsuperscript{2}  \\
			\multicolumn{1}{c|}{$\|\hat{X}\hat{X}^\top-\check{X}\check{X}^\top\|_F$} & 
			\multicolumn{1}{c|}{\ding{55}\textsuperscript{1}} & 
			1.08 & 
			114.22 & 
			$4.08\times 10^{6}$ & 
			$9.94\times 10^{33}$ & 
			\ding{55}\textsuperscript{2}  \\
			\multicolumn{1}{c|}{$\|\check{X}\check{X}^\top-ZZ^\top\|_F$} & 
			\multicolumn{1}{c|}{\ding{55}\textsuperscript{1}} & 
			0.36 & 
			113.29 & 
			$4.08\times 10^{6}$ & 
			$9.94\times 10^{33}$ & 
			\ding{55}\textsuperscript{2}  \\
			\multicolumn{1}{c|}{$\frac{\|\hat{X}\hat{X}^\top-\check{X}\check{X}^\top\|_F}{\|\check{X}\check{X}^\top-ZZ^\top\|_F}$} & 
			\multicolumn{1}{c|}{\ding{55}\textsuperscript{1}} & 
			3.03 & 
			1.01 & 
			1.00 & 
			1.00 & 
			\ding{55}\textsuperscript{2}  \\
			\multicolumn{1}{c|}{$X_\text{final}$} & 
			\multicolumn{1}{c|}{\ding{55}\textsuperscript{1}} & 
			+ & 
			\ding{55}\textsuperscript{2} & 
			\ding{55}\textsuperscript{2} & 
			\ding{55}\textsuperscript{2} & 
			\ding{55}\textsuperscript{2}  \\ 
			\toprule
			\multicolumn{7}{c}{$l=5\Rightarrow\rho_\text{min}=0.097$}  \\ 
			\hline
			\multicolumn{1}{c|}{$U$} & 
			\multicolumn{2}{c|}{$\mathbb{R}^+ \setminus (U_\beta \cup U_\gamma)$} & 
			\multicolumn{1}{c|}{$U_\beta=(26948.72, 33974.73)$} & 
			\multicolumn{3}{c}{$U_\gamma=(33974.73, +\infty)$}  \\
			\multicolumn{1}{c|}{$t$} & 
			1000 & 
			\multicolumn{1}{c|}{5000} & 
			\multicolumn{1}{c|}{33500} & 
			100000 & 
			500000 & 
			2000000  \\ 
			\hline
			\multicolumn{1}{c|}{$\check{X}$} & 
			\ding{55}\textsuperscript{1} & 
			\multicolumn{1}{c|}{\ding{55}\textsuperscript{1}} & 
			\multicolumn{1}{c|}{\checkmark} & 
			\checkmark & 
			\checkmark & 
			\checkmark  \\
			\multicolumn{1}{c|}{$\|\hat{X}\hat{X}^\top-\check{X}\check{X}^\top\|_F$} & 
			\ding{55}\textsuperscript{1} & 
			\multicolumn{1}{c|}{\ding{55}\textsuperscript{1}} & 
			\multicolumn{1}{c|}{0.59} & 
			0.80 & 
			2.00 & 
			25.65  \\
			\multicolumn{1}{c|}{$\|\check{X}\check{X}^\top-ZZ^\top\|_F$} & 
			\ding{55}\textsuperscript{1} & 
			\multicolumn{1}{c|}{\ding{55}\textsuperscript{1}} & 
			\multicolumn{1}{c|}{0.66} & 
			0.43 & 
			1.08 & 
			24.72  \\
			\multicolumn{1}{c|}{$\frac{\|\hat{X}\hat{X}^\top-\check{X}\check{X}^\top\|_F}{\|\check{X}\check{X}^\top-ZZ^\top\|_F}$} & 
			\ding{55}\textsuperscript{1} & 
			\multicolumn{1}{c|}{\ding{55}\textsuperscript{1}} & 
			\multicolumn{1}{c|}{0.89} & 
			1.86 & 
			1.84 & 
			1.04  \\
			\multicolumn{1}{c|}{$X_\text{final}$} & 
			\ding{55}\textsuperscript{1} & 
			\multicolumn{1}{c|}{\ding{55}\textsuperscript{1}} & 
			\multicolumn{1}{c|}{--} & 
			+ & 
			+ & 
			\ding{55}\textsuperscript{2}  \\ 
			\toprule
			\multicolumn{7}{c}{$l=7\Rightarrow\rho_\text{min}=0.043$}  \\ 
			\hline
			\multicolumn{1}{c|}{$U$} & 
			\multicolumn{5}{c|}{$U_\beta=(0, 1093342.41)$} & 
			\multicolumn{1}{c}{$U_\gamma=(1093342.41, +\infty)$}  \\
			\multicolumn{1}{c|}{$t$} & 
			1000 & 
			5000 & 
			33500 & 
			100000 & 
			\multicolumn{1}{c|}{500000} & 
			2000000  \\ 
			\hline
			\multicolumn{1}{c|}{$\check{X}$} & 
			\checkmark & 
			\checkmark & 
			\checkmark & 
			\checkmark & 
			\multicolumn{1}{c|}{\checkmark} & 
			\checkmark  \\
			\multicolumn{1}{c|}{$\|\hat{X}\hat{X}^\top-\check{X}\check{X}^\top\|_F$} & 
			0.665 & 
			0.665 & 
			0.665 & 
			0.666 & 
			\multicolumn{1}{c|}{0.669} & 
			0.754  \\
			\multicolumn{1}{c|}{$\|\check{X}\check{X}^\top-ZZ^\top\|_F$} & 
			0.601 & 
			0.601 & 
			0.600 & 
			0.600 & 
			\multicolumn{1}{c|}{0.598} & 
			0.461  \\
			\multicolumn{1}{c|}{$\frac{\|\hat{X}\hat{X}^\top-\check{X}\check{X}^\top\|_F}{\|\check{X}\check{X}^\top-ZZ^\top\|_F}$} & 
			1.106 & 
			1.106 & 
			1.107 & 
			1.109 & 
			\multicolumn{1}{c|}{1.119} & 
			1.634  \\
			\multicolumn{1}{c|}{$X_\text{final}$} & 
			-- & 
			-- & 
			-- & 
			-- & 
			\multicolumn{1}{c|}{--} & 
			+  \\ 
			\toprule
		\end{tabular}
	}
	\begin{flushleft}
		\footnotesize
		* \checkmark indicates that the quantity is well-defined and numerically stable. \ding{55}\textsuperscript{1} indicates that the quantity is undefined or not considered. \ding{55}\textsuperscript{2} indicates that the quantity is numerically unstable (e.g., results in NaN). \\
		* In $X_{\text{final}}$, "+" denotes convergence to $+Z$ under matrix-space gradient descent after SOD escape, while "--" denotes convergence to $-Z$.
	\end{flushleft}
\end{table*}

As shown in Equation~\eqref{equation:tilde_w_t}, the closed-form solution of $\tilde{\mathbf{w}}^{(t)}$ 
depends on four hyperparameters:
\begin{itemize}
	\item The lifting order $l$ used to simulate over-parameterization in the tensor space (note that it is not necessary to actually lift to the tensor space).
	\item The number of TPGD steps $t$ applied to the objective in Equation~\eqref{equation:tensor_objfunc_h} (this is a simulated process rather than an actual TPGD).
	\item The escape step size $\rho$ in Equation~\eqref{equation:w_escape}, which simulates leaving the strict saddle point $\hat{\mathbf{w}}$ in the tensor space.
	\item The step size $\eta$ used in the simulated TPGD process.
\end{itemize}
According to Equation~\eqref{equation:sufficient_decrease_TPGD_rho_eta}, both $\rho$ and $\eta$ should be chosen as small as possible. Therefore, in practice, they are typically initialized with values less than $1$ (e.g., $0.1$). Based on these values, $U_\beta$ is computed using Equation~\eqref{equation:u_beta}, and $U_\gamma$ is computed using Equation~\eqref{equation:u_gamma}. The corresponding escape points after SOD escape are then obtained using Equations~\eqref{equation:x_check_beta} and~\eqref{equation:x_check_gamma}, which are subsequently used as initialization for gradient descent in the matrix space.
If it is required that $U_\beta\neq\emptyset$, then $\rho$ must not be too small and must satisfy the condition $\rho>\rho_\text{min}$ as defined in Equation~\eqref{equation:u_beta_validity}. This condition becomes easier to satisfy as $l$ increases, because a larger $l$ leads to a smaller $\rho_\text{min}$.

As shown in Table~\ref{table:ablation}, we conduct an ablation study on hyperparameters using the matrix sensing instance described in Appendix~\ref{append:case-2}. As previously mentioned, we fix both $\rho$ and $\eta$ to a small constant value of 0.1. We then evaluate the effect of different lifting orders $l\in\{3,5,7\}$ and various truncation steps $t\in\{1000,5000,33500,1\times 10^5,5\times 10^5,2\times 10^6\}$.
The effectiveness of the escape is primarily measured by the ratio
\begin{equation}
	\frac{\|\hat{X}\hat{X}^\top-\check{X}\check{X}^\top\|_F}{\|\check{X}\check{X}^\top-ZZ^\top\|_F},
	\label{equation:escape_ratio}
\end{equation}
which compares the distance between the escape point and the local minimum to the distance between the escape point and the ground truth. A ratio significantly greater than 1 indicates a successful escape. The variable $X_\text{final}$ denotes the point to which gradient descent converges after the SOD escape. If it converges to $QZ$ where $Q$ is any orthogonal matrix, the escape is considered successful. Otherwise, it is regarded as a failure.

Combining the results from Table \ref{table:ablation}, our experimental findings are as follows:
\begin{itemize}
	\item When $l=3$, $\rho=0.1 < \rho_\text{min}$; thus, $U_\beta = \emptyset$, meaning no $t \in \mathbb{R}^+$ exists such that $\mathbf{b}$-term becomes the dominant term. However, a small $t$, such as 5000, allows $\mathbf{c}$-term to become the dominant term and successfully escape the local minimum $\hat{X}$.
	
	\item When $l=5$, $\rho=0.1 > \rho_\text{min}$. In this case, a small $t \in U_\beta$ allows $\mathbf{b}$-term to become the dominant term and successfully escape. For larger $t$, the iteration falls into $U_\gamma$, where the escape effect is more pronounced, with the ratio exceeding 1.
	
	\item When $l=7$, $\rho=0.1 > \rho_\text{min}$. Here, the range of $U_\beta$ expands significantly, and within $U_\beta$, larger values of $t$ yield more pronounced escape effects. Although $U_\gamma$ exists, it only appears for sufficiently large $t$.
\end{itemize}
These results align perfectly with Figure \ref{figure:ablation}. Interestingly, we observe that if $\mathbf{b}$-term becomes the dominant term and an appropriate $t$ is selected for successful escape, the final $X_\text{final}$ converges to $-Z$. In contrast, if $\mathbf{c}$-term becomes the dominant term and an appropriate $t$ is selected for successful escape, the final $X_\text{final}$ converges to $+Z$.

\subsection{Experimental Setup of Numerical Case Study 1}
\label{append:case-1}
The general Perturbed Matrix Completion (PMC) problem aims to recover a ground-truth low-rank matrix $M^\star \in \mathbb{R}^{n \times n}$ with $\operatorname{rank}(M^\star) = r$ from noisy observations generated by a perturbed linear operator $\mathcal{A}_\epsilon$. This operator is defined as
\begin{equation}
	\mathcal{A}_\epsilon(M_{ij}) =
	\begin{cases}
		M_{ij},\quad & \text{if } (i,j) \in \Omega, \\
		\epsilon M_{ij},\quad & \text{otherwise,}
	\end{cases}
	\label{equation:pmc_data}
\end{equation}
where $\Omega$ is the measurement set defined by
\begin{equation*}
	\{(i,i),\,(i,2k),\,(2k,i) \mid \forall\, i \in [n],\, k \in [\lfloor n/2 \rfloor]\},
\end{equation*}
and $\epsilon > 0$ is a small perturbation parameter. 
The RIP constant of the PMC problem is approximately $\delta_p \approx (1-\epsilon)/(1+\epsilon)$ \citep{yalccin2022factorization}. When $\epsilon$ is small, many spurious local minima tend to appear near the ground-truth factor $z$. Note that: (i) compared with the basic example in Equation~\eqref{equation:basic_example}, the PMC landscape is more intricate and requires a higher-order tensor lift to escape spurious regions; and (ii) in such complex regimes, the $\mathbf{a}$-term and $\mathbf{b}$-term alone are insufficient to constitute a matrix-space escape oracle.

\subsection{Experimental Setup of Numerical Case Study 2}
\label{append:case-2}
The sensing matrices in this case are derived from a real-world scenario~\citep{xu2024implicit}. The matrices are given as follows:
\begin{equation}
    {\setlength{\arraycolsep}{0.8pt}\renewcommand{\arraystretch}{0.82}
    \begin{array}{>{\scriptstyle}c >{\scriptstyle}c}
        \begin{gathered} A_1={}\\[-2pt]
        \left[\begin{smallmatrix}
    		0.0783 & 0.2372 & -0.0439 \\
    		0.2372 & -0.0397 & 0.1456 \\
    		-0.0439 & 0.1456 & -0.4724
    		\end{smallmatrix}\right]\end{gathered}
        &
        \begin{gathered} A_2={}\\[-2pt]
        \left[\begin{smallmatrix}
    		0.0389 & 0.0536 & -0.0059 \\
    		0.0536 & 0.4614 & -0.3907 \\
    		-0.0059 & -0.3907 & 0.2760
    		\end{smallmatrix}\right]\end{gathered}
        \\[3pt]
        \begin{gathered} A_3={}\\[-2pt]
        \left[\begin{smallmatrix}
    		0.0293 & -0.1456 & -0.1656 \\
    		-0.1456 & 0.3257 & 0.2528 \\
    		-0.1656 & 0.2528 & -0.0078
    		\end{smallmatrix}\right]\end{gathered}
        &
        \begin{gathered} A_4={}\\[-2pt]
        \left[\begin{smallmatrix}
    		0.0762 & 0.3193 & -0.3338 \\
    		0.3193 & -0.3364 & 0.5847 \\
    		-0.3338 & 0.5847 & 0.1873
    		\end{smallmatrix}\right]\end{gathered}
        \\[3pt]
        \begin{gathered} A_5={}\\[-2pt]
        \left[\begin{smallmatrix}
    		-0.0889 & 0.7089 & 0.4472 \\
    		0.7089 & 0.3788 & 0.0902 \\
    		0.4472 & 0.0902 & -0.3193
    		\end{smallmatrix}\right]\end{gathered}
        &
        \begin{gathered} A_6={}\\[-2pt]
        \left[\begin{smallmatrix}
    		0.4097 & 0.1190 & 0.2078 \\
    		0.1190 & 0.2282 & -0.2274 \\
    		0.2078 & -0.2274 & 0.4046
    		\end{smallmatrix}\right]\end{gathered}
    \end{array}
    }
\label{equation:sensing-matrices-real-world}
\end{equation}

Moreover, this problem exhibits a pair of prominent local minima:
\begin{equation*}
	\hat{x} = \pm 
	\begin{bmatrix}
		0.2234 & 0.0918 & 0.5985
	\end{bmatrix}^\top.
\end{equation*}


\end{document}